\documentclass[twoside,11pt]{article}

\usepackage{blindtext}

\usepackage{blindtext}
\usepackage{amsmath,amssymb,amsthm}
\usepackage{algorithm}
\usepackage{algorithmic}
\usepackage{dsfont}
\usepackage{mathrsfs}
\usepackage{enumitem}
\usepackage{crossreftools}
\usepackage{hyperref}
\usepackage[]{cleveref}
\usepackage{color}
\crefname{ineq}{ineq.}{inequalities}
\creflabelformat{ineq}{#2{\upshape(#1)}#3}
\newcommand{\killdo}[2]{\unskip}

\usepackage{mathtools}
\usepackage{pifont}
\usepackage[perpage]{footmisc}
\usepackage{geometry}
\usepackage{tablefootnote}
\usepackage[nohyperref]{jmlr2e}

\newcommand{\cA}{\mathcal{A}}
\newcommand{\cB}{\mathcal{B}}

\newcommand{\cF}{\mathcal{F}}

\newcommand{\cN}{\mathcal{N}}
\newcommand{\cO}{\mathcal{O}}
\newcommand{\cP}{\mathcal{P}}

\newcommand{\cS}{{\mathcal{S}}}

\newcommand{\BB}{\mathbb{B}}
\newcommand{\EE}{\mathbb{E}}

\newcommand{\PP}{\mathbb{P}}

\newcommand{\RR}{\mathbb{R}}

\newcommand{\E}{{\mathbb E}}
\DeclareMathOperator{\ind}{\mathds{1}} 
\def \iid {\stackrel{\text{i.i.d.}}{\sim}}
\def\given{{\,|\,}}
\def\argmin{\mathop{\text{\rm arg\,min}}}
\def\argmax{\mathop{\text{\rm arg\,max}}}
\def\poly{\mathop{\text{\rm poly}}}
\def\var{\mathop{\text{var}}}

\hypersetup{
    colorlinks=true,
    linkcolor=red,
    citecolor=blue}

\usepackage{lastpage}

\usepackage{xcolor}
\newcommand\re[1]{\textcolor{black}{#1}}

\usepackage{graphicx}
\usepackage{subcaption}

\usepackage{lastpage}
\jmlrheading{27}{2026}{1-\pageref{LastPage}}{10/22; Revised 12/23}{2/26}{22-1194}{Ai, Lyu, Wang, Yang, and Jordan}

\ShortHeadings{RL in Multi-Phase Second-Price Auction Design}{Ai, Lyu, Wang, Yang, and Jordan}
\firstpageno{1}

\newtheorem{assumption}[theorem]{Assumption}

\begin{document}

\title{A Reinforcement Learning Approach in Multi-Phase Second-Price Auction Design}

\author{
    \hspace{-0.33em}\name Rui Ai \email ruiai@mit.edu\\ \addr  Institute for Data, Systems, and Society\\
    Massachusetts Institute of Technology\\
    Cambridge, MA 02139, USA
    \AND
    \name Boxiang Lyu \email blyu@chicagobooth.edu\\ \addr Booth School of Business\\
    The University of Chicago\\
    Chicago, IL 60637, USA
    \AND
    \name Zhaoran Wang \email zhaoranwang@gmail.com\\
    \addr Department of Industrial Engineering and Management Sciences\\
    Northwestern University\\
    Evanston, IL 60208, USA
    \AND
    \name Zhuoran Yang \email zhuoran.yang@yale.edu\\
    \addr Department of Statistics and Data Science\\
    Yale University\\
    New Haven, CT 06511, USA
    \AND
    \name Michael I. Jordan \email jordan@cs.berkeley.edu\\
    \addr Department of EECS, Department of Statistics \\
    University of California\\
    Berkeley, CA 94720, USA
}

\editor{Tor Lattimore}

\maketitle
\begin{abstract}
We study reserve price optimization in multi-phase second price auctions, where the seller's prior actions affect the bidders' later valuations through a Markov Decision Process (MDP). Compared to the bandit setting in existing works, the setting in ours involves three challenges. 
First, from the seller's perspective, we need to efficiently explore the environment in the presence of potentially untruthful bidders who aim to manipulate the seller's policy. 
Second, we want to minimize the seller's revenue regret when the market noise distribution is unknown. Third, the seller's per-step revenue is an unknown, nonlinear random variable, and cannot even be directly observed from the environment but realized values.

We propose a mechanism addressing all three challenges. To address the first challenge, we use a combination of a new technique named ``buffer periods'' and inspirations from Reinforcement Learning (RL) with low switching cost to limit bidders' surplus from untruthful bidding, thereby incentivizing approximately truthful bidding. The second one is tackled by a novel algorithm that removes the need for pure exploration when the market noise distribution is unknown. The third challenge is resolved by an extension of LSVI-UCB, where we use the auction's underlying structure to control the uncertainty of the revenue function. The three techniques culminate in the \underline{C}ontextual-\underline{L}SVI-\underline{U}CB-\underline{B}uffer (CLUB) algorithm which achieves $\tilde{\mathcal{O}}(H^{5/2}\sqrt{K})$ revenue regret, where $K$ is the number of episodes and $H$ is the length of each episode, when the market noise is known and $\tilde{\mathcal{O}}(H^{3}\sqrt{K})$ revenue regret when the noise is unknown with no assumptions on bidders' truthfulness.
\end{abstract}


\section{Introduction}
Second price auction with reserve prices is one of the most popular auctions both in theory~\citep{nisan2007algorithmic} and in practice~\citep{roth2002last}. While closed-form expressions for the optimal reserve price have been known ever since the seminal work of~\citet{myerson1981optimal}, directly applying the result requires that population information, such as the bidders' valuations' distribution, is known a priori. Various attempts have been made to weaken the assumption, with one of the most prominent lines of literature being reserve price optimization for repeated auctions in the contextual bandit setting~\citep{amin2014repeated,golrezaei2019dynamic,javanmard2019dynamic,deng2020robust}.

A limitation of existing works lies in the bandit assumption. Indeed, while reserve price optimization is already challenging as-is, allowing the auction to be both contextual and introducing temporally dependent dynamics, particularly, incorporating Markov Decision Process (MDP) induced dynamics in the evolution of bidders' preferences, opens up a wider range of problems for studying. For example,~\citet{dolgov2006resource} studies optimal auction under the setting and developed novel resource allocation mechanisms,~\citet{jiang2015data} leverages both MDP and auctions to better analyze resource allocation in IaaS cloud computing, and~\citet{zhao2018deep} uses deep Reinforcement Learning (RL) to study sponsored search auctions. We refer interested readers to~\citet{athey2013efficient} for more motivating examples. A question naturally arises: is it possible to optimize reserve prices when bidders' preferences evolve according to MDPs?

In this article, we provide an affirmative answer. Our work assumes that the state of the auction is affected by the state and the seller's action in the preceding step. To facilitate interpretation, we refer to the seller's action in this context as ``item choice'': bidders' later preferences could be affected by the types of items sold in previous rounds, a phenomenon well-documented by empirical works in auctions ~\citep{lusht1994order,jones2004auction,lange2010price,ginsburgh2007organizing}.

As is the case in many real-world problems, we assume that the underlying transition dynamics and the bidder's valuations are both unknown. We further emphasize that we do not make any truthfulness assumptions about the bidders, allowing them to be strategic with their reporting.
Under such a challenging setting, our goal is to learn the optimal policy of the seller in the unknown environment, in the presence of untruthful bidders. 

{\noindent\textbf{Our Contributions.}}
We begin by summarizing the three key challenges we face.
First, bidders have the incentive to report their valuation untruthfully, in hopes of manipulating the seller's learned policy, through either overbidding or underbidding, making it difficult to estimate their true preferences and the underlying MDP dynamics. Existing works such as~\citet{amin2014repeated,golrezaei2019dynamic,deng2020robust} do not apply due to technical challenges unique to MDP. Second, when the market noise distribution is unknown, even in the bandit setting existing literature often only obtains $\tilde{\cO}(K^{2/3})$ guarantee~\citep{amin2014repeated,golrezaei2019dynamic} and $\Omega(K^{2/3})$ revenue regret lower bound exists \re{in the worst case}~\citep{kleinberg2003value}. Third, the seller's reward function, namely revenue, is unknown, nonlinear, and can not be directly observed from the bidders' submitted bids, and LSVI-UCB~\citep{jin2020provably} cannot be directly applied.

We are able to address all three challenges with the CLUB algorithm. Motivated by the ever increasing learning periods in existing works~\citep{amin2014repeated,golrezaei2019dynamic,deng2020robust}, our work further draws inspiration from RL with low switching cost~\citep{wang2021provably} and proposes a novel concept dubbed ``buffer periods'' to ensure that the bidders are sufficiently truthful. Additionally, we feature a novel algorithm we dub ``simulation'' which, combined with a novel proof technique leveraging the Dvoretzky–Kiefer–Wolfowitz inequality~\citep{dvoretzky1956asymptotic}, yields $\tilde{\cO}(\sqrt{K})$ revenue regret under only mild additional assumptions. Finally, by exploiting the mathematical properties of the revenue function, our work provides a provably efficient RL algorithm for when the reward function is nonlinear.

\subsection{Related Works}\label{subsec:related_work}
We summarize below two lines of existing literature pertinent to our work.

{\noindent\textbf{Reserve Price Optimization.}}
There is a vast amount of literature on price estimation~\citep{cesa2014regret,qiang2016dynamic,shah2019semi,drutsa2020reserve,kanoria2020dynamic,keskin2021dynamic,guo2022no}. \citet{deng2020robust} considers a model where buyers and sellers are equipped with different discount rates, proposing a robust mechanism for revenue maximization in contextual auctions. \citet{javanmard2020multi} proposes an algorithm with $\tilde \cO(\sqrt T)$ regret ($T=KH$ in our paper) while \citet{fan2021policy} achieves sublinear regret in a more complex setting. \citet{cesa2014regret} studies reserve price optimization in non-contextual second price auctions, obtaining $\tilde \cO(\sqrt{T})$ revenue regret bound.~\citet{drutsa2017horizon,drutsa2020reserve} studies revenue maximization in repeated second-price auctions with one or multiple bidders, proposing an algorithm with a $\cO (\log \log T)$ worst-case regret bound. However, their setting is non-contextual, and they cannot be applied to our setting. 

Among this line of research, \citet{golrezaei2019dynamic,golrezaei2023incentive} are possibly the closest to our work. \citet{golrezaei2019dynamic} assumes a linear stochastic contextual bandit setting, where the contexts are independent and identically distributed, achieving $\tilde \cO(1)$ regret when the market noise distribution is known and $\tilde \cO (K^{2/3})$ when it is unknown and nonparametric. While the $\tilde{\cO}(1)$ regret under known market noise distribution seems to be better than our bound, we emphasize that their stochastic bandit setting does not require exploration over the action space but needed in our work and, even in generic linear MDPs, a $\Omega{(\sqrt{K})}$ regret lower bound exists~\citep{jin2020provably}. \re{For unknown distribution, there's another difference that they consider a time-varying model while we focus on dealing with the underlying fixed MDP. Though the difficulty of these tasks is hard to compare directly, \citet{amin2014repeated} considers a non-parametric but fixed distribution setting and suffers $\tilde O(K^{2/3})$ regret, which may hint at the main difficulty comes from a non-parametric rather than time-varying setting. We delay more discussion about concrete techniques in \citet{golrezaei2019dynamic} in \Cref{app:comp}.}
 Lastly, as we discussed previously, the approaches in~\citet{golrezaei2019dynamic} cannot be directly applied in the MDP setting, necessitating our novel algorithmic structure.

\re{At the same time with our paper, \citet{golrezaei2023incentive} considers another pricing problem with non-parametric noise, achieving $\tilde{O}(\sqrt{T})$ regret. However, they only set a reserve price for all bidders, while we customize reserve prices for each bidder to attain more revenue. On the one hand, the seller will achieve more revenue by setting different reserve prices for different bidders, which is in line
with the goal of the seller because there are fewer corresponding constraints. On the other hand, in the
real world, it is more common to set up personalized reserve prices in the online advertisement
market, like price discrimination~\citep{paes2016field,wu2019effective}. 
Additionally, \citet{golrezaei2023incentive} is in the scope of contextual bandits and is a special case of our MDP setting. Pricing in contextual bandit settings is much easier than MDP because i.i.d. context will form a positive definite covariance matrix, and linear regression works well. But in MDP, features depend on action and are absolutely not i.i.d. Without the positive definite assumption, algorithms designed for contextual bandits lose effectiveness, and we need innovative algorithms to incorporate pricing and complex information structures.}

{\noindent\textbf{RL with Linear Function Approximation.}}
Linear contextual bandit is a popular model for online decision making~\citep{rusmevichientong2010linearly,abbasi2011improved,chu2011contextual,li2019nearly,lattimore2020bandit} that has also been extensively studied from the auction design perspective~\citep{amin2014repeated,golrezaei2019dynamic}. Its dynamic counterpart,
Linear MDP, remains popular in the analysis of provably efficient RL \citep{yang2019sample,jin2020provably,jin2021pessimism,yang2020function,zanette2020learning,jin2021bellman,uehara2021representation,yu2022strategic,wang2021provably,gao2021provably}. In particular, \citet{jin2020provably} is one of the first papers to introduce the concept, proposing a provably efficient RL algorithm with $\tilde\cO(\sqrt{K})$ regret. \citet{jin2021pessimism} generalizes the idea to offline RL.

While we use linear function approximation, the seller's per-step reward function, revenue, is non-linear. Our work also features novel per-step optimization problems to combat effects from untruthful reporting. While our work draws inspiration from~\citet{wang2020optimism} and~\citet{gao2021provably}, as we discussed previously, these inspirations are needed for obtaining high-quality estimates when the bidders are untruthful. Thus, our work differs significantly from prior works on linear MDPs.

{\noindent\textbf{Notations.}}
For any positive integer $n$ we let $[n]$ denote the set $\{1, \ldots, n\}$. For any set $A$ we let $\Delta(A)$ denote the set of probability measures over $A$. For sets $A$, $B$, we let $A \times B$ be the Cartesian product of the two. \re{During the whole paper, we use $k\in[K]$ to refer to an episode and $h\in[H]$ to refer to a horizon. In addition, we use $\tilde k$ to refer to a buffer period associated with the $k$-th episode.}

\section{Preliminaries}
We consider a repeated (lazy) multi-phase second-price auction with personalized reserve prices. Particularly, we assume that there are $N$ rational bidders, indexed by $[N]$, and one seller participating in the auction. For ease of presentation, we use ``he" to refer to a specific bidder and ``she" to the seller.

{\noindent\textbf{Second Price Auction with Personalized Reserve Prices.}}
We begin by describing a single round of the auction. Each bidder $i \in [N]$ submits some bid $b_i \in \RR_{\geq 0}$ and the seller determines the personalized reserve prices for the bidders in the form of reserve price vector $\rho \in \RR_{\geq 0}^N$, with $\rho_i$ denoting bidder $i$'s reserve price. The bidder with the highest bid only wins if he also clears his personal reserve price, i.e., $b_i \geq \rho_i$. If the bidder $i$ receives the item, he pays the seller the maximum of his personalized reserve and the second highest bid, namely $\max\{\rho_i, \max_{j \neq i} b_j\}$, which we dub $m_i$ for simplicity. When the bidder with the highest bid fails to clear his personalized reserve price, the auction fails, the seller gains zero, and the item remains unsold. In summary, bidder $i$ receives the item if and only if $b_i \geq m_i$ and the price he pays is $m_i$. For any round of auction, we let $q_{i} = \ind(\textrm{bidder $i$ receives the item})$ indicate whether bidder $i$ received the item or not. For the sake of convenience, throughout the paper, we assume that there are no ties in the submitted bids.

{\noindent\textbf{A Multi-Phase Second Price Auction.}}
We now characterize the dynamics of the multi-phase auction setting we study. Assume that the transition dynamic between rounds can be modeled as an episodic Markov Decision Process (MDP)\footnote{We can easily extend our setting to that of an infinite-horizon MDP by improperly learning the process as an episodic one. Here we focus on the finite-horizon case purely for simplicity of presentation.}. A multi-phase second price auction with personalized reserves is parameterized as $(\mathcal{S}, \Upsilon, H,\mathbb{P},\{r_i\}_{i = 1}^N)$, with the state space denoted by $\mathcal{S}$, seller's item choice space $\Upsilon$\footnote{Here we use ``item choice" to better illustrate what $\Upsilon$ intuitively represents. The term can be extended to more generic notions of the seller's action.}, horizon $H$, transition kernel $\mathbb{P}=\{\mathbb{P}_h\}_{h=1}^H$ where $\PP_h: \cS \times \Upsilon \to \Delta(\cS)$, and the individual bidders' reward functions $r_i=\{r_{i h}\}_{h=1}^{H}$ for all $i \in [N]$. The choice of item $\upsilon\in\Upsilon$ affects the bidders' rewards as well as the transition.

The interaction between the bidders and the seller is then defined as follows. We assume without loss of generality that the state at the initial step is fixed at some $x_{1} \in \cS$. For each $h \in [H]$, the seller and the bidders engage in a single round of the second-price auction. Given the seller's item choice at step $h$, $\upsilon_{h}$, nature transitions to the next state according to the transition kernel $\PP_{h}$.

{\noindent\textbf{Bidder Rewards.}} 
We assume that for each bidder $i \in [N]$ at time $h\in [H]$, his reward\footnote{We use the term ``reward'' to maintain consistency with existing RL literature.} depends on both the state $x$ and item being auctioned off at that round $\upsilon \in \Upsilon$, which we formalize as
\[r_{ih}(x,\upsilon)=1+\mu_{ih}(x,\upsilon)+z_{ih} \textrm{, where } z_{ih} \iid F.\]
Here, $z_{ih}$ denotes the randomness within bidders' rewards and is drawn i.i.d. from the market noise distribution $F(\cdot)$. We assume that $F(\cdot)$ is supported on $[-1,1]$ and has mean 0. Let $\mu_{i, h}: \cS \times \Upsilon \to [0,1]$ denote the conditional expectation of the reward less one, where the constant is added to ensure $r_{ih}(x, \upsilon) \in [0, 3]$.

{\noindent \textbf{Policies and Value Functions.}}
Before we describe the seller's policy, we first discuss the action space $\cA = \Upsilon \times \RR_{\geq 0}^N$. At each $h \in [H]$, the seller chooses some action $a_{h} = (\upsilon_{h}, \rho_{h})$, comprising of item choice $\upsilon \in \Upsilon$ and reserve price vector $\rho \in \RR_{\geq 0}^N$. The seller's policy is then $\pi = \{\pi_h\}_{h = 1}^H$, where $\pi_h: \cS \to \Delta(\cA)$. We let $\pi^{\upsilon}$ and $\pi^{\rho}$ denote the marginal item choice and reserve price policies, respectively. Recall that the seller garners revenue only when the item is sold to a bidder. At each $h \in [H]$, her per-step expected revenue is then

\begin{equation}
\label{eqn:defn_revenue}
    R_h = \E_{\{z_{ih}\}_{i = 1}^{N}}\left[\sum_{i = 1}^N m_{ih}\ind(m_{ih} \leq b_{ih})\right]
\end{equation}
as we recall that $m_{ih} = \max\{\rho_{ih}, \max_{j \neq i}b_{jh}\}$ and bidder $i$ pays the seller $m_{ih}$ if and only if $b_{ih} \geq m_{ih}$. The value function (V-function) of the seller's revenue for any policy $\pi$ 
and the action-value function (Q-function) is $Q_h^\pi: \mathcal{S}\times \mathcal{A}\rightarrow \mathbb{R}$ are then
\[
    V_h^\pi(x)=\E_\pi \left[\sum_{h'=h}^H R_{h'}(x_{h'},a_{h'})\given x_h=x\right]
\]\textrm{and }
\[
    Q_h^\pi(x, a)=\E_\pi \left[\sum_{h'=h}^H R_{h'}(x_{h'}, a_{h'})\given x_h=x,a_h=a\right],
\]
respectively. 

Since the bidder reward only depends on state $x$ and the choice of item $\upsilon$ instead of reserve $\rho$, we have a family of mappings from $\cS \times \Upsilon$ to $\RR_{\geq 0}^N$ that determines $\rho$. Therefore, with a slight abuse of notation, we can rewrite our Q-function as $Q(x,a)=Q(x,(\upsilon,\rho(x,\upsilon)))$, restricting the role of setting reserve prices using such mappings without loss of generality. From now on, we use $Q(x,\upsilon)$ to denote the Q-function for simplicity.
For any function $f:\mathcal{S}\rightarrow \mathbb{R}$, we define the transition operator $\cP$ and the Bellman operator $\cB$ as
\[
    (\cP_h f)(x, a)=\E[f(x_{h+1})\given x_h=x,a_h= a], \,(\mathcal{B}_hf)(x,a)=\E[R_h(x_h, a_h)]+(\mathbb{P}_hf)(x, a),
\] respectively. Finally, we let $\pi^{\star}$ denote the optimal policy when the bidders' reward functions, the MDP's underlying transition, and the market noise distribution are all known to the seller. We remark that when these parameters are known, second price auctions with personalized reserve prices are inherently incentive compatible and rational bidders will bid truthfully.

{\noindent \textbf{Performance Metric.}}
The revenue suboptimality for each episode $k \in [K]$ is
\[
    \textrm{SubOpt}_k(\pi_k)=V_1^{\pi^*}(x_{1})-V_1^{\pi_k}(x_{1}),
\]
with $\pi_k$ being the strategy used in episode $k$. Our evaluation metric is then the revenue regret attained over $K$ episodes, namely
\begin{equation}\label{regret}
    \textrm{Regret}(K)=\sum_{k=1}^{K} \textrm{SubOpt}_k(\pi_k).
\end{equation}

{\noindent \textbf{Impatient Utility-Maximizing Bidders.}}
We assume the bidders are equipped with some discount rate $\gamma\in (0,1)$ while the seller's reward is not discounted. For the sake of simplicity, we assume $\gamma$ is common knowledge. \re{\citet{drutsa2020reserve} consider a scenario where $\gamma$ is unknown but with a strictly less than one upper bound. We highlight that it also works with our CLUB algorithm as long as we replace $\gamma$ with its upper bound. We can have regret bounds with the same order because we adopt more conservative estimators, and buyers won't violate as much as the corresponding results of $\gamma$. Then all results in our paper hold up to some changes of absolute constants.} Rational bidder $i$'s utility at step $h$ is given by $(r_{ih}(x_h, \upsilon_h) - m_{ih})\ind(b_{ih} \geq m_{ih})$, as we note that he only receives nonzero utility upon winning the auction. His objective is to maximize his discounted cumulative utility
\[
    \textrm{Utility}_i = \sum_{k = 1}^K\gamma^k\EE_{\pi_k}\left[\sum_{h = 1}^H (r_{ih}(x^k_h, \upsilon^k_h) - m^k_{ih})\ind(b^k_{ih} \geq m^k_{ih}) \given x^k_1 = x_1\right].
\]

Note that in practical applications, sellers are usually more patient than bidders and discount their future rewards less. Consider a sponsored search auction, where the seller usually auctions off large numbers of ad slots every day. Bidders usually urgently need advertisements and value future rewards less. On the other hand, the seller is not especially concerned with slight decreases in immediate rewards. We refer the readers to \citet{drutsa2017horizon,golrezaei2019dynamic} for a more detailed discussion on the economic justifications of the assumption and emphasize that the assumption is necessary, as \citet{amin2013learning} shows that when the bidders are as patient as the seller, achieving sub-linear revenue regret is impossible.

{\noindent\textbf{Linear Markov Decision Process.}}
As a concrete setting, we study linear function approximation. 
\begin{assumption}\label{assumption:linearmdp}
Assume that there exists known feature mapping $\phi: \cS \times \Upsilon \to \RR^d$ such that there exist $d$-dimension unknown (signed) measures $\mathcal{M}_h$ over $\mathcal{S}$ and unknown vectors $\{\theta_{ih}\}_{i = 1}^N \in \RR^d$ that satisfy
\[
\mathbb{P}_h(x'|x,\upsilon)=\langle\phi(x,\upsilon),\mathcal{M}_h(x')\rangle, \, \mu_{ih}(x, \upsilon) = \langle \phi(x, \upsilon), \theta_{ih} \rangle
\]for all $(x,\upsilon,x')\in \cS\times \Upsilon\times\cS$, $i \in [N]$, and $h \in [H]$. Without loss of generality, we assume that $\|\phi(x,\upsilon)\|\le 1$ for all $(x,\upsilon)\in \cS\times \Upsilon$, $\|\mathcal{M}_{h}(\mathcal{S})\| \le \sqrt d$, and $\|\theta_{ih}\|\leq \sqrt{d}$ for all $h \in [H]$ and $i \in [N]$. 
\end{assumption}
\re{There're some scenarios in reality that mapping $\phi(\cdot,\cdot)$ is public knowledge like representing the order of items. However, for unknown mapping~\citep{lattimore2020learning}, there are some ways to pre-train features using a reproducing kernel Hilbert space, neural networks, or the Knowledge Discovery in Databases (KDD) method~\citep{lange2010deep,claessens2016convolutional,wang2020gcn}. Utilizing these, we can obtain a working feature representation in practice.} 

We remark that while the transition kernel $\PP_{h}$ and the bidders' individual expected reward functions $\{\mu_{i}\}_{i = 1}^{N}$ are linear, the seller's objective, revenue, is not linear, differentiating our work from typical linear MDP literature (see \citet{yang2019sample, jin2020provably} for representative works).

{\noindent\textbf{Motivations for the MDP Model .}}
We close off the section by providing some practical applications of our MDP model. The core of our setting is to study what will happen when selling heterogeneous goods and how the order, part of the state, will affect the revenue. We provide three real-world scenarios to motivate this phenomenon.
\re{\begin{itemize}
    \item \textbf{(Online Advertisement)} Google sells lots of advertising positions every day, while buyers face budget constraints. In the early rounds, since buyers have more budget left, they are usually eager to bid higher and have a stronger willingness to pay. Therefore, Google may want to sell the most valuable position at first so that buyers have the ability to pay higher acceptable prices and avoid being underbid and unsold.
    \item \textbf{(Antique Auction)} For traditional auction design, the prior auctions may affect the latter auctions. For instance, consider when Sotheby's wants to sell several antiques. The order of selling is of significance, and that's the reason why Sotheby's needs to sell a few other pieces to warm up before selling the final flagship piece. The order, part of the state, influences people's valuation and, consequently, total revenue. For example, if Sotheby's wishes to auction a valuable Chinese ancient artifact, they would auction some related artifacts during the warm-up session to enhance buyers' expectations.
    \item \textbf{(Automobile Sales Market)} The last example is on the market of cars. If one buyer wants to buy a sedan from General Motors, recommending Chevrolet first or Cadillac first will influence his preference for the course. If he sees Chevrolet first, he may think Cadillac is too expensive. However, if he sees Cadillac first, he may think Chevrolet lacks a sense of experiential quality. To achieve maximum profitability, General Motors carefully arranges the recommended order. In a broader sense, they meticulously design the sequence in which cars appear in advertisements.
\end{itemize}}
\re{All in all, contextual bandits lack the ability to depict such kinds of problems. We need to use MDP to model these issues.} 

\section{Known Market Noise Distribution}\label{sec:KnownF}
We remind the readers of our three main challenges, with the first challenge being exploring the environment, even when the bidders submit their bids potentially untruthfully. The second challenge emerges only when the market noise distribution is unknown, and we defer its resolution to Section~\ref{sec:unknownF}. The third challenge is performing provably efficient RL even when the seller's per-step revenue, detailed in~\Cref{eqn:defn_revenue}, is nonlinear and not directly observable.

In this section, we present a version of CLUB when the market noise distribution is known. We assume for convenience that $K$ is known, as we can use the doubling trick (see \citet{auer2002adaptive} and \citet{besson2018doubling} for discussions) to achieve the same order of regret when $K$ is unknown or infinite. \re{Since we can utilize the doubling trick to partition $K$ into at most $[\log_2 K]+1$, adding corresponding regret will lead to a regret bound of the same order up to some logarithmic terms.}

\subsection{CLUB Algorithm When \texorpdfstring{$F(\cdot)$}{} is Known}
We start with the first challenge, which we address by a collection of algorithms that successfully induce approximately truthful bids from the bidders.

{\noindent \textbf{Addressing Challenge 1: Untruthfulness.}} To curb the sellers' untruthfulness, we need to punish such behavior, achieved through a random pricing policy in the form of~\Cref{algo:pizero}. For each $h \in [H]$, $\pi_{\rm rand}$ randomly chooses an item and a bidder, offering him the item with a reserve price drawn uniformly at random. The bidder's utility decreases whenever he reports untruthfully, risking either not receiving the item when he underbids or overpaying for an item when he overbids. \re{Combining lazy updates (see \Cref{algo:buffer}), we can ensure approximate truthfulness because with the discount rate being less than one, the benefit the bidder gains from misreporting the bids will decay as the timestep increases. However, since we consider a multi-phase auction design, there remains some nuisance introduced by MDP. For instance, there is no guarantee of a positive definite covariance matrix, and it's challenging to give a low regret union bound. In other words, $\sum\phi\phi^T$ might have some zero eigenvalues. We will see how to solve them in the following paragraphs.}

\begin{algorithm}[bht]
    \caption{Definition of $\pi_{\rm rand}$}
    \label{algo:pizero}
    \begin{algorithmic}[1]
        \FOR {$h = 1, \ldots, H$}
        \STATE Randomly chooses an item $\upsilon_{h} \in \Upsilon_{h}$.
        \STATE Choose a bidder $i \in [N]$ uniformly at random and
        offer him the item with reserve price $\rho_{ih}\sim
        \textrm{Unif}([0, 3])$. Set other bidders' reserve prices to infinity.
        \ENDFOR
    \end{algorithmic}
\end{algorithm}

\re{We further introduce a novel technique, ``buffer period'', which explicitly forces the bidders to wait before starting a new learning period, thereby decreasing the discounted utility the impatient bidders may gain from untruthfulness.
Indeed, a typical algorithm in the bandit setting only features $\pi_{\rm rand}$ and a sequence of learning periods that double in length~\citep{amin2014repeated, golrezaei2019dynamic,deng2020robust}. In the bandit setting, data collected in all previous periods is used to update the policy at the end of each period. The increasingly lengthy periods ensure that the seller switches policy less frequently, ensuring that the impatient buyers need to wait longer before benefiting from untruthful reporting, deterring them from doing so. Unfortunately, the same technique does not work for MDPs, as the rate at which the smallest eigenvalue of the covariance matrix estimate grows cannot be determined, and we cannot ensure our estimate of the underlying environment is not ``stale'' when we double the length of the periods. }

\begin{algorithm}[bht]
    \caption{Buffer Period with Known $F(\cdot)$}
    \label{algo:buffer}
    \begin{algorithmic}[1]
      \STATE Receives buffer start $\mathtt{buffer.s}(\tilde{k} + 1) = k$ and end $\mathtt{buffer.e}(\tilde{k} + 1) = k + \frac{3\log K}{\log(1/\gamma)}$.
      \STATE Do nothing for all episodes $\mathtt{buffer.s}(\tilde{k} + 1) \leq k < \mathtt{buffer.e}(\tilde{k} + 1)$, i.e., do nothing during the buffer period before the end.
      \STATE At the end of the buffer period, update policy estimate $\pi_{\tilde{k} + 1}$ and Q-function estimate $\hat{Q}_h^{\pi_{\tilde{k} + 1}}(\cdot,\cdot)$ using \Cref{algo:estimate}, and then increment buffer period counter $\tilde{k} \gets \tilde{k} + 1$.
    \end{algorithmic}
\end{algorithm}

While we can mimic the aforementioned bandit algorithms by drawing inspiration from low-switching cost RL literature, we cannot guarantee that the periods are sufficiently long without buffer periods. Indeed, we can use the smallest eigenvalue of the covariance matrix to determine when to start a new period. However, it is impossible to determine a priori the rate at which the smallest eigenvalue grows. Buffer periods ensure that each period is sufficiently long, deferring any utility gain from untruthful reporting. Combined with the bidders' discount rate, a combination of $\pi_{\rm rand}$ and buffer periods ensures that the bidders behave approximately truthfully. The technique is detailed in~\Cref{algo:buffer}. 

\begin{algorithm}[hbt]
    \caption{ \underline{C}ontextual-\underline{L}SVI-\underline{U}CB-\underline{B}uffer (CLUB) with Known $F$}
    \label{algo:KnownF}
    \begin{algorithmic}[1]
      \STATE Initialize policy estimate $\pi_{0}$, buffer period counter $\tilde{k} = 0$, buffer period starting points $\mathtt{buffer.s}(0) = 1$, and buffer period end points $\mathtt{buffer.e}(0) = 1$.
      \FOR {episodes $k = 1, \ldots, K$}
      \STATE Execute mixture policy $\frac{1}{HK}\circ\pi_{\rm rand}+(1-\frac{1}{HK})\circ \pi_{\tilde{k}}$, collecting outcomes $q_{ih}^{\tau}$ and updating matrices $\Lambda_h^k\leftarrow \sum_{\tau = 1}^{k}\phi(x^{\tau}_h,\upsilon^{\tau}_h)\phi(x^{\tau}_h,\upsilon^{\tau}_h)^T + I$ for all $h \in [H]$.
      \STATE If there exists $h \in [H]$ such that $(\Lambda_h^{\mathtt{buffer.e}(\tilde{k})})^{-1} \not \preceq 2(\Lambda_h^k)^{-1}$, schedule a new buffer period starting at $\mathtt{buffer.s}(\tilde{k} + 1) = k$ and ending at $\mathtt{buffer.e}(\tilde{k} + 1) = k + \frac{3\log K}{\log(1/\gamma)}$ using~\Cref{algo:buffer}\re{, and set $k\leftarrow \mathtt{buffer.e}(\tilde{k} + 1)$.} \label{algoline:unknownFupdate}
      \ENDFOR
    \end{algorithmic}
\end{algorithm}

With buffer periods defined, we summarize CLUB's update schedule in~\Cref{algo:KnownF} and include~\Cref{fig:my_label} for visual representation. Let $\frac{1}{HK}\circ\pi_{\rm rand}+(1-\frac{1}{HK})\circ \pi_{\tilde{k}}$ represent a mixture policy combining $\pi_{\rm rand}$ and $\pi_{\tilde{k}}$ where for each $h$, with probability $\frac{1}{HK}$ we act according to $\pi_{\rm rand}$ and with probability $1 - \frac{1}{HK}$ according to $\pi_{\tilde{k}}$. For convenience, we assume $\mathtt{buffer.e}(\tilde{k})$ is an integer, as rounding up $\mathtt{buffer.e}(\tilde{k})$ does not affect asymptotic regret. Unlike a typical low switching cost RL algorithm,~\Cref{algo:estimate} further delays updating for $\frac{3\log K}{\log(1/\gamma)}$ episodes after the switching criterion in Line~\ref{algoline:unknownFupdate} is satisfied. 
\begin{figure}[bth]
    \centering
    \includegraphics[width=0.75\linewidth]{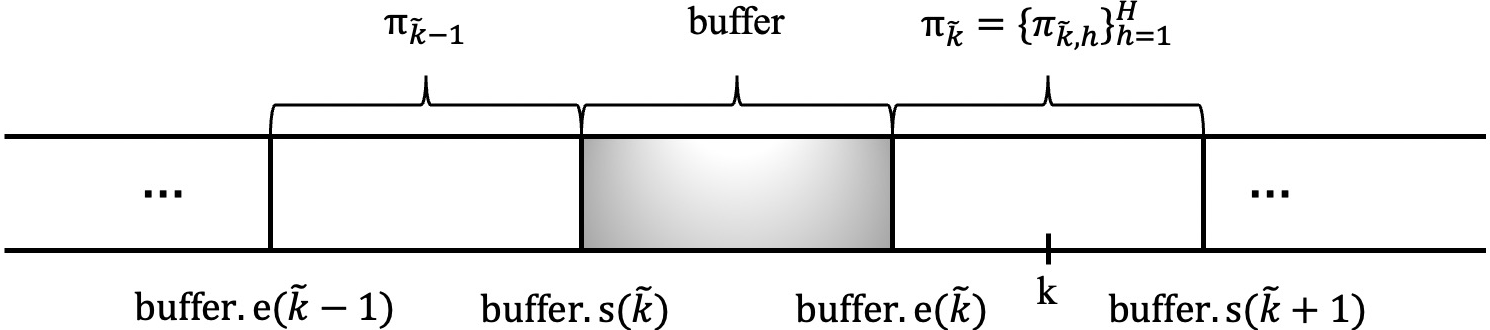}
    \caption{Learning periods and buffer periods\re{: $\mathtt{buffer.s(\cdot)}$ and $\mathtt{buffer.e(\cdot)}$ represent the start point and the end point of a buffer, respectively. Episode $k$ lays between $\mathtt{buffer.e(\tilde k)}$ and $\mathtt{buffer.s(\tilde k+1)}$ and the length of each buffer is $\frac{3\log K}{\log(1/\gamma)}$.}}
    \label{fig:my_label}
\end{figure}

The mixture policy sufficiently punishes untruthfulness. Combined with buffer periods (\Cref{algo:buffer}) and the update schedule (Line~\ref{algoline:unknownFupdate}),~\Cref{algo:KnownF} also limits the discounted utility bidders gain from untruthfulness, thereby curbing excessive overbidding and/or underbidding. \re{Line~\ref{algoline:unknownFupdate} represents a kind of lazy update. We only calculate the new Q-function when at least one eigenvalue decays by half, restricting the total number of updates and benefiting the construction of high probability union regret bounds. At the same time, we wait for the length of buffer periods before updating to motivate truthful bidding.} While $\pi_{\rm rand}$ is suboptimal, the mixture policy ensures that it is not executed too many times, reducing its damage to revenue.

With the techniques discussed above, namely Algorithms~\ref{algo:pizero},~\ref{algo:buffer}, and~\ref{algo:KnownF}, {we now have sufficiently addressed our first challenge}, obtaining approximately truthful reports in the face of strategic bidders. We now turn to tackling the third challenge outlined in the abstract: provably efficient reinforcement learning even when the per-step revenue is nonlinear.

{\noindent \re{\textbf{Addressing Challenges 2 and 3: Regret Minimization and Nonlinear Revenue.}}}
Having shown that our algorithm punishes untruthful behavior, we begin by showing that the resulting reports are sufficiently truthful for obtaining accurate parameter estimates. \re{It's still quite intricate as regret depends on both the state, action related to the transition kernel, and reserve prices. Traditional point estimation with uncertainty quantity is not enough since we need to not only combine the structure of the underlying MDP and coordinate with buffer periods, the so-called lazy updates, but also consider the small proportion of untruthful bids.} Whereas LSVI-UCB directly learns from empirical rewards, here we use indicators $q_{ih}^{k}$, which we recall is one if bidder $i$ receives the item at episode $k$ step $h$ and zero otherwise. As we cannot guarantee that the empirical covariance matrix is positive definite, existing techniques in~\citet{amin2014repeated, golrezaei2019dynamic} cannot be applied.
We instead have
\begin{equation}
    \label{algo:thetahat}
    \hat \theta_{ih}=\argmin_{\|\theta\|\le 2\sqrt{d}}\sum_{\tau=1}^{\mathtt{buffer.e}(\tilde{k}+1)} (q_{ih}^\tau -1 +F(m_{ih}^\tau   -1-\langle\phi(x_h^\tau,\upsilon_h^\tau),\theta\rangle))^2,
\end{equation}
where $\rho_{ih}^{\tau}$ is agent $i$'s reserve price and $m_{ih}^{\tau} = \max\{\max_{j \neq i}b_{ih}^{\tau}, \rho_{ih}^{\tau}\}$.~\Cref{algo:thetahat} is justified by the observation that, assuming that he bids truthfully, bidder $i$ wins the auction with probability $1 - F(m_{ih}^\tau - 1 - \langle\phi(x_h^\tau,\upsilon_h^\tau),\theta\rangle)$, conditioned on $x_h^\tau, \upsilon_h^\tau$, and $m_{ih}^\tau$. Controlling the uncertainty around $\hat{\theta}_{ih}$ then resembles controlling the uncertainty of a generalized linear model with $F(\cdot)$ being the link function. As bidders need to overbid or underbid significantly to alter the outcome of the auction, $\hat{\theta}_{ih}$ is less susceptible to untruthfulness.

While we use a typical linear function approximation assumption, the seller's revenue function $R_{h}$ is not linear, and we cannot directly apply existing approaches. We instead directly estimate $R_{h}$ and link our uncertainty on the seller's revenue to the typical linear MDP uncertainty quantifier, summarized~\Cref{algo:estimate}.

\begin{algorithm}[bht]
    \caption{Estimation of $\hat Q_h^{\pi_{\tilde k+1}}(\cdot,\cdot)$}
    \label{algo:estimate}
    \begin{algorithmic}[1]
        \STATE Estimate $\hat \theta_{ih}$ using \Cref{algo:thetahat} and set $\hat \mu_{ih}(\cdot,\cdot)\leftarrow \langle\phi(\cdot,\cdot),\hat \theta_{ih}\rangle$ for all $i, h$.
        \label{algoline:thetaest}
        \STATE Estimate reserve price $\hat \rho_{ih}(\cdot,\cdot)=\argmax_y{y(1-F(y-1-\hat \mu_{ih} (\cdot,\cdot)))}$ for all $i, h$.\label{algoline:rhoest}
        \STATE Estimate revenue $\hat R_h(\cdot,\cdot)\leftarrow \E[\max\{\tilde b_h^{-}(\cdot,\cdot),\hat \rho_h^+(\cdot,\cdot)\}\ind(\tilde b_h^{+}(\cdot,\cdot)\ge \hat \rho^+_h(\cdot,\cdot) )]$.\label{algoline:r_est}
        \FOR {$h=H,\ldots,1$ \textbf{do} \COMMENT{Estimate $Q$-function and optimal policy.} \killdo} \label{algoline:p_est_start}
        \STATE $\Lambda_h\leftarrow \sum_{\tau=1}^{\textrm{buffer.e} (\tilde{k}+1)}\phi(x_h^\tau,\upsilon_h^\tau)\phi(x_h^\tau,\upsilon_h^\tau)^T+\lambda I$.\hfill $\rhd$ We set $\lambda=1$ in this paper.
        \STATE $\omega_h \leftarrow \Lambda_h^{-1} \sum_{\tau=1}^{\textrm{buffer.e} (\tilde{k}+1)} \phi(x_h^\tau,\upsilon_h^\tau)[\max_{\upsilon} \hat Q_{h+1} (x_{h+1}^\tau,\upsilon)]$.
        \STATE  $\hat Q^{\pi_{\tilde{k}+1}}_h(\cdot,\cdot)\leftarrow \min\{\omega_h^T\phi(\cdot,\cdot)+\hat R(\cdot,\cdot)+\poly (\log K)\|\phi(\cdot,\cdot)\|_{\Lambda_h^{-1}},3H\}$. \label{algoline:ucb}
        \STATE $\pi^{\upsilon}_{\tilde{k}+1,h}(\cdot)\gets\argmax_v \hat Q^{\pi_{\tilde k+1}}_h(\cdot,v)$.
        \STATE $\pi^{\rho_i}_{\tilde{k}+1,h}(\cdot)\leftarrow \hat  \rho_{ih}(\cdot,\pi^{\upsilon}_{\tilde{k}+1,h}(\cdot)) $.
        \ENDFOR \label{algoline:p_est_end}
        \STATE Return $\{\hat Q^{\pi_{\tilde k+1}}_h(\cdot,\cdot)\}_{h=1}^H$ and $\{\pi_{\tilde k+1,h}(\cdot)\}_{h=1}^H$.
    \end{algorithmic}
\end{algorithm}

We let $\tilde b^{+}$ and $\rho^+$ denote the highest truthful bid and the highest reserve price, respectively. Similarly, let $\tilde b^{-}$ and $\rho^-$ denote the second-highest. \Cref{algo:estimate} estimates the Q-function optimistically by dividing the problem into two halves: per-step revenue estimation (Lines~\ref{algoline:thetaest} to~\ref{algoline:r_est}) and transition estimation (Lines~\ref{algoline:p_est_start} to~\ref{algoline:p_est_end}). In the first half, we use~\Cref{algo:thetahat} to estimate all $\theta_{ih}$, which in turn gives estimates for bidders' rewards in the form of $\hat{\mu}_{ih}$. We then feed the reward function estimates to Line~\ref{algoline:rhoest}, yielding an estimate for the optimal reserve price. With Algorithms~\ref{algo:pizero},~\ref{algo:buffer}, and~\ref{algo:KnownF}, the effects of untruthful reports are controlled, and we can ensure that the revenue estimate is sufficiently close to the ground truth. With $\rho_{ih}$ estimated, we then obtain revenue estimates for all states and item choices via Line~\ref{algoline:r_est}. Consequently, we decide both nearly optimal reserve prices and the order of items, addressing the second challenge of regret minimization. 

While the rest of~\Cref{algo:estimate} resembles a typical LSVI-UCB algorithm~\citep{jin2020provably}, we highlight several key differences. First, we use the plug-in revenue estimate, whereas existing works estimate the Q-function with the empirically observed rewards. To accommodate the plug-in estimate, here $\omega_{h}$ estimates $\PP_{h} V_{h + 1}$, the transition operator applied to the V-function, as opposed to $\BB_{h} V_{h + 1}$, which uses the Bellman evaluation operator instead. Lastly, in Line~\ref{algoline:ucb} we link the uncertainty of revenue to the uncertainty bonus typically seen in linear MDPs, thereby obtaining an optimistic estimate of the Q-function induced by revenue. We conjecture the transition estimation procedure in \Cref{algo:estimate} can be changed to other suitable online RL algorithms under other function approximation assumptions.

In summary, in this section, we address the first and third challenges. The first challenge is addressed mainly by a novel technique dubbed ``buffer periods" and the third one through nontrivial extensions to the LSIV-UCB framework. \re{By combining the loss from incentivizing a truthful mechanism and learning the underlying model to set reserve prices, we get the final \Cref{algo:KnownF}, which explores efficiently and achieves the following regret upper bound, and then addresses the second challenge of regret minimization.}

\subsection{Regret Bound When \texorpdfstring{$F(\cdot)$}{} is Known}
We introduce the following assumptions before we bound the regret. These regularity assumptions are commonly found in economics literature \citep{kleiber2003statistical, bagnoli2006log}.

\begin{assumption}\label{assumptionf}
    Market noise pdf $f$ is bounded, i.e. there exist constants $c_{1}, C_1$ such that $c_{1} \leq f\leq C_1$.
\end{assumption}

\begin{assumption}\label{assumptionfdiff}
    Market noise pdf $f$ is differentiable and its derivative is bounded. That is, there exists a constant $L$ such that $|f'|\le L$.
\end{assumption}

\begin{assumption}\label{logconcave}
    Market noise cdf $F(\cdot)$ and $1-F(\cdot)$ are log-concave.
\end{assumption}

At a high level, Assumptions~\ref{assumptionf} and~\ref{assumptionfdiff} ensure that the pdf $f$ is generally well-behaved, namely, bounded and smooth. Assumption~\ref{logconcave} is a popular assumption in economics that ensures the validity of the Myerson lemma~\citep{myerson1981optimal,kleiber2003statistical,bagnoli2006log}. We further remark that these assumptions are mild and are satisfied by commonly used distributions such as the truncated Gaussian distribution and the uniform distribution~\citep{golrezaei2019dynamic}.
\re{\begin{remark}
    We note that \Cref{logconcave} is in fact made redundant by \Cref{assumptionf} because we have a quite ``smooth'' distribution with bounded differential. Then, once we have a good estimation for the parameters, ``smooth'' $F(\cdot)$ leads to a good estimation of the reward function. Nevertheless, we retain this assumption as it streamlines our proof by avoiding discussion of market stability with multi-optimal reserve prices and getting bogged down in tedious regret decomposition. 
\end{remark}}

We are now ready to state our results.
If we set $\poly (\log K)=C_7+C_6 H \log^2K$ in \Cref{algo:estimate}, where constant $C_6$ is determined in \Cref{lem:estimater} and constant $C_7=B_8 H^{\frac{3}{2}} \log K$ with constant $B_8$ determined in \Cref{lem:omegaandQ}, then we have \Cref{thm:knownf}.

\begin{theorem}\label{thm:knownf}
Under \Cref{assumption:linearmdp}, \ref{assumptionf}, \ref{assumptionfdiff} and \ref{logconcave}, for any fixed failure probability $\delta\in(0,1)$, with probability at least $1 - \delta$, \Cref{algo:KnownF} achieves at most $\tilde \cO (\sqrt{H^5K})$ revenue regret, where $\tilde \cO (\cdot)$ hides only absolute constants and logarithmic terms.
\end{theorem}
\begin{proof}
See \Cref{sec:proofknownf} for a detailed proof.
\end{proof}

As we discussed previously, when $H = 1$, our result cannot be compared to existing works that focus on the stochastic bandit setting due to our need to explore the action space $\Upsilon$ (see~\citet{broder2012dynamic,drutsa2020reserve,drutsa2017horizon,golrezaei2019dynamic} for works that achieves $\tilde{\cO}(1)$ revenue regret in the stochastic bandit setting). The closest work we are aware of is~\citet{cesa2014regret}, which obtains a similar $\tilde{\cO}(\sqrt{K})$ regret in the adversarial multi-armed bandit setting, matched by our bounds.

\section{Unknown Market Noise Distribution}\label{sec:unknownF}
We now discuss when the market noise distribution is unknown. Recall from previous discussions that our second challenge lies in minimizing revenue regret when the market noise distribution is unknown. Existing techniques, similar to the one in~\citet{golrezaei2019dynamic}, incorporate pure exploration rounds to address the challenge, yet necessitate a $\tilde{\cO}(K^{2/3})$ revenue regret. In this section, we instead introduce a novel technique dubbed ``simulation", which eliminates the need for pure exploration rounds and achieves instead a $\tilde{\cO}(\sqrt{K})$ regret. While the first and third challenges have been previously addressed, the approaches in Section~\ref{sec:KnownF} also require careful adjustments, as the unknown market noise distribution makes a direct application of these approaches impossible. We detail our techniques and procedures in the rest of this section.

\subsection{CLUB Algorithm When \texorpdfstring{$F(\cdot)$}{} is Unknown}
\re{Similarly, there are three steps to do auction design when $F(\cdot)$ is unknown. First, we leverage \Cref{algo:pizero} and \Cref{algo:buffer_unknownF} to motivate an approximately truthful mechanism. Second, we utilize \Cref{algo:estimateunknownF} in coordination with the newly proposed \Cref{algo:simulation} to estimate the underlying MDP and set reserve prices. We motivate truthfulness through buffer periods and quantify the uncertainty by constructing corresponding ellipsoid bounds. Finally, we add up all these uncertainties and minimize regret with high probability.}

{\noindent \textbf{Addressing Challenge 1: Untruthfulness.}} When the market noise distribution is unknown, the techniques used in Section~\ref{sec:KnownF} cannot be applied directly, necessitating careful adaptations. We summarize the changes to these techniques, beginning by introducing~\Cref{algo:buffer_unknownF}, the counterpart to~\Cref{algo:buffer}, for when $F(\cdot)$ is unknown. The key difference lies in the optimization subroutine called in Line~\ref{algoline:buffer_unknownF_estimate}, which is required for addressing the third challenge when the market noise distribution $F(\cdot)$ is unknown.

\begin{algorithm}[htbp]
    \caption{Buffer Period with Unknown $F(\cdot)$}
    \label{algo:buffer_unknownF}
    \begin{algorithmic}[1]
      \STATE Receives buffer start $\mathtt{buffer.s}(\tilde{k} + 1) = k$ and end $\mathtt{buffer.e}(\tilde{k} + 1) = k + \frac{3\log K}{\log(1/\gamma)}$.
      \STATE Do nothing for all episodes $\mathtt{buffer.s}(\tilde{k} + 1) \leq k < \mathtt{buffer.e}(\tilde{k} + 1)$, i.e., do nothing during the buffer period before the end.
      \STATE At the end of the buffer period, update policy estimate $\pi_{\tilde{k} + 1}$ and Q-function estimate $\hat{Q}_h^{\pi_{\tilde{k} + 1}}(\cdot,\cdot)$ using {\Cref{algo:estimateunknownF}}, and then increment buffer period counter $\tilde{k} \gets \tilde{k} + 1$.\label{algoline:buffer_unknownF_estimate}
    \end{algorithmic}
\end{algorithm}

We then discuss~\Cref{algo:unKnownF}, a close variant of~\Cref{algo:KnownF}, whose biggest change lies in the update schedule in Line~\ref{algoline:unknownFupdate}. \Cref{algo:KnownF} maintains only an accurate estimate of the underlying MDP, achieved with a low switching cost style update schedule, which in turn deters untruthful bidding. On the other hand, \Cref{algo:unKnownF} needs accurate estimates of both the MDP and the market noise distribution $F(\cdot)$. We force additional updates whenever $k$ is a power of 2, also ensuring that $\hat{F}(\cdot)$ is close to $F(\cdot)$. As the number of updates remains in $\cO(\log K)$, the extraneous updates do not affect the regret asymptotically.

\begin{algorithm}[htbp]
    \caption{ \underline{C}ontextual-\underline{L}SVI-\underline{U}CB-\underline{B}uffer (CLUB) with Unknown $F$}
    \label{algo:unKnownF}
    \begin{algorithmic}[1]
      \STATE Initialize policy estimate $\pi_{0}$, buffer period counter $\tilde{k} = 0$, buffer period starting points $\mathtt{buffer.s}(0) = 1$, and buffer period end points $\mathtt{buffer.e}(0) = 1$.
      \FOR {episodes $k = 1, \ldots, K$}
      \STATE Execute mixture policy $\frac{1}{HK}\circ\pi_{\rm rand}+(1-\frac{1}{HK})\circ \pi_{\tilde{k}}$, collecting outcomes $q_{ih}^{\tau}$ and updating matrices $\Lambda_h^k\leftarrow \sum_{\tau = 1}^{k}\phi(x^{\tau}_h,\upsilon^{\tau}_h)\phi(x^{\tau}_h,\upsilon^{\tau}_h)^T + I$ for all $h \in [H]$.
      \STATE If there exists $h \in [H]$ such that $(\Lambda_h^{\mathtt{buffer.e}(\tilde{k})})^{-1} \not \preceq 2(\Lambda_h^k)^{-1}$ {or $\log_{2}(k)$ is an integer}, schedule a new buffer period starting at $\mathtt{buffer.s}(\tilde{k} + 1) = k$ and ending at $\mathtt{buffer.e}(\tilde{k} + 1) = k + \frac{3\log K}{\log(1/\gamma)}$ {using~\Cref{algo:buffer_unknownF}}\re{, and set $k\leftarrow \mathtt{buffer.e}(\tilde{k} + 1)$.}  \label{algoline:unknownFupdate_alt}
      \ENDFOR
    \end{algorithmic}
\end{algorithm}

Similar to Section~\ref{sec:KnownF}, these techniques, namely the buffer periods and the update schedule, ensure that the impatient bidders are sufficiently truthful. However, for estimating $\theta_{ih}$, as we do not know $F(\cdot)$, the optimization problem in~\Cref{algo:thetahat} no longer applies. Fortunately, we know that whenever $\pi_{\rm rand}$ is executed, assuming the bidders are truthful, $\Pr(q_{i}^\tau = 1) = \frac{1}{3N}(2 - \langle \phi(x_h^\tau, \upsilon_h^\tau),\theta\rangle)$ conditioned on $x_h^\tau, \upsilon_h^\tau$, as the bidder $i$ and the reserve price $\rho_{ih}^\tau$ are drawn uniformly at random. Leveraging this observation, we quickly realize that we can simply use the outcomes from when $\pi_{\rm rand}$ is executed to estimate the bidders' rewards, even when $F(\cdot)$ is unknown. Unfortunately, using the observation naively introduces the second challenge: minimizing revenue regret when $F(\cdot)$ is unknown.

{\noindent \textbf{Addressing Challenge 2: Regret Minimization.}} An intuitive way to incorporate the previous observation is to simply perform pure exploration rounds with $\pi_{\rm rand}$, similar to the technique in~\citet{golrezaei2019dynamic}. However, doing so incurs $\tilde{\cO}(K^{2/3})$ revenue regret, as $\pi_{\rm rand}$ does not set the reserve prices optimally and we are not exploring and exploiting simultaneously. To balance exploration and exploitation, we propose a new technique that we dub ``simulation", which allows us to continue exploiting with the mixture policy.

\begin{algorithm}[htbp]
    \begin{minipage}{1\textwidth}
    \caption{Simulation}
    \label{algo:simulation}
    \begin{algorithmic}[1]
        \FOR{$h = 1, \ldots, H$ and $\tau = 1, \ldots, K$}
            \STATE Generate virtual reserve prices $\tilde{\rho}_{ih}^\tau$ by selecting one bidder $i\in[N]$ uniformly at random. Let $\tilde{\rho}_{ih}^\tau \sim \textrm{Unif}([0, 3])$ and set all other reserve prices to infinity, i.e. $\tilde{\rho}_{jh}^\tau = \infty$ for all $j \neq i$.
            \STATE Use real bidding data $b_{ih}^\tau$ and simulated reserve prices $\tilde{\rho}_{ih}^\tau$ to simulate outcome $\tilde{q}_{ih}^\tau$ for all  $i \in [N]$, namely set $b_{ih}^\tau = \ind(b_{ih}^\tau \geq \tilde{\rho}_{ih}^\tau)$ for all $i \in [N]$.
        \ENDFOR
        \STATE Return the simulated outcomes $\{\tilde q_{ih}^k\}$.
    \end{algorithmic}
    \end{minipage}
\end{algorithm}

Here we introduce a new random variable $\tilde{q}_{ih}^\tau = \ind(b_{ih}^\tau \geq \tilde{\rho}_{ih}^\tau)$, where for each $h, \tau$ we select one $i \in [N]$ uniformly at random and then draw $\tilde{\rho}_{ih}^\tau$ from $\textrm{Unif}([0, 3])$. For all $j \neq i$ we set $\tilde{\rho}_{ih}^\tau$ to $\infty$. At a high level, $\tilde{q}_{ih}^\tau$ ``simulates" executing $\pi_{\rm rand}$: holding $x_h^\tau$ and $\upsilon_h^\tau$ constant, what would be the outcome if we were to act according to $\pi_{\rm rand}$ instead? As we do not need to execute $\pi_{\rm rand}$, revenue regret can be decreased. Furthermore, $\tilde{q}_{ih}^\tau$ still enjoys the same resilience towards untruthful reporting that $q_{ih}^\tau$ does. Indeed, when the bidder overbid or underbid by a small amount, the number of times $\tilde{q}_{ih}^{\tau}$ changes could be controlled effectively. 

More technically,~\Cref{algo:simulation} is critical for two reasons. First, the difference between $\hat{F}(\cdot)$ and $F(\cdot)$ decays at a rate of $O(1/\sqrt{K})$. If we simply use \Cref{algo:thetahat}, only replacing $F(\cdot)$ with $\hat{F}(\cdot)$, the estimation error is roughly on the order of $\tilde{\cO}(\sqrt{{\rm buffer.e}(\tilde k+1)})$ which precludes achieving $\tilde{\cO}(\sqrt{K})$ regret. Second, replacing $\tilde{q}_{ih}^\tau$ with $q_{ih}^\tau$ does not work, as we need to de-bias the estimator when we switch from $F(\cdot)$ to the uniform distribution induced by $\pi_{\rm rand}$. Even when the bidders report truthfully, we cannot guarantee that $\Pr({q}_{ih}^\tau = 1 \given x_{h}^\tau, \upsilon_{h}^\tau)$ could be related to $\frac{1}{3N}(1 + \langle \phi(x_h^\tau, \upsilon_h^\tau), \theta_{ih}\rangle)$. Consequently, it would be hard to ensure that when all bidders are truthful, the estimator $\hat{\theta}_{ih}^\tau$ would converge to $\theta_{ih}$.

{\noindent \textbf{Addressing Challenge 3: Nonlinear Revenue.}} 
With the first challenge addressed by carefully adjusting techniques in Section~\ref{sec:KnownF} and the second by the simulation technique detailed in~\Cref{algo:simulation}, we now discuss the third challenge: provably efficient reinforcement learning when the revenue is nonlinear and $F(\cdot)$ is unknown. We start with summarizing how we simultaneously estimate $\theta_{ih}$ and $F(\cdot)$ in the form of~\Cref{algo:Fandthetahat}.
\begin{equation}
    \label{algo:Fandthetahat}
    \begin{aligned}
        &\hat{\theta}_{ih}=\argmin_{\|\theta\|\le 2\sqrt{d}}\sum_{\tau=1}^{\mathtt{buffer.e}(\tilde{k}+1)} (3N\tilde{q}_{ih}^\tau - (1 + \langle\phi(x_h^\tau,\upsilon_h^\tau),\theta\rangle))^2,\\
        &\hat F(z)=\frac{1}{N\mathtt{buffer.e}(\tilde{k} + 1)H}\sum_{i=1}^N \sum_{\tau=1}^{\mathtt{buffer.e}(\tilde{k} + 1)} \sum_{h=1}^H \ind(b_{i\tau h}-1-\langle \phi_h^\tau ,\hat \theta_{ih}\rangle\le z).
    \end{aligned}
\end{equation} We note that we are simply using a histogram to estimate $F(\cdot)$ and, as we have successfully decoupled the estimation error of $F(\cdot)$ from that of $\theta_{ih}$, using histogram is sufficient for achieving $\tilde{\cO}(\sqrt{K})$ revenue regret. We then introduce~\Cref{algo:estimateunknownF}, whose key difference with ~\Cref{algo:estimate} lies in the added uncertainty due to $\hat{F}(\cdot)$ and the inclusion of the simulation subroutine. Similar to Section~\ref{sec:KnownF}, the procedure then provides us with sufficiently accurate policy and Q-function estimates, resolving our third and final challenge.
\begin{algorithm}[htbp]
    \begin{minipage}{1\textwidth}
    \caption{Estimation of $\hat Q_h^{\pi_{\tilde k+1}}(\cdot,\cdot)$ with Unknown $F(\cdot)$}
    \label{algo:estimateunknownF}
    \begin{algorithmic}[1]
        \STATE {Collect simulation outcome $\tilde q$ using \Cref{algo:simulation}.}
        \STATE {Estimate $\hat \theta_{ih}, \hat{F}(\cdot)$ using \Cref{algo:Fandthetahat}.}
        \STATE Estimate $\hat \mu_{ih}(\cdot,\cdot)\leftarrow \langle\phi(\cdot,\cdot),\hat \theta_{ih}\rangle$.
        \STATE Set reserve price $\hat \rho_{ih}(\cdot,\cdot)=\argmax_y{y(1-\hat F(y-1-\hat \mu (\cdot,\cdot)))}$.
        \STATE Estimate revenue $\hat R_h(\cdot,\cdot)\leftarrow \E[\max\{\tilde b_h^{-}(\cdot,\cdot),\hat \rho_h^+(\cdot,\cdot)\}\ind(\tilde b_h^{+}(\cdot,\cdot)\ge \hat \rho^+_h(\cdot,\cdot) )]$.
        \FOR {$h=H,\ldots,1$ \textbf{do} \COMMENT{Estimate $Q$-function and optimal policy.} \killdo}         
        \STATE   $\Lambda_h\leftarrow \sum_{\tau=1}^{\mathtt{buffer.e}(\tilde{k} + 1)}\phi(x_h^\tau,\upsilon_h^\tau)\phi(x_h^\tau,\upsilon_h^\tau)^T+\lambda I$.\hfill $\rhd$ We set $\lambda=1$ in this paper.
        \STATE   $\omega_h \leftarrow \Lambda_h^{-1} \sum_{\tau=1}^{\mathtt{buffer.e}(\tilde{k} + 1)} \phi(x_h^\tau,\upsilon_h^\tau)[\max_a \hat Q_{h+1} (x_{h+1}^\tau,a)]$.
        \STATE   $\hat Q^{\pi_{\tilde k+1}}_h(\cdot,\cdot)\leftarrow \min\{\omega_h^T\phi(\cdot,\cdot)+\hat R(\cdot,\cdot)+\poly _1(\log K) \|\phi(\cdot,\cdot)\|_{\Lambda_h^{-1}}+{\frac{\poly _2(\log K)}{\sqrt{\mathtt{buffer.e}(\tilde{k} + 1)}}},3H\}$
        \STATE   $\pi^\upsilon_{\tilde{k}+1,h}(\cdot)\gets\argmax_v \hat Q^{\pi_{\tilde k+1}}_h(\cdot,v)$.
        \STATE   $\pi^{\rho_i}_{\tilde{k}+1,h}(\cdot)\leftarrow \hat  \rho_{ih}(\cdot,\pi^a_{\tilde{k}+1,h}(\cdot)) $.
        \ENDFOR
        \STATE Return $\{\hat Q^{\pi_{\tilde k+1}}_h(\cdot,\cdot)\}_{h=1}^H$ and $\{\pi_{\tilde k+1,h}(\cdot)\}_{h=1}^H$.
    \end{algorithmic}
    \end{minipage}
\end{algorithm}

In summary, we have addressed all three challenges when the market noise distribution is unknown. The first challenge is resolved by carefully adjusting the techniques introduced in Section~\ref{sec:KnownF}, ensuring that they are still valid when $F(\cdot)$ is unknown. For the second challenge we feature a novel technique dubbed ``simulation" that allows us to ``simulate" pure exploration rounds without actually executing them, reducing revenue regret. For the third challenge, we build off of the simulation technique and introduce a new estimation procedure for jointly estimating $F(\cdot)$ and $\theta$.

\subsection{Regret Bound of CLUB Algorithm When \texorpdfstring{$F(\cdot)$}{} is Unknown}
We now argue that~\Cref{algo:unKnownF} achieves $\tilde \cO (\sqrt{K})$ regret. We begin with a slight detour, making a basic assumption on the hypothesis class for $F(\cdot)$.

\begin{assumption}\label{assumption:coveringnumber}
The market noise distribution $F(\cdot)$ belongs to a distribution family $\mathcal{F}$. 
\end{assumption}

We further let $\cN_{\epsilon}(\mathcal{F})$ be the $\epsilon$-covering number of $\cF$ with respect to the metric that ${\rm dist}(F,G)=\sup_x|F(x)-G(x)|$. We now have our main theorem when the noise distribution is unknown.
If we let $\poly _1(\log K)=C_{15}+C_{13}H \log^2 K $ and $\poly _2(\log K)=C_{14} H^2\log^4K$ in \Cref{algo:estimateunknownF}, where $C_{15}=D_7 H^\frac{3}{2}$ and the constant $D_{7}$ is determined in \Cref{lem:unknownomegaandQ}, constants $C_{13}$ and $C_{14}$ are determined in \Cref{unknownboundR}, we would attain the following regret guarantee.

\begin{theorem}\label{thm:unknownf}
Under Assumptions \ref{assumption:linearmdp}, \ref{assumptionf}, \ref{assumptionfdiff}, \ref{logconcave} and~\ref{assumption:coveringnumber}, when $F(\cdot)$ is unknown, for any fixed failure probability $\delta\in(0,1)$, \Cref{algo:unKnownF} achieves at most $\tilde \cO (H^3\sqrt{K}+H^{2.5}\sqrt{ K \log \cN_{1/K}(\cF) })$ regret with probability at least $1-\delta$ in the worst case, where $\tilde \cO (\cdot)$ hides only absolute constants and logarithmic terms.
\end{theorem}
\begin{proof}
See \Cref{sec:proofunknownf} for a detailed proof.
\end{proof}

We highlight that when $\cN_{1/K}(\cF)$ is polynomial in $K$, an implicit assumption found in \citet{kong2021online,foster2021statistical,jin2021bellman},~\Cref{thm:unknownf} shows that~\Cref{algo:unKnownF} achieves $\tilde{\cO}(\sqrt{K})$ regret, improving over revenue regret guarantees found in~\citet{amin2014repeated, golrezaei2019dynamic} with only mild additional assumptions on the nonparametric hypothesis class $\cF$. 
\re{Our result is able to beat the well-known $\Omega(K^{2/3})$ revenue lower bound in \citet{kleinberg2003value} with the help of Assumptions~\ref{assumptionf} and~\ref{assumptionfdiff} for similar but not totally the same scenarios to be fair.}
Nevertheless, as argued previously, these assumptions are satisfied by widely-used parametric distribution families such as the uniform distribution and the truncated normal distribution~\citep{golrezaei2019dynamic}, hence our result still remains broadly applicable. \re{The way \citet{kleinberg2003value} constructs regret lower bound is to find a special case containing no information. As they say ``the expected
revenue per buyer is a constant independent of the offer price outside the interval of good prices'', it provides nothing useful for learning. However, with \Cref{assumptionf}, it guarantees the information in each exploration and partial out this extreme situation.}

\re{Finally, we highlight that both bounds in \Cref{sec:KnownF,sec:unknownF} match corresponding lower bounds with respect to $K$.
    From the $\Omega(\sqrt{K})$ lower bound in \citet{jin2020provably}, we directly know that results in \Cref{thm:knownf,thm:unknownf} match corresponding regret lower bounds, as the problem in \citet{jin2020provably} is a subproblem of our problem.
}

\section{Proof Sketch}\label{proof sketch}
Before sketching out the proof techniques, we take a slight detour and discuss how revenue regret could be decomposed. Recall that $\pi_{\tilde{k}}$ denotes the optimistic policy estimate maintained from episode $\mathtt{buffer.e}(\tilde{k}) + 1$ to $\mathtt{buffer.s}(\tilde{k} + 1)$, namely the estimate from the end of the $\tilde{k}$-th buffer period to the start of the $(\tilde{k} + 1)$-th. We also recall that $\pi^*$ is the optimal item choice and pricing policy when the seller knows the bidders' reward functions, the transition kernel $\PP$ and the market noise distribution $F(\cdot)$ beforehand, and we use $V^{\pi^*}$ to denote the revenue's V-function for the optimal policy $\pi^*$. 

We now introduce several new notations that will be used in the rest of the section. We use $\pi_k$ to denote the policy executed at episode $k$. Intuitively, the policy $\pi_k$ consists of some steps in which the corresponding $\pi_{\tilde{k}}$ is executed and some steps where $\pi_{\rm rand}$ is executed. Let $\ind(k \in \mathtt{buffer})$ indicate the event that there exists some integer $\tilde{k}$ such that $k \in [\mathtt{buffer.s}(\tilde{k}), \mathtt{buffer.e}(\tilde{k})]$, i.e. the episode $k$ is within a buffer period. To better highlight the effect of untruthfulness, we let $\tilde{V}$ denote the optimistic V-function estimate if all bidders were to report truthfully.
\subsection{Regret Decomposition}
The regret can be decomposed into the following five parts.
\begin{enumerate}
    \item \textbf{$\Delta_1=\sum_{k=1}^K [V_1^{\pi^*}(x_1)- \tilde V_1^{\pi_k}(x_1)]\ind(\pi_k=\pi_{\tilde k} \ {\rm and} \  k\not \in {\mathtt{buffer}})$}. The term $\Delta_1$ is due to the seller not knowing the bidders' reward functions and the underlying transition dynamics of the MDP. The term is nonzero even if we were to assume that all bidders report truthfully due to the uncertainty of the environment.
    \item  \textbf{$\Delta_2=\sum_{k=1}^K [V_1^{\pi^*}(x_1)- V_1^{\pi_k}(x_1)]\ind(k\in \mathtt{buffer})$}. The second term comes from the buffer periods, which cause suboptimality as we intentionally delay the policy update in order to further punish untruthfulness. While conducive to more truthful reports, a delayed update schedule induces regret as the policy estimate is stale during these buffer periods.
    \item \textbf{$\Delta_3=\sum_{k=1}^K [V_1^{\pi^*}(x_1)- V_1^{\pi_k}(x_1)]\ind({\rm exists } \, h\, {\rm such\, that}\  \pi_{k,h}=\pi_{\rm rand}\ {\rm and} \  k\not \in {\mathtt{buffer}})$}. The third term $\Delta_3$ is caused by $\pi_{\rm rand}$, as it sets reserve prices and chooses items entirely randomly.
    \item \textbf{$\Delta_4=\sum_{k=1}^K [V_1^{\pi^*}(x_1)- V_1^{\pi_k}(x_1)]\ind( k\in \mathtt{L} \ {\rm and} \  k\not \in \mathtt{buffer})$.} We only provide intuition behind the term $\mathtt{L}$ and defer its precise mathematical definition to~\Cref{eqn:lie_set_defn_knownF} for when $F(\cdot)$ is known and~\Cref{eqn:defn_lie_unknownF} for when $F(\cdot)$ is not. The term $\mathtt{L}$ is a collection of episode indices where the bidders' untruthful bids alter the outcome of the multi-phase auction, through either $q_{ih}$ or $\tilde{q}_{ih}$. At a high level, while we could measure the revenue suboptimality of the selected reserve prices if the bidders are truthful, the seller's revenue could be harmed arbitrarily by bidders who underbid/overbid so much that the auction's outcome itself is altered. The term $\Delta_4$ then measures the effect of the changed outcomes due to untruthful bidding.
    \item \textbf{$\Delta_5=\sum_{k=1}^K [\tilde V_1^{\pi_k}(x_1)-  V_1^{\pi_k}(x_1)]\ind(\pi_k=\pi_{\tilde k} \ {\rm and} \  k\not \in \mathtt{buffer})$}. Compared to $\Delta_4$, which measures the effect of changed outcomes due to untruthfulness, $\Delta_5$ measures the effect of changed bids due to untruthfulness. Intuitively, a bidder who overbids/underbids a small amount would not affect the auction's outcome, but could change the amount the seller charges slightly. We measure the effect with $\Delta_5$.
\end{enumerate}

With easy algebra calculation, we have the following proposition.
\begin{proposition}
With $\Delta_1$ to $\Delta_5$ defined as above, it holds that ${\rm Regret}\le \Delta_1+\Delta_2+\Delta_3+\Delta_4+\Delta_5$.
\end{proposition}
\begin{proof}
Since our benchmark is the maximum revenue when everything is common knowledge, it holds that $V_1^{\pi^*}(x_1)\ge V_1^{\pi_k}(x_1)$ at any time. It is because that $V^{\pi^*}$ is no less than the revenue achieved when existing hidden information with any policy due to its optimality.

Since $\Delta_1+\Delta_5=\sum_{k=1}^K [V_1^{\pi^*}(x_1)-  V_1^{\pi_k}(x_1)]\ind(\pi_k=\pi_{\tilde k} \ {\rm and} \  k\not \in \mathtt{buffer})$ and $\ind(k\in \mathtt{buffer})+\ind(\pi_k=\pi_{\tilde k} \ {\rm and} \  k\not \in {\mathtt{buffer}})+\ind({\rm exists } \, h\, {\rm such\, that}\  \pi_{k,h}=\pi_{\rm rand}\ {\rm and} \  k\not \in {\mathtt{buffer}})+\ind( k\in \mathtt{L} \ {\rm and} \  k\not \in \mathtt{buffer})\ge 1$, it holds that 
\[
{\rm Regret}=\sum_{k=1}^K V_1^{\pi^*}(x_1)-  V_1^{\pi_k}(x_1)\le \Delta_1+\Delta_2+\Delta_3+\Delta_4+\Delta_5,
\]which ends the proof.
\end{proof}

\subsection{Proof Techniques}
With the sources of revenue regret sketched out, we summarize the high-level intuition behind our proof, which mainly comprises the following steps.

\paragraph{Step 1: Limit the magnitude of untruthful reporting.}
As we discussed in~\Cref{sec:KnownF}, reducing the frequency at which we update the policies and including the buffer periods forces bidders to wait before they can gain from untruthful reporting. When the bidders are impatient, the amount they can gain from untruthful reports is then upper-bounded. With the help of $\pi_{\rm rand}$, we are also always punishing the bidders for untruthful reports. Combining the two halves, we can control the total amount by which bidders overbid or underbid, as overbidding or underbidding too much would decrease their utilities. Moreover, by directly controlling the ``amount" of overbidding and underbidding, we are able to upper bound $\Delta_5$, the part of the revenue regret due to untruthfulness. 
The step corresponds to \Cref{overbid} in \Cref{sec:proofknownf}.

\paragraph{Step 2: Control the number of times $q_{ih}^{k}$ change due to untruthfulness.} 
Since we are using $q_{ih}^{k}$, as opposed to $b_{ih}^{k}$, to learn the bidder's reward functions, to ensure the estimates' accuracy, we only need to show that the $q_{ih}^{k}$'s are close to their values when the bidders are truthful. As $q_{ih}^{k}$'s are outcomes of an auction, bidders need to overbid or underbid by a significant amount in order to alter $q_{ih}^{k}$. Combined with the previous step, we can limit the number of times $q_{ih}^{k}$ is altered due to untruthful behavior. With the number of changes controlled, we can also control $\Delta_4$. The step corresponds to \Cref{lem:bound_lie} and \Cref{lem:unknownbound_lie}.

\paragraph{Step 3: Prove the estimates of personal parameters and noise distribution are good.}
Having shown that the bidders provide us with sufficiently truthful reports, we connect our work to RL with \emph{generalized
} linear function approximation in \citet{wang2020optimism} to show $\hat{\theta}_{ih}$ is sufficiently accurate and apply the Dvoretzky–Kiefer–Wolfowitz inequality to show $\hat{F}(\cdot)$ is sufficiently accurate~\citep{dvoretzky1956asymptotic}. When the market noise distribution is known, the step corresponds to \Cref{lem:glm} and \Cref{lem:diffmu}. When the market noise distribution is unknown, the step corresponds to \Cref{lem:unknownglm}, \Cref{lem:unknowndiffmu} and \Cref{lem:boundF}.

\paragraph{Step 4: Prove $\hat R(\cdot,\cdot)\approx  R(\cdot,\cdot)$ and extend LSVI-UCB to non-linear reward function.} 
By applying Taylor expansion to the revenue function $R_{h}$, we relate the accuracy of $\hat{R}(\cdot, \cdot)$ to the accuracy of $\hat{\theta}_{ih}$, which is shown to be accurate in the previous step. We can then show our policy estimate $\pi_{\tilde{k}}$ is approximately optimal with standard LSVI-UCB analysis. Steps 3 and 4 then combine to control $\Delta_1$. The step corresponds to \Cref{lem:estimater}, \Cref{lem:lsvi}, \Cref{unknownboundR} and \Cref{lem:unknownlsvi}.

\paragraph{Step 5: Limit the effects of $\pi_{\rm rand}$ and buffer periods.} 
We finally control the revenue regret due to $\pi_{\rm rand}$ and buffer periods. A key observation is that the number of times in which $\pi_{\rm rand}$ is executed and the length of the buffer periods are all in $\tilde{\cO}(1)$ and hence do not harm our regret asymptotically. Consequently, terms $\Delta_2$ and $\Delta_3$ are controlled effectively. For $\pi_{\rm rand}$, the step corresponds to \Cref{numofpi0}, and as for buffer periods, it corresponds to \Cref{lem:buffer} and \Cref{lem:unknownbuffer}.

\re{
\section{Numerical Experiments}
Here, we present numerical simulations to compare the performance of \Cref{algo:unKnownF} with several baseline policies in different settings\footnote{Code is available at \url{https://github.com/Air-8/SPA_CLUB}.}. To be specific, we compare the performances of CLUB (i.e. \Cref{algo:unKnownF}), SCORP~\citep{golrezaei2019dynamic} and NPAC-S~\citep{golrezaei2023incentive} in contextual bandit settings (i.e. $H=1$) and the performances of CLUB and NPAC-S in MDP settings. In all experiments, we assume that the noise distribution $F(\cdot)$ is unknown. The numerical experiment written in Python 3.10.9 runs on a laptop with an Apple M2 CPU. All three algorithms use less than 30 seconds to calculate 10000 episodes, which shows their practicality in reality. We delay more details and further robustness experiments in \Cref{app:exp}.
}
\begin{figure}[!htbp]
    \centering
    \begin{subfigure}[b]{0.3\textwidth}
        \includegraphics[width=\textwidth]{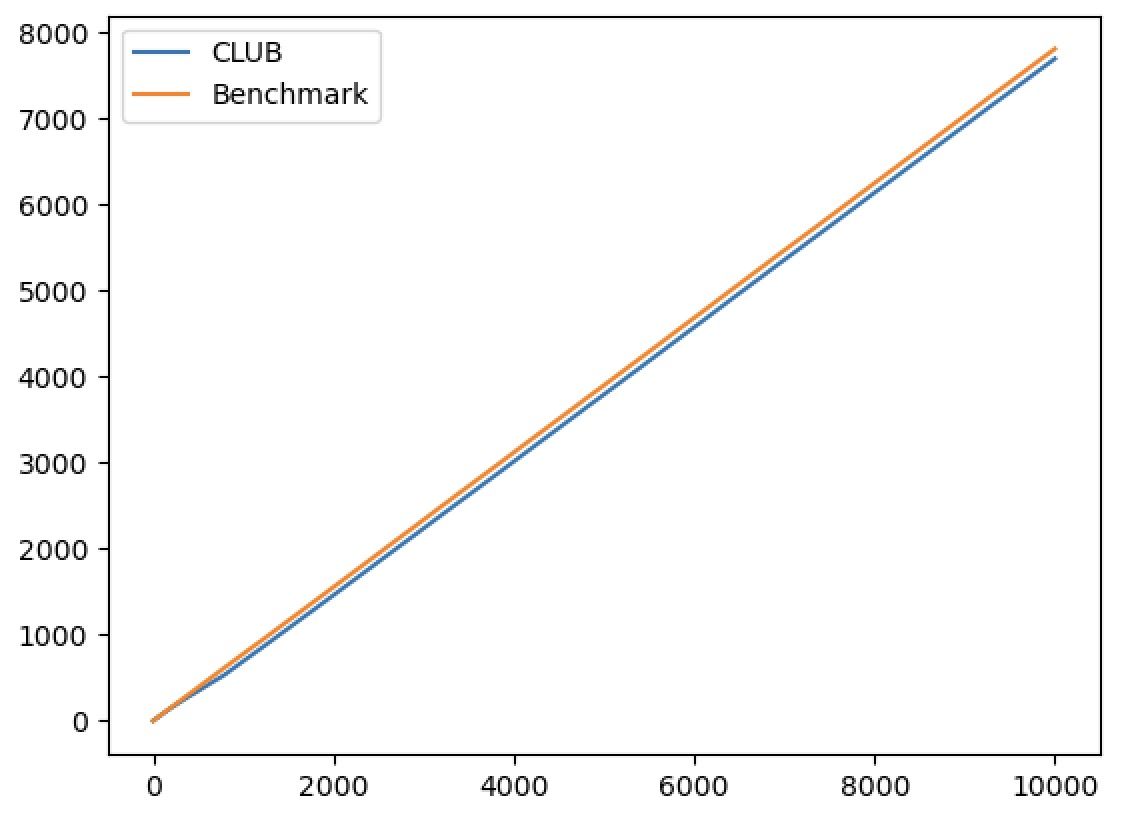}
        \caption{The performance of CLUB against the benchmark.}
        \label{fig:sub1}
    \end{subfigure}
    \hfill
    \begin{subfigure}[b]{0.3\textwidth}
        \includegraphics[width=\textwidth]{ 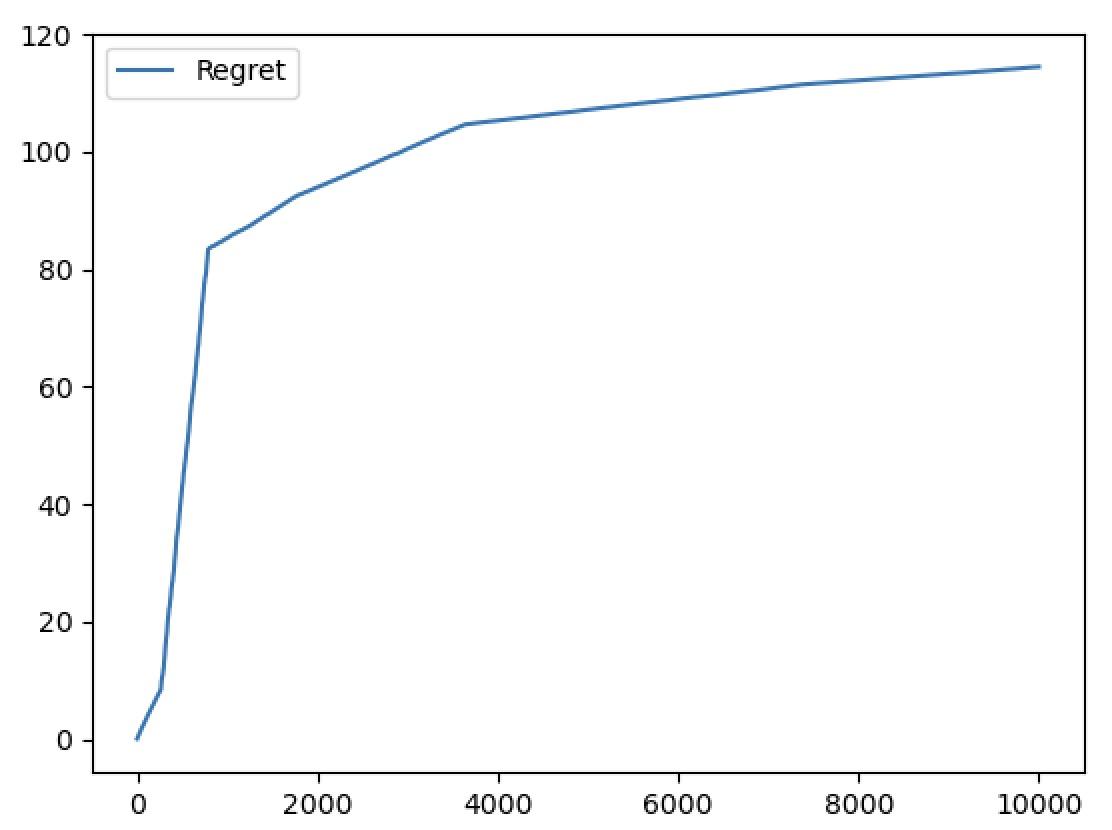}
        \caption{The regret accumulation of CLUB.}
        \label{fig:sub2}
    \end{subfigure}
    \hfill
    \begin{subfigure}[b]{0.3\textwidth}
        \includegraphics[width=\textwidth]{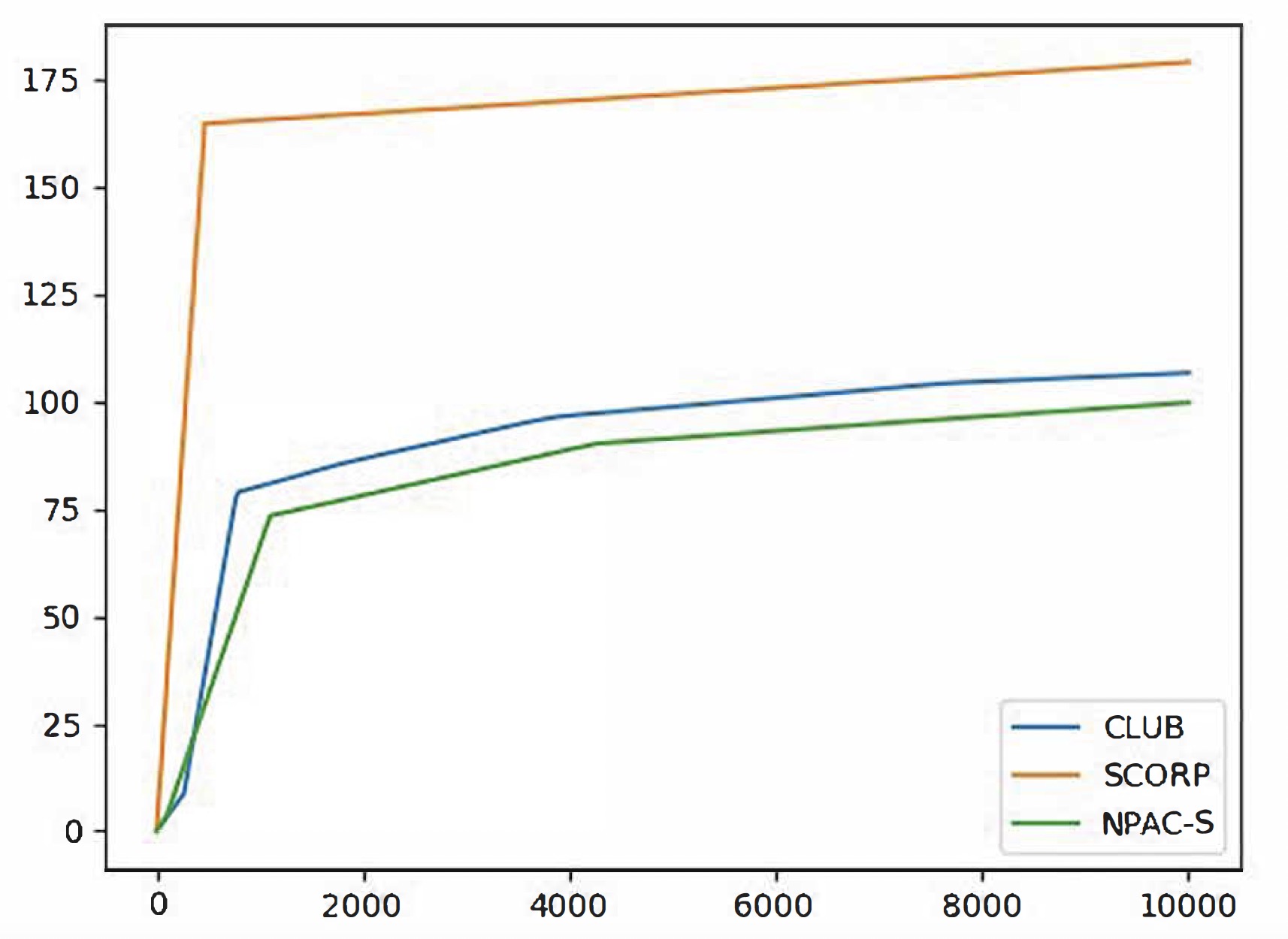}
        \caption{The average performances of three algorithms.}
        \label{fig:sub3}
    \end{subfigure}
    \caption{\re{Experiment results for the contextual bandit setting: \Cref{fig:sub1} compares the revenue achieved by CLUB and benchmark (the maximum revenue when everything is common knowledge), showing CLUB obtains more than 98\% revenue. \Cref{fig:sub2} shows the sublinear regret associated with our CLUB algorithm as the curve trend is below linear. \Cref{fig:sub3} exhibits that CLUB is comparable with NPAC-S, overwhelming SCORP.} 
    }
    \label{fig:contextual}
\end{figure}

\re{In contextual bandit setting, we set $K=10000$, $\gamma=0.9$ for each setting and repeat the procedure for $n=30$ trails for each algorithm. We show results in \Cref{fig:contextual}. \Cref{fig:sub1,fig:sub2} show results in one trial, where we find that CLUB can obtain more than 98\% revenue compared with the benchmark, in which the underlying model is common knowledge. At the same time, \Cref{fig:sub2} testifies $\tilde O(\sqrt{K})$-shaped regret. In \Cref{fig:sub3}, we show the average regrets among all 30 trials of these three different algorithms. The average regrets in 30 trials are 106.62, 178.96 and 99.69 respectively. 
Because the scale of regret depends on the specific instances, and most of the variance comes from randomness in the instances themselves, we report the regret of each trial and the number of winning times in \Cref{tab:expbandits}, which better reflects the performance of the algorithms.
As for the number of winning times, CLUB wins 15 times while NPAC-S wins 14 times. SCORP only wins once. Therefore, we conclude that the performances of CLUB and NPAC-S are comparable, overwhelming the performance of SCORP. Since SCORP doesn't work well even in contextual bandit settings, we only compare CLUB and NPAC-S under MDP.}

\re{In the MDP setting, we incorporate $K=10000$, $H=2$, $\gamma=0.9$ and also conduct $n=30$ trails for both two algorithms. We show the corresponding results in \Cref{fig:MDP}. Our CLUB can obtain more than 98\% revenue (c.f. \Cref{fig:MDP+sub1}) against the benchmark which highlights its great performance. In \Cref{fig:MDP+sub2}, it's clear to see the $\tilde O(\sqrt{K})$-shaped regret. Among all 30 trails, CLUB wins all 30 times. As for average regrets, it's 203.07 for CLUB and 756.31 for NPAC-S. Therefore, we can conclude that under the MDP setting, CLUB sufficiently works better than NPAC-S. Together with experiments under contextual bandits, our experiments are in favor of CLUB algorithms, which show the importance of our newly proposed techniques.}
\begin{figure}[!htbp]
    \centering
    \begin{subfigure}[b]{0.3\textwidth}
        \includegraphics[width=\textwidth]{ 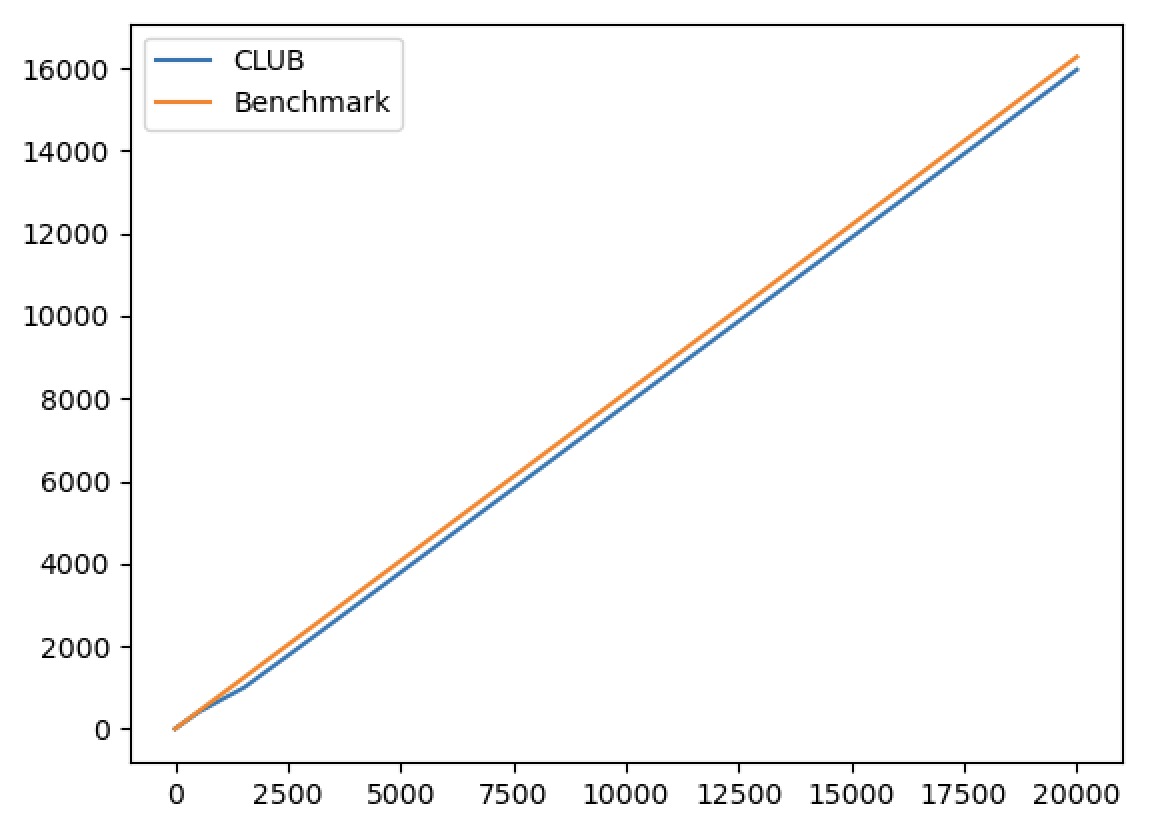}
        \caption{The performance of CLUB against the benchmark.}
        \label{fig:MDP+sub1}
    \end{subfigure}
    \hfill
    \begin{subfigure}[b]{0.3\textwidth}
        \includegraphics[width=\textwidth]{ 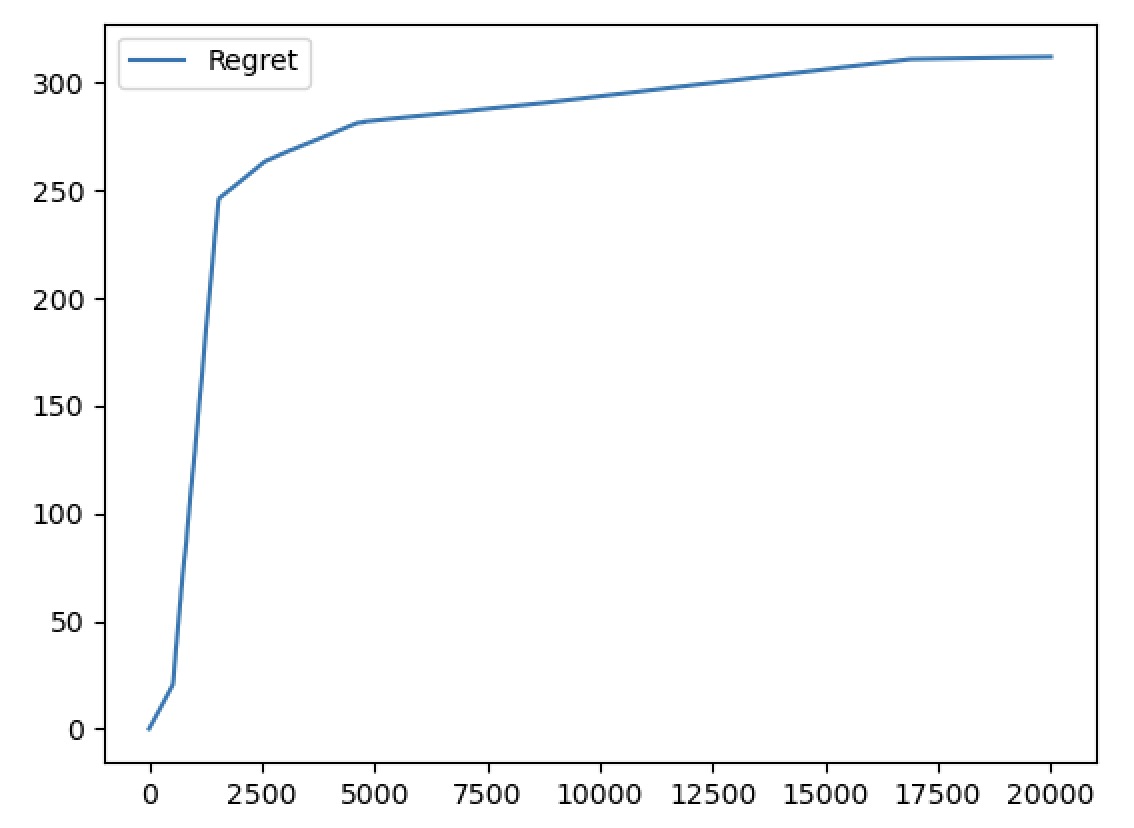}
        \caption{The regret accumulation of CLUB.}
        \label{fig:MDP+sub2}
    \end{subfigure}
    \hfill
    \begin{subfigure}[b]{0.3\textwidth}
        \includegraphics[width=\textwidth]{ 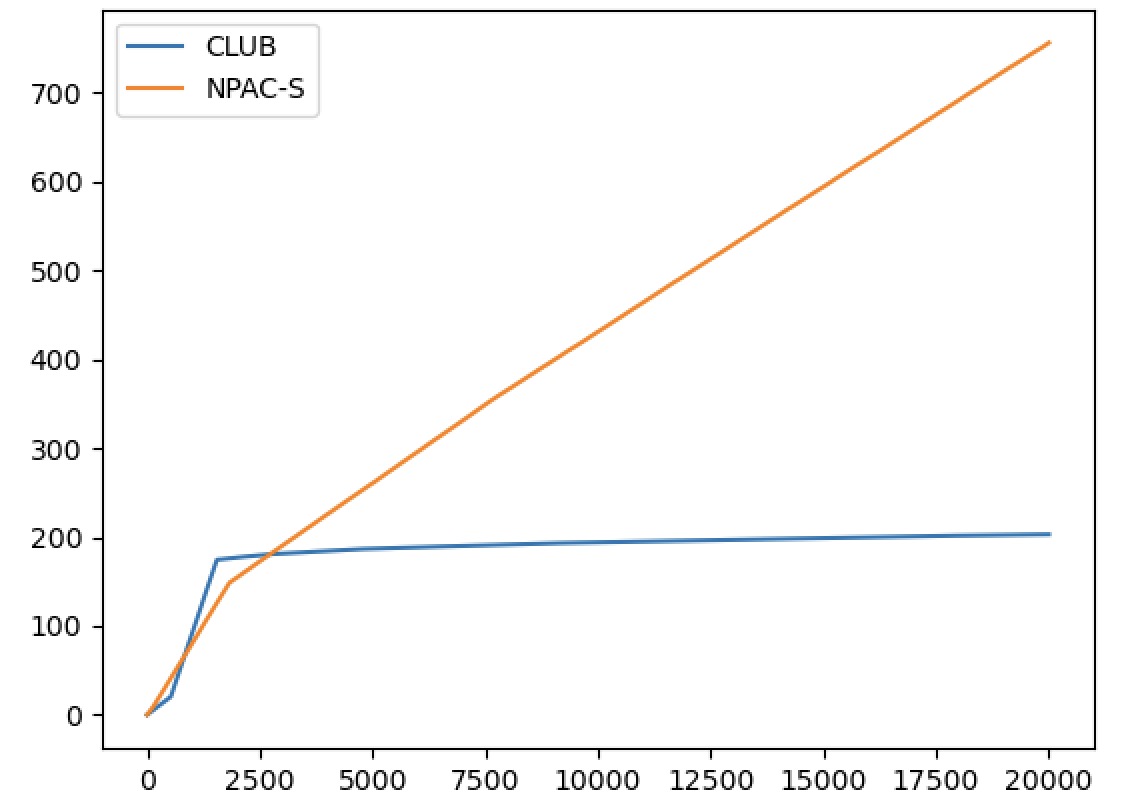}
        \caption{The average performances of two algorithms.}
        \label{fig:MDP+sub3}
    \end{subfigure}
    \caption{\re{Experiment results for the MDP setting: \Cref{fig:MDP+sub1} compares the revenue achieved by CLUB and benchmark (the maximum revenue when everything is common knowledge), showing CLUB obtains more than 98\% revenue. \Cref{fig:MDP+sub2} shows the sublinear regret associated with our CLUB algorithm, as the curve trend is below linear. \Cref{fig:MDP+sub3} exhibits that compared with NPAC-S, CLUB has less regret, testifying to its optimality.}}
    \label{fig:MDP}
\end{figure}

\section{Conclusion and Discussion}
In this paper, we propose a multi-phase second-price auction mechanism based on reinforcement learning.
We highlight that when market noise distribution is unknown, our algorithm achieves $\tilde \cO (H^3\sqrt{K})$ regret, improving upon the $\tilde{\cO}(K^{2/3})$ guarantees in~\citet{amin2014repeated,golrezaei2019dynamic}, using a new method to deal with unknown distribution. Our work is also the first to introduce the notion of ``buffer periods'', a concept crucial to bringing existing techniques in the bandit setting to the more general MDP setting.

Questions raise themselves for future explorations. Is it possible to further generalize our results to RL with general function approximation under bounded Bellman Eluder dimensions~\citep{jin2021bellman}? Can we optimize the dependence on horizon $H$ and feature dimension $d$? 
\re{Is it possible to consider a continuous representation of buyers and form a mean-field game?} 
We leave these interesting questions as potential next steps.

\acks{The authors thank the anonymous Reviewers and the AE for their very constructive comments.
Part of the work was done when Rui Ai was an undergraduate at Peking University, and he was partially supported by the elite undergraduate training program of School of Mathematical Sciences at Peking University. 
Zhaoran Wang acknowledges National Science Foundation (Awards 2235451, 2225087, 2211210, CAREER-2048075, 2015568, 2008827, 1934931/2216970), Simons Institute (Theory of Reinforcement Learning), Amazon, J.P. Morgan, Two Sigma, and Google for their support.
Zhuoran Yang
acknowledges Simons Institute (Theory of Reinforcement
Learning).
}


\newpage

\appendix

\re{
\section{Detailed Comparison with \citet{golrezaei2019dynamic}}\label{app:comp}
There are three different models and corresponding algorithms named CORP, CORP-\uppercase\expandafter{\romannumeral2}, and SCORP, respectively, in \citet{golrezaei2019dynamic}. We compare them with our model one by one.}

\re{CORP considers a contextual bandit setting with known noise distribution achieving $\tilde O(1)$ regret. However, as we mentioned before, accommodating underlying MDP, $\Omega(\sqrt{K})$ regret lower bound is inevitable. In \Cref{sec:KnownF}, we propose our optimal CLUB algorithm matching the lower bound. }

\re{CORP-\uppercase\expandafter{\romannumeral2} considers an unknown but parametric noise distribution and achieves $\tilde O(\sqrt{K})$ regret. However, in \Cref{sec:unknownF}, we consider an unknown and non-parametric noise distribution. Therefore, compared with the setting for CORP-\uppercase\expandafter{\romannumeral2}, our model is strictly much harder for the following two-fold reasons. We need to consider the extra MDP structure and non-parametric noise distribution. Moreover, since CORP-\uppercase\expandafter{\romannumeral2} doesn't have enough horizons to explore, it doesn't work well under our MDP setting and cannot achieve its original $\tilde O(\sqrt{K})$ regret. }

\re{SCORP considers time-varying and non-parametric noise distribution achieving $\tilde O(K^{2/3})$ regret. We share the similarity that both of these settings can't be parameterized, which means that we lose the opportunity to utilize some concentration inequalities directly, and we need to bypass these obstacles to achieve sublinear regrets. There are two main differences between our model and SCORP. First, the underlying MDP makes our problem harder than that of SCORP. Second, we consider a fixed noise distribution and use a different benchmark, making these two models not directly comparable. 
As a result, our algorithm achieves $\tilde{\cO}(\sqrt{K})$ regret with mild additional assumptions on the shape of $F(\cdot)$ (c.f. \Cref{assumptionf} and \ref{assumptionfdiff}). Although it is hard to compare the difficulties between our setting and SCORP in strict order, we believe they have a similar degree of difficulty. As we mentioned in \Cref{subsec:related_work}, the work \citep{amin2014repeated} explores a scenario with a non-parametric yet fixed distribution setting, experiencing a regret of $\tilde O(K^{2/3})$. This observation suggests that the primary challenge might arise from the non-parametric nature of the problem, as opposed to the time-varying setting. Moreover, we should highlight that an $\tilde{O}(K^{2/3})$ regret is inevitable even though the distribution is fixed corresponding to a saddle point for SCORP, as they spend ``too many'' episodes to explore while we don't ``waste'' time to do pure exploration so that balance the exploration-exploitation tradeoff better and achieve better regret bounds. Objectively, our method will suffer $\Omega(K^{2/3})$ regret lower bound for a time-varying model, and it's of independent interest for future research.}

\section{Omitted Proof in \texorpdfstring{\Cref{sec:KnownF}}{sec:KnownF}} \label{sec:proofknownf}
In this section, we show some useful lemmas in order to prove theorems in \Cref{sec:KnownF}. We organize the section as follows. Firstly, we introduce lemmas to bound the effect of untruthful bidding. Then, we will show that we are able to estimate unknown parameters accurately. Finally, combining them leads to bounded regret with high probability.
\subsection{Useful Lemmas for Proving \texorpdfstring{\Cref{thm:knownf}}{thm:knownf}}
Now, we begin to prove our conclusions. First of all, we show the following lemma to bound the number of buffers. Buffer episodes represent those episodes in buffer periods.
\begin{lemma}\label{lem:buffer}
Under \Cref{assumption:linearmdp} about linear MDP, it holds that the number of buffer episodes is not larger than $\frac{3HC_2\log^2 K}{\log \frac{1}{\gamma}}$. Then, the number of corresponding steps is not larger than $\frac{3H^2C_2\log^2 K}{\log \frac{1}{\gamma}}$, where $C_2$ is a constant that only depends on $d$ and $\lambda$.
\end{lemma}


Because of the existence of a buffer, the bidder will not overbid or underbid a lot in the other episodes. Then, we have the following lemma.
\begin{lemma}\label{overbid}
Apart from the buffer periods, a rational bidder won't overbid or underbid for more than $\frac{3H\sqrt{2N}}{K\sqrt{1-\gamma}}$, denoted by $\frac{C_3 H}{K}$.
\end{lemma}

Then we define $\mathtt{L}$ as the number of steps the bidder doesn't bid his true value and change the outcome of the auction. Then, it holds the following lemma with the help of \Cref{overbid}. We formalize the definition of $\mathtt{L}$ for any given $i,h$ as follows.
\begin{equation}
\label{eqn:lie_set_defn_knownF}
    \mathtt{L}=\{k: \ind(v_{ih}^kw>\max\{b_{-ih}^{k+},\rho_{ih}^k\})\neq \ind(b_{ih}^k>\max\{b_{-ih}^{k+},\rho_{ih}^k\})\}.
\end{equation}
\begin{lemma}\label{lem:bound_lie}
    With probability at least $1-\delta$, it holds that for any given $i$, $h$
    \[
    \mathtt{L}\le  \frac{3HC_2\log^2 K}{\log \frac{1}{\gamma}}+4{C_1C_3 H}+8\log (\frac{2NH}{\delta}) \le C_4 H \log^2K,
    \]
    where $C_4$ is a constant independent of $K$ and $H$.
\end{lemma}

Now, we bound the number of steps we use $\pi_{\rm rand}$ instead of $\pi_{\tilde{k}}$. Especially, we regard $\pi_{\rm rand}$ as the policy used in the situation that happens with probability $\frac{1}{KH}$.

\begin{lemma}\label{numofpi0}
With probability at least $1-\delta$, the number of steps using $\pi_{\rm rand}$ is smaller than $\max\{4,1+\frac{4}{3}\log \frac{1}{\delta}\}$.
\end{lemma}


Now, we will show the wedge between $\hat \mu_{ih}(\cdot,\cdot)$ and $\mu_{ih}(\cdot,\cdot)$ for any bidder $i$ and step $h$. It holds the following lemma.
\begin{lemma}\label{lem:glm}
We use $\theta_{ih}^*$ to denote the true parameter and $\hat \theta_{ih}$ to represent the outcome from \Cref{algo:thetahat} in episode $\mathtt{buffer.e} (\tilde{k})$. Therefore, under \Cref{assumptionf} and \Cref{assumptionfdiff}, for any $i$ and $h$, it holds the following union bound that $C_5$ is a constant and 
\[
\sqrt{(\hat \theta_{ih}-\theta_{ih}^*)^T\Lambda^{\mathtt{buffer.e} (\tilde{k})}(\hat \theta_{ih}-\theta_{ih}^*)} \le C_5 \sqrt{H}\log K,
\]
with probability at least $1-\delta$, conditional on \textrm{Good Event} $\mathscr{E}$.
\end{lemma}

Then, we are ready to have the bound for $\hat \mu$. It holds the following lemma:
\begin{lemma}\label{lem:diffmu}
Conditional on \textrm{Good Event} $\mathscr{E}$, it holds that
\[
|\hat \mu_{ih}^k(\cdot,\cdot)-\mu_{ih}^k(\cdot,\cdot)|\le C_5 \sqrt{H}\log K \|\phi(\cdot,\cdot)\|_{({\Lambda_h^{\mathtt{buffer.e} (\tilde{k})}})^{-1}},
\]where $\mathtt{buffer.e} (\tilde{k})$ is the last episode using \Cref{algo:thetahat} before episode k, similarly hereinafter.
\end{lemma}

Now, we focus on the gap between $R(\cdot,\cdot)$ and the estimate $\hat R(\cdot,\cdot)$. We are ready to show the following lemma.

First of all, we introduce some notations. $R_h^k(\cdot,\cdot)=\sum_{i=1}^N \E [\max \{r_{ih}^{k-} ,\alpha_{ih}^{k*}\}\ind (r_{ih}^k\ge \max \{r_{ih}^{k-} ,\alpha_{ih}^{k*}\} )]$ and $\hat R_h^k(\cdot,\cdot)=\sum_{i=1}^N \E [\max \{\hat r_{ih}^{k-} ,\alpha_{ih}^k\}\ind (\hat r_{ih}^k\ge \max \{\hat r_{ih}^{k-} ,\alpha_{ih}^k\} )]$. In short, $R(\cdot,\cdot)$ is the expectation of revenue if we choose the optimal reserve price $\alpha_{ih}^{k*}$ for every bidder based on the knowledge of $\mu_{ih}^k(\cdot,\cdot)$ and everyone bids truthfully based on his valuation. Respectively, $\hat R(\cdot,\cdot)$ is the one we choose reserve price $\alpha_{ih}^k$ with the estimation of $\mu_{ih}^k(\cdot,\cdot)$, i.e., $\hat \mu_{ih}^k(\cdot,\cdot)$.
\begin{lemma}\label{lem:estimater}
When \Cref{lem:diffmu} holds, we have 
\[
|R_h^k(\cdot,\cdot)-\hat R_h^k(\cdot,\cdot)|\le C_6 H\log^2 K \|\phi(\cdot,\cdot)\|_{({\Lambda_h^{\mathtt{buffer.e} (\tilde{k})}})^{-1}},
\]where $C_6$ is a constant independent of $K$ and $H$.
\end{lemma}


Let's have an example when $N=1$, i.e., there is only one bidder.
\begin{example}
In this situation, $R(\cdot,\cdot)=\alpha^*(1-F(\alpha^*-1-\mu(\cdot,\cdot)))$ and $\hat R(\cdot,\cdot)=\alpha(1-F(\alpha-1-\hat\mu(\cdot,\cdot)))$. Therefore,
\[
|R(\cdot,\cdot)-\hat R(\cdot,\cdot)|\le (6C_1+1)C_5\sqrt{H}\log K \|\phi(\cdot,\cdot)\|_{{\Lambda}^{-1}},
\]
which is consistent with \Cref{lem:estimater}.
\end{example}

Now, we focus on the regret not in buffer caused by \Cref{algo:estimate}, denoted by $\Delta_1$. In order to facilitate the understanding, we rewrite the definition of $\Delta_1$ explicitly as follows.
\[
    \Delta_1=\sum_{\tau=1}^K[V_1^{\pi^*}(x_1^k)-\tilde V_1^{\pi_{\tilde k}}(x_1^k)]\ind(k\not \in \mathtt{buffer} ).
\]
Let's revisit our thought of bounding regret. We use empirical data to estimate unknown parameters, and then we assume that bidders will bid truthfully to construct the estimation of the R-function and Q-function. Then, we chase down the greedy policy. Therefore, when we take the expectation operator, we assume truthful bidding. Since $\Delta_5$ is easy to bound, we focus on how to bound $\Delta_1$. With a little abuse of notation, we will use $V(\cdot)$ to replace $\tilde V(\cdot)$ from now on.

Then, we have the following lemma.
\begin{lemma}\label{lem:lsvi}
Under \Cref{assumption:linearmdp}, \Cref{assumptionf} and \Cref{assumptionfdiff}, if we set $\poly (\log K)=(C_7+C_6 H \log^2K)\|\phi(\cdot,\cdot)\|_{({\Lambda_h^{\mathtt{buffer.e} (\tilde{k})}})^{-1}}$ in \Cref{algo:estimate}, where $C_7=B_8 H^{\frac{3}{2}} \log K$ and $B_8$ is determined in \Cref{lem:omegaandQ}, it holds that with probability at least $1-2\delta$,
\[
{\Delta}_1 \le C_8 \sqrt{H^5K\log^5 K} ,
\]where $C_8$ is a constant independent of $H$ and $K$.
\end{lemma}

\subsection{Proof of \texorpdfstring{\Cref{thm:knownf}}{thm:knownf}}
Let's make a decomposition of the regret at first. It holds that
\[
\textrm{Regret}\le 
\Delta_1+\Delta_2+\Delta_3+\Delta_4+\Delta_5.
\]

$\Delta_1$ is defined in \Cref{lem:lsvi} and with probability at least $1-2\delta$, $\Delta_1\le C_8 \sqrt{H^5K\log^5 K}$. 

$\Delta_2$ comes from the use of a buffer. With \Cref{lem:buffer}, it holds that $\Delta_2\le 3H\frac{3HC_2 \log^2K}{\log \frac{1}{\gamma}}$.

$\Delta_3$ comes from the use of policy $\pi_{\rm  rand}$. By applying \Cref{numofpi0}, it holds that $\Delta_3\le 3H \max\{4,1+\frac{4}{3}\log \frac{1}{\delta}\}$ with probability at least $1-\delta$.

$\Delta_4$ comes from the consequence from the existence of $\mathtt{L}$. Due to \Cref{lem:bound_lie}, we have $\Delta_4\le N H (4{C_1C_3 H}+8\log (\frac{2NH}{\delta})) 3H=3NH^2 (4{C_1C_3 H}+8\log (\frac{2NH}{\delta}))$, with probability at least $1-\delta$. As we have already considered the loss from buffer in $\Delta_2$, there is no need for us to consider it in $\Delta_4$.

$\Delta_5$ comes from the difference between the expectation of revenue when buyers bid truthfully and the actual expectation of revenue when buyers overbid or underbid, but it does not change the outcome. Since we already consider the loss from buffer, the size of overbid or underbid we should think about is less than $\frac{C_3H}{K}$ thanks to \Cref{overbid}. Therefore, the difference between the expectation of revenue when buyers bid truthfully and the actual expectation of revenue when buyers overbid or underbid without changing the outcome is less than $\frac{C_3H}{K}$ each step. So, it holds that $\Delta_5\le C_3 H^2$.

When estimating $\hat R(\cdot,\cdot)$, we have at most probability $\delta$ not satisfying the inequality in \Cref{lem:glm}.

Consequently, we set $\delta=\frac{p}{5}$, and it ends our proof.\qed 

\section{Omitted Proof in \texorpdfstring{\Cref{sec:unknownF}}{sec:unknownF}}\label{sec:proofunknownf}
Compared to \Cref{sec:proofknownf}, this section introduces a well-performed estimator to estimate the underlying distribution. With its help, we prove corresponding theorems when the market noise distribution is unknown.
\subsection{Useful Lemmas for Proving \texorpdfstring{\Cref{thm:unknownf}}{thm:unknownf}}
 In order to estimate noise distribution, we have the following lemma \citep{dvoretzky1956asymptotic} to bound the gap between the true distribution and the empirical distribution. We assume that $\hat F(\cdot)$ and $\hat f(\cdot)$ inherit all the properties of $F(\cdot)$ and $f(\cdot)$, because we can easily use some smooth kernels\footnote{It may introduce a constant 2 when describing the distance of two distributions. However, it doesn't matter as we consider order only.} to achieve this goal. However, in order to make the paper easy to understand, we do not explicitly write down the choice of a smooth kernel.

\begin{lemma}\label{lem:DKW}
    Given $t \in \mathbb N$, let $m_1, m_2, \dotsc, m_t$ be real-valued independent and identically distributed random variables with cumulative distribution function $F(\cdot)$. Let $\hat F_t(\cdot)$ denote the associated empirical distribution function defined by $\hat F_t(x)=\frac{1}{t}\sum_{i=1}^t \mathbf{1}_{\{m_i\le x\}}$ where $x\in \mathbb{R}$. Then with probability $1-\delta$, it holds
    \[ \sup_x |\hat F_t(x)-F(x)| \leq \sqrt{\frac{1}{2} \log \frac{2}{\delta}} t^{-\frac{1}{2}}. \]
\end{lemma}

Now, similar to the methodology in \Cref{sec:proofknownf}, we state the following lemmas in parallel.

\begin{lemma}\label{lem:unknownbuffer}
Under \Cref{assumption:linearmdp} about linear MDP, it holds that the number of buffer episodes is not larger than $C_{9} H\log^2 K$. Then, the number of corresponding steps is not larger than $C_{9} H^2\log^2 K$, where $C_{9}$ is a constant that only depends on $d$ and $\lambda$.
\end{lemma}

Recall that when market noise distribution is unknown, we implement \Cref{algo:simulation} to generate $\tilde q$ and we use $\tilde q$ to estimate $\theta$ instead of $q$. Therefore, $\mathtt{L}$ there considers simulation outcome $\tilde q$ rather than real outcome $q$. We formalize the definition of ${\mathtt{L}}$ there as follows, and we use $\tilde \rho$ to represent the reserve price in \Cref{algo:simulation}.
\[
\mathtt{L}=\{k: \ind(v_{ih}^k>\max\{b_{-ih}^{k+},\tilde \rho_{ih}^k\})\neq \ind(b_{ih}^k>\max\{b_{-ih}^{k+},\tilde \rho_{ih}^k\})\}.
\]
\begin{equation}
    \label{eqn:defn_lie_unknownF}
    \mathtt{L}=\{k: \ind(v_{ih}^k>\max\{b_{-ih}^{k+},\tilde \rho_{ih}^k\})\neq \ind(b_{ih}^k>\max\{b_{-ih}^{k+},\tilde \rho_{ih}^k\})\}.
\end{equation}
\begin{lemma}\label{lem:unknownbound_lie}
    With probability at least $1-\delta$, it holds that for any given $i$, $h$
    \[
    \mathtt{L}\le  C_{9}H\log^2 K+4{C_1C_3 H}+8\log (\frac{2NH}{\delta}) \le C_{10} H \log^2K,
    \]
    where $C_3$ is defined in \Cref{overbid} and $C_{10}$ is a constant independent of $K$ and $H$.
\end{lemma}

\begin{lemma}\label{lem:unknownglm}
We use $\theta_{ih}^*$ to denote the true parameter and $\hat \theta_{ih}$ to represent the outcome from \Cref{algo:Fandthetahat} in episode $\mathtt{buffer.e} (\tilde{k})$. Therefore, under \Cref{assumptionf} and \Cref{assumptionfdiff}, for any $i$ and $h$, it holds the following union bound that $C_{11}$ is a constant and
\[
\sqrt{(\hat \theta_{ih}-\theta_{ih}^*)^T\Lambda^{\mathtt{buffer.e} (\tilde{k})}(\hat \theta_{ih}-\theta_{ih}^*)} \le C_{11} \sqrt{H}\log K,
\]
with probability at least $1-\delta$, conditional on \textrm{Good Event} $\mathscr{E}$.
\end{lemma}

Same as \Cref{lem:diffmu}, we have the following lemma.
\begin{lemma}\label{lem:unknowndiffmu}
Conditional on \textrm{Good Event} $\mathscr{E}$, it holds that
\[
|\hat \mu_{ih}^k(\cdot,\cdot)-\mu_{ih}^k(\cdot,\cdot)|\le C_{11} \sqrt{H}\log K \|\phi(\cdot,\cdot)\|_{({\Lambda_h^{\mathtt{buffer.e} (\tilde{k})}})^{-1}}.
\]
\end{lemma}

Now, we introduce a lemma bounding the gap between the noise distribution $F(\cdot)$ and $\hat F(\cdot)$. 
\begin{lemma}\label{lem:boundF}
Conditional on \textrm{Good Event} $\mathscr{E}$, it holds with probability at least $1-\delta$ that for any $x$ in episode $\mathtt{buffer.e} (\tilde{k})$
\begin{align*}
|F(x)-\hat F(x)|\le& \sqrt{\frac{1}{2}\log \frac{2K}{\delta}} {(NH\mathtt{buffer.e} (\tilde{k}))}^{-\frac{1}{2}}+\frac{C_1C_3 H}{K}+\frac{C_{9} H \log^2 K}{\mathtt{buffer.e} (\tilde{k})}\\
&+C_1C_{11}\sqrt{H}\log K \overline{\|\phi(x_h^\tau,\upsilon_h^\tau)\|}_{({\Lambda_h^{\mathtt{buffer.e} (\tilde{k})}})^{-1}} \\
\le& C_{12} \frac{H\log^2 K}{\sqrt{\mathtt{buffer.e} (\tilde{k})}},    
\end{align*}
where $C_{12}$ is a constant.
\end{lemma}

Now, we begin to bound the wedge of $R(\cdot,\cdot)$ and $\hat R(\cdot,\cdot)$ corresponding to $\hat F(\cdot)$. It holds the following lemma.
\begin{lemma}\label{unknownboundR}
Conditional on \textrm{Good Event} $\mathscr{E}$, we have 
\[
|R_h^k(\cdot,\cdot)-\hat R_h^k(\cdot,\cdot)|\le C_{13}H \log^2 K \|\phi(\cdot,\cdot) \|_{(\Lambda_h^{\mathtt{buffer.e} (\tilde{k})})^{-1}}+C_{14} \frac{H^2\log^4K}{\sqrt{\mathtt{buffer.e} (\tilde{k})}},
\]where $C_{13}$ and $C_{14}$ are constants independent of $K$ and $H$.
\end{lemma}

We define $\Delta_1$ as the one in \Cref{lem:lsvi} of \Cref{sec:proofknownf}.
\begin{lemma}\label{lem:unknownlsvi}
Under \Cref{assumption:linearmdp}, \Cref{assumptionf} \Cref{assumptionfdiff} and \Cref{assumption:coveringnumber}, if we set $\poly _1(\log K)=C_{15}+C_{13}H \log^2 K $ and $\poly _2(\log K)=C_{14} H^2\log^4K$ in \Cref{algo:estimateunknownF}, where $C_{15}=D_7 H^\frac{3}{2}$ and $D_{7}$ is determined in \Cref{lem:unknownomegaandQ}, it holds that with probability at least $1-2\delta$,
\[
\Delta_1 \lesssim \tilde \cO (H^3\sqrt{K}) .
\]
\end{lemma}

\subsection{Proof of \texorpdfstring{\Cref{thm:unknownf}}{thm:unknownf}}
It is similar to the proof of \Cref{thm:knownf}. The only difference comes from \Cref{lem:boundF}. The probability of \textrm{Bad Event} $\mathscr{E}^c$ is now less than $6\delta$. Then, we set $\delta=\frac{p}{6}$ and it ends the proof. \qed

\section{Auxiliary Lemmas and Proofs in \texorpdfstring{\Cref{sec:proofknownf}}{sec:proofknownf}}
In this section, we prove the lemmas mentioned in \Cref{sec:proofknownf} in detail. It is organized by the order of lemmas.
\subsection{Proof of \texorpdfstring{\Cref{lem:buffer}}{lem:buffer}}
First of all, we have the following lemmas.
\begin{lemma}[Lemma 2, \citep{gao2021provably}]\label{doubledet}
Assume $m\le n$, $A=\sum_{\tau=1}^m \phi_\tau \phi_\tau^T+\lambda I$. $B=\sum_{\tau=1}^n \phi_\tau \phi_\tau^T+\lambda I$, where $\phi_\tau$ is abridge for $\phi(x_\tau, \upsilon_\tau)$, similarly hereinafter. Then if $A^{-1} \not \prec 2B^{-1}$, we have 

\[
\log \det B \ge \log \det A +\log 2.
\]
\end{lemma}
\begin{lemma}[Lemma 1, \citep{gao2021provably}]\label{bounddet}
 Since $\|\phi_\tau\|\le 1$. Let $A=\sum_{\tau=1}^K \phi_\tau \phi_\tau^T+\lambda I$, then we have 
\[
\log \det A \le d\log d +d \log(K+\lambda)\le K_1 \log K.
\]
\end{lemma}
Therefore, for $2{(\Lambda^{\mathtt{buffer.s} (\tilde{k}+1)})}^{-1}\not \succeq (\Lambda^{\mathtt{buffer.e} (\tilde{k})})^{-1}$ and $\mathtt{buffer.e} (\tilde{k}+1)\ge \mathtt{buffer.s} (\tilde{k}+1)$, it holds that $2(\Lambda^{\mathtt{buffer.e} (\tilde{k}+1)})^{-1}\not \succeq (\Lambda^{\mathtt{buffer.e} (\tilde{k})})^{-1}$. Therefore, $\det \Lambda^{\mathtt{buffer.e} (\tilde{k}+1)} \ge 2 \det \Lambda^{\mathtt{buffer.e} (\tilde{k})}$. Then, using \Cref{doubledet}, we know that for any $h$ and $k$, it holds $\log \det \Lambda_h^k\le K_1 \log K$. We have $\log \det \Lambda_h^0=d \log \lambda$. Combining \Cref{doubledet}, we have that the number of buffer episodes for any $h$ is not larger than $\frac{3\log K}{\log \frac{1}{\gamma}}\frac{K_1 \log K-d \log \lambda}{\log 2}$. Then, there is a constant $C_2$ satisfying $K_1 \log K-d \log \lambda\le C_2 \log 2 \log K$. Therefore, the total number of episodes in buffer is not larger than $\frac{3HC_2\log^2 K}{\log \frac{1}{\gamma}}$. For the number of total steps, it is obvious that it is smaller than $H$ times the number of episodes. Then, it ends the proof. \qed 

\subsection{Proof of \texorpdfstring{\Cref{overbid}}{overbid}}
\citet{myerson1981optimal} shows that the optimal strategy for a one-round second-price auction is to bid truthfully. Therefore, if a bidder overbids or underbids for more than $\frac{3H\sqrt{2N}}{K\sqrt{1-\gamma}}$, his loss holds that
\[
{\mathtt{L}oss} \ge \frac{1}{NHK}\frac{\beta}{2K}\frac{1}{3}\frac{\beta}{K}=\frac{3H}{K^3(1-\gamma)},
\]where $\beta=\frac{3H\sqrt{2N}}{\sqrt{1-\gamma}}$.

The inequality holds since with probability $\frac{1}{KHN}$, the policy will be $\pi_{\rm rand}$ and the bidder is selected, and the total loss is higher than the loss with policy $\pi_{\rm rand}$. With a uniform reserve price, the probability that a loss happens is $\frac{\beta}{3K}$. Then, average loss is $\frac{\beta}{2K}$. Due to the existence of a buffer, the overbid or underbid can only make an influence on policy $t=\frac{3\log K}{\log \frac{1}{\gamma}}$ episodes later. Because of the existence of a discount rate, an upper bound of revenue for each buyer after $t$ episodes is $\frac{\gamma^t}{1-\gamma}3H=\frac{3H}{K^3(1-\gamma)}$. 

Therefore, with the assumption that buyers are all rational, it finishes the proof.\qed

The proof also illuminates why we need to keep $\pi_{\rm rand}$ in every episode instead of only utilizing buffer periods to explore. We run $\pi_{\rm rand}$ so that each untruthful bid will suffer an immediate loss. Note that there are buffer periods before the policies are updated. Therefore, any possible reward from manipulating the policy estimate is delayed by the length of the buffer period. Consequently, rational bidders would reduce their extent of untruthful bidding, as they need to offset the immediate losses incurred by $\pi_{\rm rand}$. If we only run $\pi_{\rm rand}$ in the buffer period as we hope to explore, we cannot bound the extent of untruthful bidding in other horizons. Then, there is no guarantee
for regret bounds.

\subsection{Proof of \texorpdfstring{\Cref{lem:bound_lie}}{bound lie}}
For convenience, similar to \citet{golrezaei2019dynamic}, we define  
\[
L_i=\{t:t\in[0,K]\ {\rm and} \ \ind(v_{i}^{t}\ge m_{i}^{t}) \neq \ind(b_{i}^{t}\ge m_{i}^{t})\},
\]for each buyer $i$.

We define $o_{i}^{t}=(b_{i}^{t}-v_{i}^{t})_+$ and $s_{i}^{t}=(v_{i}^{t}-b_{i}^{t})_+$, where $t=1,\ldots,K$ given $h$. When we can determine the subscript through the context, we omit the subscript $h$ for convenience.

Then we define $q_{i}^{t}$, which is a binary variable. It equals one if buyer $i$ wins and zero if loses. Therefore, we have $S_i=\{t:t\in[1,K], q_{i}^{t}=0\  and\  s_{i}^{t}\ge \alpha\}$ and $O_i=\{t:q_{i}^{t}=1\ and \ o_{i}^{t}\ge \alpha\}$.
As a result, $L_i=L_i^s\bigcup L_i^o$, where $L_i^s=\{t:\ind(v_{i}^{t}\ge r_{i}^{t})=1, \ind(b_{i}^{t}\ge r_{i}^{t})=0\}$ and $L_i^o=\{t:\ind(v_{i}^{t}\ge r_{i}^{t})=0, \ind(b_{i}^{t}\ge r_{i}^{t})=1\}$. Finally, we have $S_i^c=\{t:q_{i}^{t}=1\ or\ s_{i}^t\le \alpha\}$. So, $|L_i^s|=|S_i\bigcap L_i^s|+|S_i^c\bigcap L_i^s|$.

To bound $|(S_i\bigcap L_i^s)\bigcup(O_i\bigcap L_i^o)|$: using \Cref{lem:buffer} and \Cref{overbid}, we have that if we set $\alpha=C_3\frac{H}{K}$, it is bounded by $\frac{3HC_2\log^2 K}{\log \frac{1}{\gamma}}$.

To bound $|S_i^c\bigcap L_i^s|$: it means that underbid changes the outcome, and the level of underbid is smaller than $\alpha$. Since $|f|\le C_1$, it holds for origin $x$:
\[
\Pr(t\in S_i^c\bigcap L_i^s |\mathcal{F}_t)\le \int_x^{x+\alpha}f(z)dz\le C_1\alpha.
\]
Let's define $\xi_t=\ind(t\in S_i^c\bigcap L_i^s)$ while $\omega_t=\Pr(t\in S_i^c \bigcap L_i^s \given \mathcal{F}_t)$. Then $|S_i^c\bigcap L_i^s|=\sum_{t=1}^{K}\xi_t$ and $\E (\xi_t-\omega_t \given \mathcal{F}_t)=0$.

Using Azuma-Hoeffding inequality \citep{hoeffding1994probability}, it holds that
\[
\Pr(|S_i^c\bigcap L_i^s|\ge \frac{1+\iota}{1-\epsilon}\sum_1^{K}\omega_t)\le \exp(-\epsilon\iota\sum_1^{K}\omega_t).
\]

Let $A=\sum_1^{K}\omega_t\le {K}C_1\alpha$, $\epsilon=\frac{1}{2}$ and $\iota=\frac{2}{A}\log(\frac{2NH}{\delta})$, we have 
\[
|S_i^c\bigcap L_i^s|\le 2(1+\iota)A\le 2KC_1\alpha+4\log (\frac{2NH}{\delta}),
\] with probability at least $1-\frac{\delta}{2NH}$.

Similarly, we bound $|O_i^c\bigcap L_i^o|$ with the same bound that
\[
|O_i^c\bigcap L_i^o|\le 2KC_1\alpha+4\log (\frac{2NH}{\delta}),
\]
with probability at least $1-\frac{\delta}{2NH}$.

Then, we set $\alpha=\frac{C_3H}{K}$ and combine the items all to obtain 
\[
|L_i|\le  \frac{3HC_2\log^2 K}{\log \frac{1}{\gamma}}+4{C_1C_3 H}+8\log (\frac{2NH}{\delta}),
\]
with probability at least $1-\frac{\delta}{NH}$.

With the same methodology, we obtain the union bound for any given $i$ and $h$ with probability at least $1-\delta$ that
\[
 \mathtt{L}\le 
 \frac{3HC_2\log^2 K}{\log \frac{1}{\gamma}}+4{C_1C_3 H}+8\log (\frac{2NH}{\delta}) ,
\]
and it finishes the proof.\qed

\subsection{Proof of \texorpdfstring{\Cref{numofpi0}}{numofpi0}}
We use random variables $X_1,\ldots,X_{KH}$ to represent whether $\pi_{\rm rand}$ is used. If we choose policy $\pi_{\rm rand}$, then $X=1$, or $X=0$ otherwise.

Using Bernstein inequalities \citep{bernstein1924modification}, it holds that 
\[
\Pr(\sum_{i=1}^{KH} X_i-KH\frac{1}{KH}\ge t)\le \exp\{\frac{-t^2/2}{(1-1/KH)+t/3}\},
\]
since $X-\frac{1}{KH}$ has mean zero and $\var(X)=\frac{1}{KH}(1-\frac{1}{KH})$.

Therefore, set $t=\max\{3,\frac{4}{3}\log \frac{1}{\delta}\}$, the right side is smaller than $\delta$ and it finishes the proof. \qed 

\subsection{Proof of \texorpdfstring{\Cref{lem:glm}}{lem:glm}}

First of all, we omit subscripts $i$ and $h$ for convenience, and we will get the union bound in the end. 

Then, we introduce some notations. We use $\tilde q_\tau$ to represent the outcome that every bidder bids truthfully and $\hat q_\tau$ to represent the outcome with real bidding. Then $\hat \theta$ and $\tilde \theta$ correspond to $\{\hat q_\tau\}$ and $\{\tilde q_\tau\}$.

Now, we focus on buyer $i$ and step $h$, so we omit subscripts $i$ and $h$ from now on. We have the following lemma at first:
\begin{lemma}\label{6lie}
Under \Cref{algo:thetahat}, it holds that
\[
\sum_{\tau=1}^{\mathtt{buffer.e} (\tilde{k})} (\tilde q_\tau -1+ F(m_\tau-1-\langle \phi_\tau,\hat{\theta} \rangle))^2 \le \sum_{\tau=1}^{\mathtt{buffer.e} (\tilde{k})} (\tilde q_\tau -1+ F(m_\tau-1-\langle \phi_\tau,\theta^* \rangle))^2+6 \mathtt{L},
\]
where $\mathtt{L}\le C_4 H \log^2 K $ due to \Cref{lem:bound_lie}.
\end{lemma}

\subsubsection{Proof of \texorpdfstring{\Cref{6lie}}{6lie}}
Since there are at most $\mathtt{L}$ steps that overbid or underbid changes the outcome, $\hat q_\tau$ and $\tilde q_\tau$ differ in at most $\mathtt{L}$ different points. Since $\tilde q_\tau$ and $\hat q_\tau$ belong to $\{0,1\}$, we have 
\[
\sum_{\tau=1}^{\mathtt{buffer.e} (\tilde{k})} (\hat q_\tau-1)^2\le \sum_{\tau=1}^{\mathtt{buffer.e} (\tilde{k})} (\tilde q_\tau-1)^2+\mathtt{L}.
\]

Then, since $F(\cdot) \in [0,1]$, it holds that 
\[
-2\sum_\tau (1-\hat  q_\tau) F(m_\tau-1-\langle \phi_\tau,\theta \rangle)\le -2 \sum_\tau  (1-\tilde q_\tau) F(m_\tau-1-\langle \phi_\tau,\theta \rangle)+2\mathtt{L}.
\]
for any $\theta$.

Therefore, it holds that 
\begin{equation}\label[ineq]{3lie}
\sum_\tau (\tilde q_\tau -1+F(m_\tau-1-\langle \phi_\tau,\theta \rangle))^2 \le \sum_\tau (\hat q_\tau -1+F(m_\tau-1-\langle \phi_\tau,\theta \rangle))^2 +3 \mathtt{L},    
\end{equation}

for any $\theta$.

Finally, with the optimality of $\hat \theta$ and $\tilde \theta$, it holds that 
\begin{align*}
    &\sum_\tau (\tilde q_\tau -1+F(m_\tau-1-\langle \phi_\tau,\hat \theta \rangle))^2\\
    \le& \sum_\tau (\hat q_\tau -1+F(m_\tau-1-\langle \phi_\tau,\hat\theta \rangle))^2+3\mathtt{L}\\
    \le& \sum_\tau (\hat q_\tau -1+F(m_\tau-1-\langle \phi_\tau,\tilde\theta \rangle))^2+3\mathtt{L}\\
    \le& \sum_\tau (\tilde q_\tau -1+F(m_\tau-1-\langle \phi_\tau,\tilde\theta \rangle))^2+6\mathtt{L}\\
    \le& \sum_\tau (\tilde q_\tau -1+F(m_\tau-1-\langle \phi_\tau,\theta^* \rangle))^2+6\mathtt{L}.
\end{align*}
The first and third inequalities hold due to \Cref{3lie}. The second and last inequalities hold because of the optimality of $\hat \theta$ and $\tilde \theta$. Then, it finishes the proof. \qed 

Then we use $f_{m_\tau}(\langle \phi_\tau, \theta\rangle)$ to represent $F(m_\tau-1-\langle \phi_\tau, \theta\rangle)$ in shorthand.

Therefore, with \Cref{6lie}, we have 
\[
\sum_\tau [f_{m_\tau}(\langle \phi_\tau, \hat \theta\rangle)-f_{m_\tau}(\langle \phi_\tau, \theta^*\rangle)]\le 2 |\sum_\tau \xi_\tau (f_{m_\tau}(\langle \phi_\tau, \hat\theta\rangle)-f_{m_\tau}(\langle \phi_\tau, \theta^*\rangle))|+6\mathtt{L},
\]where $\xi_\tau=(1-\tilde q_\tau) -f_{m_\tau}(\langle \phi_\tau, \theta^*\rangle)$. The inequality holds because of simple rearrangement.

Then, we have 
\begin{align*}
f_{m_\tau}(\langle \phi_\tau, \hat\theta\rangle)-f_{m_\tau}(\langle \phi_\tau, \theta^*\rangle)&=\int_{\langle\phi_\tau,\theta^*\rangle}^{\langle\phi_\tau,\hat\theta\rangle}  f'_{m_\tau}(s)ds\\
&=\langle\phi_\tau,\hat \theta-\theta^*\rangle \int_0^1 f'_{m_\tau}(\langle \phi_\tau, s\hat \theta+(1-s) \theta^*\rangle)ds \\
&=\langle\phi_\tau,\hat \theta-\theta^*\rangle D_\tau,
\end{align*}where $D_\tau=\int_0^1 f'_{m_\tau}(\langle \phi_\tau, s\hat \theta+(1-s) \theta^*\rangle)ds.$

So, it holds that 
\[
\sum_\tau D_\tau^2(\langle \phi_\tau ,\hat \theta-\theta^*\rangle)^2 \le 2|\sum_\tau \xi_\tau D_\tau \langle \phi_\tau,\hat \theta-\theta^*\rangle|+6\mathtt{L}.
\]

Since $\|\theta\|\le \sqrt{d}$, we use $V_\epsilon$ which is a set of ball with radius $\epsilon$ to cover $\mathcal{B}(0,\sqrt{d})\times\mathcal{B}(0,\sqrt{d})$. Then, the cardinality of $V_\epsilon$ is smaller than $B_1 (\frac{\sqrt{d}}{\epsilon})^{2d}=\frac{B_2}{\epsilon^{2d}}$, where $B_1$ and $B_2$ are constants only depending on dimension d. Thanks to \Cref{assumptionf} and \Cref{assumptionfdiff}, we have $|f''|\le L$ and $|D_\tau|\le C_1$. 

Therefore, for any $(\hat \theta,\theta^*)$, there exists $(\theta,\theta')$, which is the center of a ball in $V_\epsilon$, so that $\|(\hat \theta,\theta^*)-(\theta,\theta')\|\le \epsilon$. In this way, it holds that
\begin{align*}
&|\langle \phi_\tau, D_\tau(\theta,\theta') (\theta-\theta')-D_\tau (\hat \theta,\theta^*) (\hat \theta-\theta^*)\rangle|\\
\le & 2\sqrt{d} |D_\tau(\theta,\theta')-D_\tau (\hat{\theta},\theta^*)|+|D_\tau|(\|\theta-\hat \theta\|+\|\theta'-\theta^*\|)\\
 \le & 2L\sqrt{d}\epsilon+C_1\epsilon \\
 \le & (2L\sqrt{d}+C_1)\epsilon.
\end{align*}
The first inequality holds since $\|\theta\|\le \sqrt{d}$. The second inequality holds since $|f''|\le L$ and $|D_\tau|\le C_1$.

Therefore, it holds that 
\[
\|\sum_\tau \xi_\tau \langle \phi_\tau, D_\tau(\hat \theta ,\theta^*)(\hat \theta -\theta^*)\|\le \|\sum_\tau \xi_\tau \langle \phi_\tau, D_\tau( \theta ,\theta')( \theta -\theta')\|+ (2L\sqrt{d}+C_1)\mathtt{buffer.e} (\tilde{k})\epsilon,
\] since $|\xi_\tau|\le 1$.

Let's define the following shorthands
\[
V(\phi)=\sum_\tau \langle\phi_t,D_\tau (\theta-\theta')\rangle^2,
\]
\[
V(\hat \phi)=\sum_\tau \langle\phi_t,D_\tau (\hat\theta-\theta^*)\rangle^2.
\]

Therefore, by applying the inequality above, we have
\begin{equation}\label[ineq]{V}
V(\phi)\le V(\hat \phi)+4C_1\sqrt{d} (2L\sqrt{d}+C_1)\mathtt{buffer.e} (\tilde{k})\epsilon.    
\end{equation}

The inequality holds because of the square difference formula.

Since for positive number $a$ $b$ and $c$, if $a\le b+c$, than $\sqrt{a}\le \sqrt{b}+\sqrt{c}$. So, it holds that
\begin{equation}\label[ineq]{sqrtv}
\sqrt{V(\phi)}\le \sqrt{V(\hat \phi)}+\sqrt{4C_1\sqrt{d} (2L\sqrt{d}+C_1)\mathtt{buffer.e} (\tilde{k})\epsilon}.    
\end{equation}

Since $\theta^*$ is the true parameter and $\xi_\tau=(1-\tilde q_\tau) -f_{m_\tau}(\langle \phi_\tau, \theta^*\rangle)$ which is determined by truthful bid, it holds $\E (\xi_\tau\given \phi_{1:\tau},\xi_{1:\tau-1})=0$ whose value is determined by $z_\tau$ only. Due to Azuma-Hoeffding inequality \citep{hoeffding1994probability}, it holds that
\begin{equation}\label[ineq]{glmhoeff}
\Pr[|\sum_\tau \xi_\tau D_\tau \langle\phi_\tau ,\theta-\theta'\rangle|\ge \sqrt{\log \frac{2B_2HN}{\delta\epsilon^{2d}}V(\phi)}]\le \frac{\delta}{HN},    
\end{equation}
for any $(\theta,\theta')$ with probability at least $1-\frac{\delta}{HN}$.

Therefore, it holds that 
\begin{align*}
V(\hat \phi)\le& 4C_1\sqrt{d} (2L\sqrt{d}+C_1)\mathtt{buffer.e} (\tilde{k})\epsilon+   V(\phi)\\
\le& 4C_1\sqrt{d} (2L\sqrt{d}+C_1)\mathtt{buffer.e} (\tilde{k})\epsilon+ 2 \sqrt{\log \frac{2B_2HN}{\delta\epsilon^{2d}}V(\phi)} +6\mathtt{L}\\
\le &4C_1\sqrt{d} (2L\sqrt{d}+C_1)\mathtt{buffer.e} (\tilde{k})\epsilon+ 2 \sqrt{\log \frac{2B_2HN}{\delta\epsilon^{2d}}}\big[\sqrt{V(\hat \phi)}\\
&+\sqrt{4C_1\sqrt{d} (2L\sqrt{d}+C_1)\mathtt{buffer.e} (\tilde{k})\epsilon}\big]+6\mathtt{L}\\
=&   4C_1\sqrt{d} (2L\sqrt{d}+C_1)+ 2 \sqrt{\log \frac{2B_2HN \mathtt{buffer.e} (\tilde{k})^{2d}}{\delta}}\big[\sqrt{V(\hat \phi)}\\
&+\sqrt{4C_1\sqrt{d} (2L\sqrt{d}+C_1)}\big]+6\mathtt{L}\\
\le& 4C_1\sqrt{d} (2L\sqrt{d}+C_1)+ 2 \sqrt{\log \frac{2B_2HN \mathtt{buffer.e} (\tilde{k})^{2d}}{\delta}}\big[\sqrt{V(\hat \phi)}\\
&+\sqrt{4C_1\sqrt{d} (2L\sqrt{d}+C_1)}\big]+6C_4H \log^2 K.
\end{align*}
The first inequality holds due to \Cref{V} while the second one holds due to \Cref{glmhoeff} and \Cref{6lie}. The third inequality holds because of \Cref{sqrtv}. The equality holds since we set $\epsilon=\frac{1}{\mathtt{buffer.e} (\tilde{k})}$. The final inequality holds because of \Cref{lem:bound_lie}.

Finally, applying the root formula of the quadratic equation, it is obvious that there exists a constant $B_3>0$ such that $V(\hat \phi)\le B_3 H\log^2 K$.

Similar to \citet{wang2020optimism}, we have 
\[
\sqrt{(\hat \theta_{ih}-\theta_{ih}^*)^T\Lambda^{\mathtt{buffer.e} (\tilde{k})}(\hat \theta_{ih}-\theta_{ih}^*)} \le c_1^{-1}\sqrt{V(\hat \phi)}+2\sqrt{d\lambda},
\]for any $i$ and $h$ with probability at least $1-\delta$.

It holds since 
\[
\sqrt{(\hat \theta-\theta^*)^T\Lambda^{\mathtt{buffer.e} (\tilde{k})}(\hat \theta-\theta^*)}\le \sqrt{(\hat \theta-\theta^*)^T(\sum_\tau \phi_\tau \phi_\tau^T) (\hat \theta-\theta^*)}+\sqrt{(\hat \theta-\theta^*)^T(\lambda I)(\hat \theta-\theta^*)}.
\]
Then, we have $D_\tau^2\ge c_1^2$ and $\|(\hat \theta_{ih}-\theta_{ih}^*)\|_{\lambda I}\le 2\sqrt{d\lambda}$.

In the end, we find that there exists a constant $C_5$ that satisfies 
\[
\sqrt{(\hat \theta_{ih}-\theta_{ih}^*)^T\Lambda^{\mathtt{buffer.e} (\tilde{k})}(\hat \theta_{ih}-\theta_{ih}^*)} \le C_5 \sqrt{H}\log K,
\]
which ends the proof.\qed 

\subsection{Proof of \texorpdfstring{\Cref{lem:diffmu}}{lem:diffmu}}
Using the Cauchy inequality, we have the following statement:
\begin{lemma}\label{cauchy}
It holds that
\[
|\langle\phi(x,v),\hat \theta-\theta\rangle|\le \sqrt{(\hat \theta-\theta)^T\Lambda (\hat \theta-\theta)}  \|\phi(x,v)\|_{\Lambda^{-1}}.  
\]  
Specially, taking $\Lambda=\Lambda_h^{\mathtt{buffer.e} (\tilde{k})}=\sum_{\tau=1}^{\mathtt{buffer.e} (\tilde{k})} \phi(x_h^\tau,\upsilon_h^\tau) \phi(x_h^\tau,\upsilon_h^\tau)^T +\lambda I$, the inequality holds.  
\end{lemma}

Then \Cref{cauchy} and \Cref{lem:glm} lead to \Cref{lem:diffmu}.\qed 

\subsection{Proof of \texorpdfstring{\Cref{lem:estimater}}{lem:estimator}}
Firstly, we define $\tilde  R_h^k(\cdot,\cdot)=\sum_{i=1}^N \E [\max \{ r_{ih}^{k-} ,\alpha_{ih}^k\}\ind (r_{ih}^k\ge \max \{r_{ih}^{k-} ,\alpha_{ih}^k )]$. Then, $|R_h^k(\cdot,\cdot)-\hat R_h^k(\cdot,\cdot)|\le |R_h^k(\cdot,\cdot)-\tilde R_h^k(\cdot,\cdot)|+|\tilde R_h^k(\cdot,\cdot)-\hat R_h^k(\cdot,\cdot)|$.

To bound $|\tilde R_h^k(\cdot,\cdot)-\hat R_h^k(\cdot,\cdot)|$, we have 
\begin{align*}
|\tilde R_h^k(\cdot,\cdot)-\hat R_h^k(\cdot,\cdot)|\le&
    \sum_{i=1}^N \E |[\max \{ r_{ih}^{k-} ,\alpha_{ih}^k\}\ind (r_{ih}^k\ge \max \{r_{ih}^{k-} ,\alpha_{ih}^k )]\\
    &-[\max \{\hat r_{ih}^{k-} ,\alpha_{ih}^k\}\ind (\hat r_{ih}^k\ge \max \{\hat r_{ih}^{k-} ,\alpha_{ih}^k )]|\\
    &\le \sum_{i=1}^N\Delta_1+\Delta_2+\Delta_3\\
    &\le (1+6C_1)N C_5 \sqrt{H}\log K \|\phi(\cdot,\cdot)\|_{({\Lambda_h^{\mathtt{buffer.e} (\tilde{k})}})^{-1}},
\end{align*}
where 
\begin{align*}
\Delta_1=&|[\max \{ r_{ih}^{k-} ,\alpha_{ih}^k\}\ind (r_{ih}^k\ge \max \{r_{ih}^{k-} ,\alpha_{ih}^k \})]\\
&-[\max \{\hat r_{ih}^{k-} ,\alpha_{ih}^k\}\ind (r_{ih}^k\ge \max \{r_{ih}^{k-} ,\alpha_{ih}^k\} )]|, 
\end{align*}
\begin{align*}
 \Delta_2=&|[\max \{\hat r_{ih}^{k-} ,\alpha_{ih}^k\}\ind (r_{ih}^k\ge \max \{r_{ih}^{k-} ,\alpha_{ih}^k \})]\\
 &-[\max \{\hat r_{ih}^{k-} ,\alpha_{ih}^k\}\ind (\hat r_{ih}^k\ge \max \{r_{ih}^{k-} ,\alpha_{ih}^k\} )]|     
\end{align*}
and 
\begin{align*}
\Delta_3=&|[\max \{\hat r_{ih}^{k-} ,\alpha_{ih}^k\}\ind (\hat r_{ih}^k\ge \max \{r_{ih}^{k-} ,\alpha_{ih}^k\} )]\\
&-[\max \{\hat r_{ih}^{k-} ,\alpha_{ih}^k\}\ind (\hat r_{ih}^k\ge \max \{\hat r_{ih}^{k-} ,\alpha_{ih}^k\})]|.    
\end{align*}

The first inequality holds due to the properties of convex functions. The second inequality holds due to the triangle inequality. The third inequality holds since $\Delta_1 \le |\max \{ r_{ih}^{k-} ,\alpha_{ih}^k\}-\max \{ \hat r_{ih}^{k-} ,\alpha_{ih}^k\}|\le |r-\hat r|$, $\Delta_2\le 3C_1 |r-\hat r|$ and $\Delta_3\le 3C_1 |r-\hat r|$. The reason why $\Delta_2\le 3C_1 |r-\hat r| $ is $\max \{\hat r,\alpha\}\le 3$ and $\E| \ind(r_{ih}^k\ge \max \{r_{ih}^{k-} ,\alpha_{ih}^k\} )- \ind (\hat r_{ih}^k\ge \max \{r_{ih}^{k-} ,\alpha_{ih}^k \})|\le C_1 |r-\hat r|$.

To bound $|R_h^k(\cdot,\cdot)-\tilde R_h^k(\cdot,\cdot)|$, we have the following lemmas. We define $W_{ih}^k(\alpha)=\E [\max \{v_{ih}^{k-},\alpha\}\ind (v_{ih}^k\ge \max\{v_{ih}^{k-},\alpha\})\given \phi_h^k]$ at first.
\begin{lemma}[Lemma C.3. \citep{golrezaei2019dynamic}] \label{lem:differ0}
Since $\alpha_{ih}^{k*}$ is determined by Myerson Lemma \citep{myerson1981optimal}, we have $W_{ih}^{'k}(\alpha_{ih}^{k*})=0$. Furthermore, there exists a constant $B_4$ that for any $\alpha$ between $\alpha_{ih}^k$ and $\alpha_{ih}^{k*}$, we have $|W_{ih}^{''k} (\alpha)|\le B_4$ for any $i$ and $h$, under assumption \Cref{assumptionf}, \Cref{assumptionfdiff} and \Cref{logconcave}.
\end{lemma}
\begin{lemma}[Lemma C.4. \citep{golrezaei2019dynamic}] \label{lem:reservegap}
Under \Cref{logconcave}, it holds that
\[
|\alpha_{ih}^{k*}-\alpha_{ih}^k|\le |\langle \phi_h^k, \theta_{ih}-\hat \theta_{ih}\rangle|.
\]
\end{lemma}

By applying \Cref{lem:reservegap}, we have 
\begin{align*}
    |R_h^k(\cdot,\cdot)-\tilde R_h^k(\cdot,\cdot)|&\le \sum_{i=1}^N \frac{B_4}{2}(\alpha_{ih}^{k*}-\alpha_{ih}^k)^2\\
    &\le N\frac{B_4}{2}(\langle \phi_h^k, \theta_{ih}-\hat \theta_{ih}\rangle)^2\\
    &\le N\frac{B_4}{2} C_5^2 H \log^2K \|\phi(\cdot,\cdot)\|_{(\Lambda_h^{\mathtt{buffer.e} (\tilde{k})})^{-1}}^2\\
    &\le N\frac{B_4}{2} C_5^2 H \log^2K \|\phi(\cdot,\cdot)\|_{(\Lambda_h^{\mathtt{buffer.e} (\tilde{k})})^{-1}}\frac{1}{\sqrt{\lambda}}.
\end{align*}
The first inequality holds due to the Taylor expansion. The second inequality holds due to \Cref{lem:reservegap}, while the third one holds due to \Cref{lem:diffmu}. The last inequality holds since $\|\phi\|_{\Lambda^{-1}}\le \frac{1}{\lambda}$.
\re{\begin{remark}\label{redundant}
    Without \Cref{logconcave}, we can get the last inequality from the integral form of $R(\cdot,\cdot)$. For example, when $N=1$, it holds that $R(\cdot,\cdot)=\alpha(1-F(\alpha-1-\langle \phi,\theta\rangle))$. Then, $|R-\tilde R|\le 3C_1|\langle\theta-\hat\theta,\phi\rangle|$ due to \Cref{assumptionf}. It shows that \Cref{logconcave} is actually redundant as \Cref{assumptionf} exists.
\end{remark}}

Combining the differences $|\tilde R_h^k(\cdot,\cdot)-\hat R_h^k(\cdot,\cdot)|$ and $|R_h^k(\cdot,\cdot)-\tilde R_h^k(\cdot,\cdot)|$, it holds that 
\[
|R_h^k(\cdot,\cdot)-\hat R_h^k(\cdot,\cdot)|\le [(1+6C_1)C_5\sqrt{H}\log K+\frac{B_4}{2\sqrt{\lambda}}C_5^2H\log^2K]N \|\phi(\cdot,\cdot)\|_{(\Lambda_h^{\mathtt{buffer.e} (\tilde{k})})^{-1}}.
\]
Therefore, there exists a constant $C_6$ which is independent of $H$ and $K$, satisfying 
\[
|R_h^k(\cdot,\cdot)-\hat R_h^k(\cdot,\cdot)|\le C_6H\log^2K \|\phi(\cdot,\cdot)\|_{(\Lambda_h^{\mathtt{buffer.e} (\tilde{k})})^{-1}},
\]
and it ends the proof.\qed 

\subsection{Proof of \texorpdfstring{\Cref{lem:lsvi}}{lem:lsvi}}
In order to prove \Cref{lem:lsvi}, we have the following lemmas for help.
\begin{lemma}\label{lem:boundomega1}
For any fixed policy $\pi$, let $\{\omega_h^\pi\}_{h\in[H]}$ be the corresponding vectors such that $Q_h^\pi(\cdot,\cdot)=R(\cdot,\cdot)+\langle \phi(\cdot,\cdot), \omega_h^\pi\rangle$ for any $h$. Then, it holds that 
\[
\|\omega_h^\pi\|\le 3H\sqrt{d},
\]for any $h$.
\end{lemma}
\begin{proof}
Since it holds
\[
Q_h^\pi(\cdot,\cdot)=(R+\mathbb{P}_h V_{h+1}^\pi)(\cdot.\cdot),
\]
and the linearity of MDP, we have
\[
\omega_h^\pi=\int V_{h+1}^\pi(\cdot)d\mathcal{M}_h(\cdot).
\]
Therefore, considering $|V|\le 3H$ and $\|\mathcal{M}_h(\mathcal{S})\|\le \sqrt{d}$, \Cref{lem:boundomega1} holds.
\end{proof}
\begin{lemma}\label{lem:boundomege2}
For any $(k,h)\in [K]\times [H]$, the vector $\omega_h^{\mathtt{buffer.e} (\tilde{k})}$ in \Cref{algo:estimate} satisfies:
\[
\|\omega_h^{\mathtt{buffer.e} (\tilde{k})}\|=\|\omega_h^{k}\|\le 3H \sqrt{\frac{d\mathtt{buffer.e} (\tilde{k})}{\lambda}}\le 3H \sqrt{\frac{dk}{\lambda}}.
\]
\end{lemma}

\begin{proof}
Since we only update at episode $\mathtt{buffer.e} (\tilde{k})$, $\omega_h^{k}$ is the same as $\omega_h^{\mathtt{buffer.e} (\tilde{k})}$.

For any vector $\nu\in \mathbb{R}^d$, we have
\begin{align*}
|\nu^T \omega_h^{\mathtt{buffer.e} (\tilde{k})}|&=|\nu^T (\Lambda_k^{\mathtt{buffer.e} (\tilde{k})})^{-1}\sum_{\tau=1}^{\mathtt{buffer.e} (\tilde{k})} \phi_h^\tau \max_a Q_{h+1}(\cdot,\cdot)|\\
&\le \sum_\tau 3H |\nu^T (\Lambda_k^{\mathtt{buffer.e} (\tilde{k})})^{-1} \phi_h^\tau|\\
&\le 3H\sqrt{[\sum_\tau \nu^T (\Lambda_k^{\mathtt{buffer.e} (\tilde{k})})^{-1} \nu][\sum_\tau (\phi_h^\tau)(\Lambda_k^{\mathtt{buffer.e} (\tilde{k})})^{-1}\phi_h^\tau]}\\
&\le 3H\|\nu\|\sqrt{\frac{d\mathtt{buffer.e} (\tilde{k})}{\lambda}}.
\end{align*}
The first inequality holds since $Q\le 3H$, while the second inequality holds due to the Cauchy inequality. The third inequality holds since $(\Lambda_k^{\mathtt{buffer.e} (\tilde{k})})^{-1}\preceq \frac{1}{\lambda}I$ and the following lemma.
\end{proof}

\begin{lemma}[Lemma D.1. \citep{jin2020provably}]\label{lem:sumphi}
 Let $\Lambda^{\mathtt{buffer.e} (\tilde{k})}=\lambda I+ \sum_{\tau=1}^{\mathtt{buffer.e} (\tilde{k})}\phi_\tau \phi_\tau^T$ where $\phi_\tau \in \mathbb{R}^d$ and $\lambda>0$. Then it holds 
\[
\sum_{\tau=1}^{\mathtt{buffer.e} (\tilde{k})} \phi_\tau^T (\Lambda^{\mathtt{buffer.e} (\tilde{k})})^{-1}\phi_\tau\le d. 
\]
\end{lemma}

Thus, with $\|\omega_h^{\mathtt{buffer.e} (\tilde{k})}\|=\max_{\nu:\|\nu\|=1}|\nu^T \omega_h^{\mathtt{buffer.e} (\tilde{k})}|$, it ends the proof.\qed 

In order to prove the next lemma, we introduce two useful lemmas first. 

\begin{lemma}\label{filtration}
For any given $h$, suppose $\{x_\tau\}_{\tau=1}^\infty$ being a stochastic process on state space $\mathcal{S}$ with corresponding filtration $\{\mathcal{F}_\tau\}_{\tau=0}^\infty$. Let $\{\phi_\tau\}_{\tau=1}^\infty$ be an $\mathbb{R}^d$-valued stochastic process when $\phi_\tau\in \mathcal{F}_{\tau-1}$. Since $\|\phi_\tau\|\le 1$ and $\Lambda_{\mathtt{buffer.e} (\tilde{k})}=\lambda I +\sum_{\tau=1}^{\mathtt{buffer.e} (\tilde{k})}\phi_\tau\phi_\tau^T$, then for any $\delta$, with probability at least $1-\delta$, for any $k$ corresponding to $\mathtt{buffer.e} (\tilde{k})$ and any $V\in \mathcal{V}$ so that $\sup_x|V(x)|\le 3H$, we have
\begin{align*}
\|\sum_{\tau=1}^k \phi_\tau \{V(x_\tau)-\E [V(x_\tau)\given \mathcal{F}_{\tau-1}]\}\|_{\Lambda_{\mathtt{buffer.e} (\tilde{k})}^{-1}}^2\le& \frac{54C_2 H^3\log ^2 K}{\lambda \log \frac{1}{\gamma}}+\frac{32k^2\epsilon^2}{\lambda}\\
&+144H^2[\frac{d}{2}\log\frac{k+\lambda}{\lambda}+\log \frac{\mathcal{N}_\epsilon}{\delta}],
\end{align*}
where $\mathcal{N}_\epsilon$ is the $\epsilon$-covering number of $\mathcal{V}$ with respect to the distance ${\rm dist}(V,V')=\sup_x(V(x)-V'(x))$.
\end{lemma}
\begin{proof}
First of all, we have
\begin{align*}
&\|\sum_{\tau=1}^k \phi_\tau \{V(x_\tau)-\E [V(x_\tau)\given \mathcal{F}_{\tau-1}]\}\|_{\Lambda_{\mathtt{buffer.e} (\tilde{k})}^{-1}}^2\\
\le & 2 \times2\|\sum_{\tau=1}^k \phi_\tau \{V(x_\tau)-\E [V(x_\tau)\given \mathcal{F}_{\tau-1}]\}\ind\{k\not \in {\mathtt{buffer}} \}\|_{\Lambda_{k}^{-1}}^2+2\times3H\frac{1}{\lambda}3H\frac{3HC_2\log^2 K}{\log\frac{1}{\gamma}}\\
\le &4 \|\sum_{\tau=1}^k \phi_\tau \{V(x_\tau)-\E [V(x_\tau)\given \mathcal{F}_{\tau-1}]\}\|_{\Lambda_{k}^{-1}}^2+\frac{54C_2 H^3\log ^2 K}{\lambda \log \frac{1}{\gamma}} .
\end{align*}
Firstly, we have $(a+b)^2\le 2a^2+2b^2$. Then, it holds since we divide the episodes into two parts: the ones in the buffer and the ones not. For the ones in buffer, due to the definition of $\mathtt{buffer.e} (\tilde{k})$, it is easy to prove that it is smaller than $4\|\sum_{\tau=1}^k \phi_\tau \{V(x_\tau)-\E [V(x_\tau)\given \mathcal{F}_{\tau-1}]\}\ind\{k\not \in \mathtt{buffer} \}\|_{\Lambda_{k}^{-1}}^2$. As for the one not in buffer, $\frac{54C_2 H^3\log ^2 K}{\lambda \log \frac{1}{\gamma}}$ is a trivial bound due to \Cref{lem:buffer} and $V(\cdot)\le 3H$.

Therefore, with Lemma D.4. in \citet{jin2020provably}, we simply replace its $H$ with our upper bound of $V(\cdot)$, i.e., $3H$, and it finishes our proof.
\end{proof}

\begin{lemma}\label{lem:covering}
Let $\mathcal{V}$ denote a class of functions mapping from $\mathcal{S}$ to $\mathbb{R}$ with the following parametric form
\[
V(\cdot)=\min \{\max_a \omega^T \phi(\cdot,\upsilon)+\hat R(\cdot,\upsilon)+\beta\|\phi(\cdot,\upsilon)\|_{\Lambda^{-1}},3H\},
\]
where $\|\omega\|\le L$, $\beta\in[0,B]$ and the minimum eigenvalue satisfies $\lambda_{\min}(\Lambda)\ge \lambda$. Suppose $\|\phi(\cdot,\cdot)\|\le 1$ and let $\mathcal{N}_\epsilon$ be the $\epsilon$-covering number of $\mathcal{V}$ with respect to the distance ${\rm dist}(V,V')=\sup_x|V(x)-V'(x)|$. Then, it holds 
\[
\log \mathcal{N}_\epsilon\le d\log(1+\frac{8L}{\epsilon})+d^2 \log (1+\frac{32\sqrt{d}B^2}{\lambda\epsilon^2})+dN\log(1+\frac{8NB_5\sqrt{d}}{\epsilon}),
\] where $B_5$ is a constant.
\end{lemma}

\begin{proof}
Due to Lemma D.6. in \citet{jin2020provably}, it holds that
\[
{\rm dist}(V_1,V_2)\le \|\omega_1-\omega_2\|+\sqrt{\|A_1-A_2\|_F}+\sup_{x,\upsilon}|\hat R_1(x,\upsilon)-\hat R_2(x,\upsilon)|,
\]where $A=\beta^2 \Lambda^{-1}$.
Let $C_\omega$ be an $\frac{\epsilon}{4}$-cover of $\{\omega\in\mathbb{R}^d\given \|\omega\|\le L\}$, and then it holds $|C_\omega|\le (1+\frac{8L}{\epsilon})^d$. Similarly, for $\frac{\epsilon^2}{16}$-cover for $\{A\}$, we have $|C_A|\le [1+\frac{32B^2\sqrt{d}}{\lambda \epsilon^2}]^{d^2}$.

Now, in order to bound the covering number corresponding to $\hat R(x,\upsilon)$, we show that it links to $\{\hat \theta_i\}_{i=1}^N$ first. As $\hat R(\cdot,\cdot)$ is function of $\{\hat \mu_i\}_{i=1}^N$ and $F(\cdot)$ is differentiable with $|f|\le C_1$, it holds that $\frac{\partial \hat R}{\partial \mu_i}\le B_5$ for any $i$, where $B_5$ is a constant. $B_5$ is bounded since $\mu_i\in[0,1]$ and the interval $[0,1]$ is compact. Therefore, since $\hat \mu=\langle \phi,\hat \theta\rangle$, it holds that
\begin{align*}
\sup_{x,\upsilon}|\hat R_1(x,\upsilon)-\hat R_2(x,\upsilon)|&\le \sup_{\phi:\|\phi\|\le 1} \sum_{i=1}^N B_5 |(\hat \theta_{1i}-\hat \theta_{2i})^T\phi|\\
&\le \sum_{i=1}^N B_5 \|\hat \theta_{1i}-\hat \theta_{2i}\|.
\end{align*}
Therefore, it holds that combining $\frac{\epsilon}{2NB_5}$-cover for $\hat\theta_i$,
\[
|C_{\hat R}|\le (1+\frac{8NB_5\sqrt{d}}{\epsilon})^{dN}.
\] Then, it finishes the proof.
\end{proof}

Now, with lemmas prepared, we have the following lemma.
\begin{lemma}\label{lem:VandPV}
For any $\delta$, with probability at least $1-\delta$, there exists constants $B_6$ and $B_7$ independent of $K$ and $H$ so that 
\begin{align*}
\forall{(k,h)}\in [K]\times[H]:\  \|\sum_{\tau=1}^k \phi_h^\tau [\hat V_{h+1}^k (x_{h+1}^\tau)-\mathbb{P}\hat V_{h+1}^k(x^\tau_h,\upsilon_h^\tau)]\|^2_{{(\Lambda_h^{\mathtt{buffer.e} (\tilde{k})})}^{-1}}\le& B_6 H^3\log^2 K\\
&+B_7 H^2 \log C_7.    
\end{align*}
\end{lemma}

\begin{proof}
Combining \Cref{lem:boundomege2}, \Cref{filtration} and \Cref{lem:covering}, we set $L=3H\sqrt{\frac{dk}{\lambda}}$. With \Cref{algo:estimate}, we have $B=C_7+C_6H\log ^2 K$. Then we have
\begin{align*}
  &\|\sum_{\tau=1}^k \phi_h^\tau [\hat V_{h+1}^k (x_{h+1}^\tau)-\mathbb{P}\hat V_{h+1}^k(x^\tau_h,\upsilon_h^\tau)]\|^2_{({\Lambda_h^{\mathtt{buffer.e} (\tilde{k})})}^{-1}}\\
  \le&  \frac{54C_2H^3\log ^2K}{\lambda\log\frac{1}{\gamma}}+72dH^2\log\frac{k+\lambda}{\lambda}+144H^2d\log(1+\frac{24H}{\epsilon}\sqrt{\frac{dk}{\lambda}})+144H^2\log \frac{KH}{\delta}\\
  &+144H^2d^2 \log[1+\frac{32\sqrt{d}(C_7+C_6H\log^2 K)^2}{\lambda \epsilon^2}]+144H^2dN\log(1+\frac{8NB_5\sqrt{d}}{\epsilon})+\frac{32k^2\epsilon^2}{\lambda}.
\end{align*}
Therefore, by setting $\lambda=1$ and $\epsilon=\frac{dH}{k}$, then we have the right side of the inequality is $\cO (H^3\log ^2 K+H^2\log C_7)$, and it finishes our proof.
\end{proof}

Now, let's show the determination of $C_7$. 
\begin{lemma}\label{lem:omegaandQ}
There exist a constant $B_8$ so that $C_7=B_8H^\frac{3}{2} \log K$, and for any fixed policy $\pi$, on \textrm{Good Event} $\mathscr{E}$, i.e., all inequalities hold, we have for all $(x,\upsilon,h,k)\in \mathcal{S}\times\Upsilon\times[H]\times[K]$ that:
\[
\langle\phi(\cdot,\cdot),\omega_h^k\rangle+\hat R_h^k(\cdot,\cdot) -  Q^\pi_h(\cdot,\cdot)=\mathbb{P}_h(\hat V_{h+1}^k-V^\pi_{h+1})(\cdot,\cdot)+\Delta_h^k(\cdot,\cdot),
\] where $\Delta_h^k(\cdot,\cdot)\le (C_7+C_6H\log ^2K)\|\phi(\cdot,\cdot)\|_{(\Lambda_h^{\mathtt{buffer.e} (\tilde{k})})^{-1}}$.
\end{lemma}
\begin{proof}
Due to Bellman equation, we know that for any $(x,\upsilon,h)\in\mathcal{S}\times\Upsilon\times[H]$, it holds
\[
Q_h^\pi(\cdot,\cdot)=R_h(\cdot,\cdot)+\langle \phi(\cdot,\cdot),\omega_h^\pi\rangle=(R_h+\mathbb{P}_h V_{h+1}^\pi)(\cdot,\cdot).
\]
Therefore, it gives
\[
\langle\phi(\cdot,\cdot),\omega_h^k\rangle+\hat R_h^k(\cdot,\cdot) -  Q^\pi_h(\cdot,\cdot)=\langle \phi(\cdot,\cdot),\omega_h^k-\omega_h^\pi\rangle +(\hat R_h^k-R_h)(\cdot,\cdot).
\]

Then, since $\omega_h^k=\omega_h^{\mathtt{buffer.e} (\tilde{k})}$, it holds that
\begin{align*}
    \omega_h^k-\omega_h^\pi&=(\Lambda_h^{\mathtt{buffer.e} (\tilde{k})})^{-1}\sum_{\tau=1}^{\mathtt{buffer.e} (\tilde{k})}\phi_h^\tau \hat V_{h+1}^k(x_{h+1}^\tau)-\omega_h^\pi\\
    &=(\Lambda_h^{\mathtt{buffer.e} (\tilde{k})})^{-1}\{-\lambda \omega_h^\pi+\sum_{\tau=1}^{\mathtt{buffer.e} (\tilde{k})}\phi_h^\tau[\hat V_{h+1}^k(x_{h+1}^\tau)-\mathbb{P}_h V_{h+1}^\pi(x_h^\tau,\upsilon_h^\tau)] \}\\
    &= \delta_1+\delta_2+\delta_3,
\end{align*} where
\[
\delta_1=-\lambda (\Lambda_h^{\mathtt{buffer.e} (\tilde{k})})^{-1} w_h^\pi,
\]
\[
\delta_2=(\Lambda_h^{\mathtt{buffer.e} (\tilde{k})})^{-1} \sum_{\tau=1}^{\mathtt{buffer.e} (\tilde{k})}\phi_h^\tau[\hat V_{h+1}^k(x_{h+1}^\tau)-\mathbb{P}_h \hat V_{h+1}^k(x_h^\tau,\upsilon_h^\tau)],
\]
\[
\delta_3=(\Lambda_h^{\mathtt{buffer.e} (\tilde{k})})^{-1} \sum_{\tau=1}^{\mathtt{buffer.e} (\tilde{k})}\phi_h^\tau \mathbb{P}_h(\hat V_{h+1}^k-V_{h+1}^\pi)(x_h^\tau,\upsilon_h^\tau).
\]

Then, we begin to bound items corresponding to $\delta_1$, $\delta_2$, and $\delta_3$ individually.

Firstly, it holds
\begin{align*}
|\langle \phi(\cdot,\cdot),\delta_1\rangle|\le& \sqrt{\lambda}\|w_h^\pi\|\|\phi(\cdot,\cdot)\|_{(\Lambda_h^{\mathtt{buffer.e} (\tilde{k})})^{-1}} \\
\le & 3H\sqrt{d\lambda}\|\phi(\cdot,\cdot)\|_{(\Lambda_h^{\mathtt{buffer.e} (\tilde{k})})^{-1}}.
\end{align*}
The first inequality holds due to Cauchy inequality and $\Lambda_{\mathtt{buffer.e} (\tilde{k})}\succeq \lambda I$. The second inequality holds due to \Cref{lem:boundomega1}.

Secondly, it holds that
\[
|\langle \phi(\cdot,\cdot),\delta_ 2\rangle|\le \sqrt{B_6 H^3\log^2 K+B_7 H^2 \log C_7}\|\phi(\cdot,\cdot)\|_{(\Lambda_h^{\mathtt{buffer.e} (\tilde{k})})^{-1}}.
\]
It holds because of \Cref{lem:VandPV}.

Lastly, we have
\begin{align*}
    \langle \phi(\cdot,\cdot),\delta_3\rangle&=\langle \phi(\cdot,\cdot),(\Lambda_h^{\mathtt{buffer.e} (\tilde{k})})^{-1} \sum_{\tau=1}^{\mathtt{buffer.e} (\tilde{k})}\phi_h^\tau \mathbb{P}_h(\hat V_{h+1}^k-V_{h+1}^\pi)(x_h^\tau,\upsilon_h^\tau)\rangle\\
    &=\langle \phi(\cdot,\cdot),(\Lambda_h^{\mathtt{buffer.e} (\tilde{k})})^{-1} \sum_1^{\mathtt{buffer.e} (\tilde{k})} \phi_h^\tau (\phi_h^\tau)^T\int (\hat V_{h+1}^k-V_{h+1}^\pi)(x')d\mathcal{M}_h(x')\rangle\\
    &=\langle \phi(\cdot,\cdot),\int (\hat V_{h+1}^k-V_{h+1}^\pi)(x')d\mathcal{M}_h(x')\rangle-\lambda \langle \phi(\cdot,\cdot),\int (\hat V_{h+1}^k-V_{h+1}^\pi)d\mathcal{M}_h\rangle\\
    &= \mathbb{P}_h(\hat V^k_{h+1}-V_{h+1}^\pi)(\cdot,\cdot)-\lambda \langle \phi(\cdot,\cdot), (\Lambda_h^{\mathtt{buffer.e} (\tilde{k})})^{-1}\int(\hat V_{h+1}^k-V_{h+1}^\pi)(x')d\mathcal{M}_h(x')\rangle\\
    &\le \mathbb{P}_h(\hat V^k_{h+1}-V_{h+1}^\pi)(\cdot,\cdot)+3H\sqrt{d\lambda}\|\phi(\cdot,\cdot)\|_{(\Lambda_h^{\mathtt{buffer.e} (\tilde{k})})^{-1}}.
\end{align*}
The second and fourth equations hold due to the definition of the operator $\mathbb{P}_h$. The third equation holds due to simple algebraic arrangement. The inequality holds due to Cauchy inequality, $V(\cdot)\le 3H$ and $\Lambda_{\mathtt{buffer.e} (\tilde{k})}\succeq\lambda I$.

With the bounds in hand, we have $\Delta_k^h(\cdot,\cdot)\le(3H\sqrt{d\lambda}+ \sqrt{B_6 H^3\log^2 K+B_7 H^2 \log C_7}+3H\sqrt{d\lambda}+C_6 H\log ^2 K)\|\phi(\cdot,\cdot)\|_{(\Lambda_h^{\mathtt{buffer.e} (\tilde{k})})^{-1}} $. Then, it is obviously that there exists a constant $B_8$, so that $B_8 H^{\frac{3}{2}} \log K\ge 3H\sqrt{d\lambda}+ \sqrt{B_6 H^3\log^2 K+B_7 H^2 \log C_7}+3H\sqrt{d\lambda} $ and it finishes the proof.
\end{proof}

Now, we are ready to show the reason why we chose such a bonus. We have the following lemma.
\begin{lemma}\label{lem:ucb}
Under the setting of \Cref{thm:knownf}, on the \textrm{Good Event} $\mathscr{E}$, it holds that for any $(x,\upsilon,h,k)\in \mathcal{S}\times\Upsilon\times[H]\times[K]$,
\[
\hat Q_h^k(x,\upsilon)\le Q^{\pi^*}_h(x,\upsilon).
\]
\end{lemma}
\begin{proof}
We will prove this lemma by induction.

First of all, for the last step $H$, since the value function is zero at $H+1$, we have
\[
|\hat R_H^k(\cdot,\cdot)+\langle \phi(\cdot,\cdot),\omega_H^k\rangle-Q^{\pi^*}_H(\cdot,\cdot)|\le  (C_7+C_6H\log ^2K)\|\phi(\cdot,\cdot)\|_{(\Lambda^{\mathtt{buffer.e} (\tilde{k})}_H)^{-1}}
\]due to \Cref{lem:omegaandQ}.
Therefore, we have
\[
Q_H^{\pi^*}(\cdot,\cdot)\le \min \{\hat R_H^k(\cdot,\cdot)+\langle \phi(\cdot,\cdot),\omega_H^k\rangle+(C_7+C_6H\log ^2K)\|\phi(\cdot,\cdot)\|_{(\Lambda^{\mathtt{buffer.e} (\tilde{k})}_H)^{-1}},3H\},
\]and we use $Q_H^k(\cdot,\cdot)$ to represent the right side.

Now, supposing the statement holds at step $h+1$, then for step $h$, with \Cref{lem:omegaandQ}, it holds that
\[
|[\hat R_h^k+\langle \phi,\omega_h^k\rangle-Q_h^{\pi^*}-\mathbb{P}_h(V_{h+1}^k-V_{h+1}^{\pi^*})](\cdot,\cdot)|\le (C_7+C_6H\log ^2K)\|\phi(\cdot,\cdot)\|_{(\Lambda^{\mathtt{buffer.e} (\tilde{k})}_h)^{-1}}.
\]
By the induction assumption that $\mathbb{P}_h(V_{h+1}^k-V_{h+1}^{\pi^*})(\cdot,\cdot)\ge 0$, it holds that
\[
Q_h^{\pi^*}(\cdot,\cdot)\le  \min \{\hat R_h^k(\cdot,\cdot)+\langle \phi(\cdot,\cdot),\omega_h^k\rangle+(C_7+C_6H\log ^2K)\|\phi(\cdot,\cdot)\|_{(\Lambda^{\mathtt{buffer.e} (\tilde{k})}_h)^{-1}},3H\}=Q_H^k(\cdot,\cdot),
\] which ends the proof.
\end{proof}

Then, we have the following lemma about a recursive formula from $\delta_h^k=V_h^k(x_h^k)-V_h^{\pi_{\tilde k}}(x_h^k)$.
\begin{lemma}\label{lem:recursive}
Let $\delta_h^k=V_h^k(x_h^k)-V_h^{\pi_{\tilde k}}(x_h^k)$ and $\xi_{h+1}^k=\E [\delta_{h+1}^k\given x_h^k,\upsilon_h^k]-\delta_{h+1}^k$. Then conditional on \textrm{Good Event} $\mathscr{E}$, it holds that for any $(k,h)\in[K]\times[H]$,
\[
\delta_h^k\le \delta_{h+1}^k+\xi_{h+1}^k+2(C_7+C_6H\log ^2K)\|\phi(\cdot,\cdot)\|_{(\Lambda^{\mathtt{buffer.e} (\tilde{k})}_h)^{-1}}.
\]
\end{lemma}
\begin{proof}
Due to \Cref{lem:omegaandQ}, it holds that
\[
\hat Q_h^k(\cdot,\cdot)-Q_h^{\pi_{\tilde k}}(\cdot,\cdot)\le \mathbb{P}_h(V_{h+1}^k-V_{h+1}^{\pi_{\tilde k}})(\cdot,\cdot)+2(C_7+C_6H\log ^2K)\|\phi(\cdot,\cdot)\|_{(\Lambda^{\mathtt{buffer.e} (\tilde{k})}_h)^{-1}}.
\]
Then, since $\pi_{\tilde{k}}=\pi_{\mathtt{buffer.e} (\tilde{k})}$ is the greedy policy before mixture at episode $k$ by \Cref{algo:estimate}, we have
\[
\delta_h^k=Q_h^k(x_h^k,\upsilon_h^k)-Q_h^{\pi_{\tilde k}}(x_h^k,\upsilon_h^k).
\]
Then, it ends the proof.\qed

With these preparations, we begin to prove \Cref{lem:lsvi}.

Using notations in \Cref{lem:recursive}, it holds that conditional on \textrm{Good Event} $\mathscr{E}$
\begin{align*}
    \Delta_1&=\sum_{\tau=1}^K[V_1^{\pi^*}(x_1^k)-V_1^{\pi_{\tilde k}}(x_1^k)]\ind(k\not \in \mathtt{buffer} )\\
    &\le \sum_{\tau=1}^K\delta_1^k\ind(k\not \in \mathtt{buffer} )\\
    &\le \sum_{\tau=1}^K\sum_{h=1}^H \xi_h^k+2(C_7+C_6H\log ^2K)\|\phi(\cdot,\cdot)\|_{(\Lambda^{\mathtt{buffer.e} (\tilde{k})}_h)^{-1}}\ind(k\not \in \mathtt{buffer} )\\
    &\le \sum_{\tau=1}^K\sum_{h=1}^H \xi_h^k+2\sqrt{2}(C_7+C_6H\log ^2K)\|\phi(\cdot,\cdot)\|_{(\Lambda^{k}_h)^{-1}}\ind(k\not \in \mathtt{buffer} )\\
    &\le \sum_{\tau=1}^K\sum_{h=1}^H \xi_h^k+2\sqrt{2}(C_7+C_6H\log ^2K)\|\phi(\cdot,\cdot)\|_{(\Lambda^{k}_h)^{-1}}.
\end{align*}
The first inequality holds due to \Cref{lem:ucb}, while the second one holds due to \Cref{lem:recursive}. The third inequality holds due to the process of \Cref{algo:KnownF}, while the last one is trivial.

For the first term, since the
computation of $\hat V_h^k(\cdot)$ is independent of the new observation $x_h^k$ at episode $k$, we obtain that $\{\xi^k_h\}$ is a martingale difference sequence satisfying $|\xi_h^k|\le 3H$ for all $(k,h)$. Therefore, with Azuma-Hoeffding inequality \citep{hoeffding1994probability}, it holds
\[
\Pr(\sum_{\tau=1}^K\sum_{h=1}^H \xi_h^k\ge \epsilon)\ge \exp(-\frac{\epsilon^2}{18KH^3}).
\]
Then, with probability at least $1-\delta$, we have
\[
\sum_{\tau=1}^K\sum_{h=1}^H \xi_h^k\le \sqrt{18KH^3\log\frac{1}{\delta}}.
\]

For the second term, thanks to \citet{abbasi2011improved}, it holds that
\[
\sum_{\tau=1}^K (\phi_h^\tau)^T (\Lambda_h^\tau)^{-1}\phi_h^\tau\le 2d\log\frac{\lambda+\tau}{\lambda}.
\]
Then, with the Cauchy inequality, we have
\[
\sum_{\tau=1}^K \sum_{h=1}^H \|\phi_h^\tau\|_{{(\Lambda_h^\tau)}^{-1}}\le \sum_{h=1}^H \sqrt{K} [\sum_{\tau=1}^K (\phi_h^\tau)^T (\Lambda_h^\tau)^{-1}\phi_h^k]^{\frac{1}{2}}\le H\sqrt{2dK\log\frac{\lambda+K}{\lambda}}.
\]

Finally, combining the two terms, we have
\begin{align*}
\Delta_1 &\le \sqrt{18KH^3\log\frac{1}{\delta}}+2\sqrt{2}(C_7+C_6H\log ^2K)H\sqrt{2dK\log\frac{\lambda+K}{\lambda}}\\
&\le C_8 H^{2.5}\sqrt{K\log^5 K},
\end{align*}
and it finishes our proof.
\end{proof}

\section{Auxiliary Lemmas and Proofs in \texorpdfstring{\Cref{sec:proofunknownf}}{sec:proofunknownf}}
In this section, we provide proof of lemmas in \Cref{sec:proofunknownf} in detail. We organize this section in the order of lemmas.
\subsection{Proof of \texorpdfstring{\Cref{lem:unknownbuffer}}{lem:unknownbuffer}}
In \Cref{algo:unKnownF}, there are two types of $\{\mathtt{buffer.e} (\tilde{k})\}$. The number of $\{\mathtt{buffer.e} (\tilde{k})\}$ satisfying $2(\Lambda_h^k)^{-1}\not \succeq (\Lambda_h^{\mathtt{buffer.e} (\tilde{k})})^{-1}$ is smaller than $\frac{3C_2H\log^2 K}{\log \frac{1}{\gamma}}$ due to \Cref{lem:buffer}. The number of $\{\mathtt{buffer.e} (\tilde{k})\}$ when $\log_2 k$ is an integer is smaller than $[\log_2 K]+1$. Combining the two parts finishes the proof.\qed 

\subsection{Proof of \texorpdfstring{\Cref{lem:unknownbound_lie}}{}}
Since we have a buffer period, the upper bound of the size of overbid or underbid is the same as the situation when the market noise distribution is known. Then, recall that the proof of \Cref{lem:bound_lie} is conditional on reserve price and others' bid, it doesn't matter whether we consider $q$ or $\tilde q$ because the only difference between them is the way generating reserve has become $\pi_0$. Conditional on reserve, the proof of \Cref{lem:bound_lie} still holds regarding $\tilde q$.

With the same methodology in \Cref{lem:bound_lie}, we have the lemma due to \Cref{lem:unknownbuffer}.\qed 

\subsection{Proof of \texorpdfstring{\Cref{lem:unknownglm}}{lem:unknownglm}}
Similar to the proof of \Cref{lem:glm}, we replace $1-F(m_\tau-1-\langle\phi_\tau,\theta\rangle)$ by $\frac{1}{3N}(1+\langle\phi_\tau,\theta\rangle)$ to form \Cref{algo:Fandthetahat}. We just need to prove that $\E [\tilde q-\frac{1}{3N}(1+\langle \phi_\tau,\theta\rangle)]=0$ if bidders bid truthfully. If $\tilde q_{ih}^\tau=1$, it satisfies that we choose $i$ using $\pi_0$ with reserve price $\rho_i$ and $1+\langle \phi_\tau,\theta\rangle +z\ge \rho_i$. With some conditional probability calculation, the probability is $\frac{1}{3N}(1+\langle \phi_\tau,\theta\rangle)$.

Therefore, by simply setting $c_1=C_1=\frac{1}{3N}$ in \Cref{lem:glm}, we prove \Cref{lem:unknownglm}. \qed  

We now discuss the intuition behind the estimator of $\theta$ in \Cref{algo:Fandthetahat}. It is a constrained form of ridge regression. First, let’s clarify the rationale behind the correct choice of the loss function. In comparison to \Cref{algo:thetahat}, where the noise distribution $F$ is known, our current scenario lacks this knowledge. Consequently, we cannot construct a random variable $q-1+F$, prompting the need to identify a new zero-mean random variable for estimating $\theta$. This is why simulation is employed to generate $\tilde{q}$ as distinct from
$q$. We have $\E [3N\tilde q-(1+\langle \phi_\tau,\theta\rangle)]=0$. Thus, $\tilde q$
facilitates the estimation of $\theta$ using $L_2$-norm constrained ridge via \Cref{algo:Fandthetahat}, while $q$ persists in exploiting to prevent excessive regret. Second, instead of resorting to ordinary ridge regression, we constrain the norm of the parameters. This choice is driven by the confined space, as outlined in \Cref{assumption:linearmdp}.

\subsection{Proof of \texorpdfstring{\Cref{lem:boundF}}{lem:boundF}}
\re{In order to estimate $F(\cdot)$ precisely. We need to bound two-fold errors. First, we need to bound errors coming from randomness. Second, we need to bound errors from untruthful bidding.}

First of all, if every buyer bids truthfully, then with \Cref{lem:DKW}, it holds with probability at least $1-\frac{\delta}{K}$ for each update that
\[
|F(\cdot)-\hat F(\cdot)|\le \sqrt{\frac{1}{2}\log \frac{2K}{\delta}}(NH\mathtt{buffer.e} (\tilde{k}))^{-\frac{1}{2}}.
\]

However, bidders may overbid or underbid for less than $\frac{C_3H}{K}$ due to \Cref{overbid} and the estimation of $\mu$ has an error. Therefore, the c.d.f that $\hat F(\cdot)$ estimates is not the same as $F(\cdot)$. Since $|f(\cdot)|\le C_1$, the difference because of overbid or underbid is smaller than $\frac{C_1C_3 H}{K}$. Then, due to \Cref{lem:unknowndiffmu}, the difference because of error in $\mu$ is smaller than 
\[
C_1C_{11}\sqrt{H}\log K \frac{\sum_{h=1}^H\sum_{\tau=1}^{\mathtt{buffer.e} (\tilde{k})}\|\phi(x_h^\tau,\upsilon_h^\tau)\|_{({\Lambda_h^{\mathtt{buffer.e} (\tilde{k})}})^{-1}}}{H\mathtt{buffer.e} (\tilde{k})}\le C_1C_{11}\sqrt{H}\log K \frac{\sqrt{d}}{\sqrt{\mathtt{buffer.e} (\tilde{k})}}.
\]The inequality holds since we have the mean value inequality and \Cref{lem:sumphi}.


Since the number of episodes in buffer for each buyer $i$ is no larger than $C_{9}H\log^2K$, it holds that
\begin{align*}
|F(\cdot)-\hat F(\cdot)|\le &\sqrt{\frac{1}{2}\log \frac{2K}{\delta}} {(NH\mathtt{buffer.e} (\tilde{k}))}^{-\frac{1}{2}}+\frac{C_1C_3 H}{K}+\frac{C_{9} H \log^2 K}{\mathtt{buffer.e} (\tilde{k})}\\
&+C_1C_{11}\sqrt{H}\log K \frac{\sqrt{d}}{\sqrt{\mathtt{buffer.e} (\tilde{k})}}.    
\end{align*}

Because the number of episodes we run \Cref{algo:Fandthetahat} is smaller than $K$, then the total probability of happening \textrm{Bad Event} $\mathscr{E}^c$ is smaller than $\delta$. Then, it ends the proof. \qed

\subsection{Proof of \texorpdfstring{\Cref{unknownboundR}}{unknownboundR}}
In order to prove \Cref{unknownboundR}, we introduce the following lemma first.
\begin{lemma}\label{lem:boundf}
Under assumption \Cref{assumptionfdiff}, when \Cref{lem:boundF} holds, using histogram method to estimate p.d.f $f(\cdot)$ leads to the following bound that for any $x$
\[
|f(x)-\hat f(x)|\le D_1 \frac{\sqrt{H}\log K}{{\mathtt{buffer.e} (\tilde{k})}^\frac{1}{4}},
\]where $D_1$ is a constant.
\end{lemma}
\subsubsection{Proof of \Cref{lem:boundf}}
With \Cref{lem:boundF} in hand, we divide $[-1,1]$ into $2M$ parts denoted by $\{-M,\ldots,0,\ldots,M-1\}$ uniformly, then we have 
\[
\hat f(x)=M[\hat F(\frac{i+1}{M})-\hat F(\frac{i}{M}) ],
\]where $x\in(\frac{i}{M},\frac{i+1}{M}]$.

Under assumption \Cref{assumptionfdiff}, it holds that 
\[
|f(x)-M[ F(\frac{i+1}{M})- F(\frac{i}{M}) ]|\le \frac{L}{M}.
\]

Therefore, it holds that 
\[
|f(x)-\hat f(x)|\le 2M C_{12}\frac{H\log ^2K}{\sqrt{\mathtt{buffer.e} (\tilde{k})}}+\frac{L}{M}.
\]

By setting $M=\frac{{\mathtt{buffer.e} (\tilde{k})}^\frac{1}{4}}{\sqrt{H}\log K}$, we finish our proof.\qed 

Therefore, unlike \Cref{lem:reservegap}, we have the following lemma.

\begin{lemma}\label{lem:unknownreservegap}
Under \Cref{logconcave}, it holds that
\[
|\alpha_{ih}^{k*}-\alpha_{ih}^k|\le |\langle \phi_h^k, \theta_{ih}-\hat \theta_{ih}\rangle|+\frac{D_2H\log^2 K}{{\mathtt{buffer.e} (\tilde{k})}^\frac{1}{4}},
\]where $D_2$ is a constant.
\end{lemma}
\subsubsection{Proof of \Cref{lem:unknownreservegap}}
\citet{myerson1981optimal} shows that the optimal reserve price satisfies
\[
\alpha=1+\mu(\cdot,\cdot)+\phi^{-1}(-1-\mu(\cdot,\cdot)),
\]where $\phi(x)=x-\frac{1-F(x)}{f(x)}$ is virtual valuation function. 

We use $\alpha^*$ to denote the optimal reserve price, while $\hat \alpha$ denotes the reserve price we use with $\hat F(\cdot)$ and $\hat f(\cdot)$. Also, we use $\tilde \alpha$ to denote reserve price corresponding to $\hat \mu$, $F(\cdot)$ and $f(\cdot)$.

\Cref{lem:reservegap} shows that $|\tilde \alpha-\alpha^*|\le |\langle \phi_h^k, \theta_{ih}-\hat \theta_{ih}\rangle|$.

To bound $|\tilde \alpha-\hat \alpha|$, we have 
\begin{align*}
    |\frac{1-F(\cdot)}{f(\cdot)}-\frac{1-\hat F(\cdot)}{\hat f(\cdot)}|&\le |\frac{1-F(\cdot)}{f(\cdot)}-\frac{1-\hat F(\cdot)}{f(\cdot)}|+|\frac{1-\hat F(\cdot)}{f(\cdot)}-\frac{1-\hat F(\cdot)}{\hat f(\cdot)}|\\
    &\le \frac{C_{12} H \log^2 K}{c_1\sqrt{\mathtt{buffer.e} (\tilde{k})}}+\frac{D_1\sqrt{H}\log K}{c_1^2 {\mathtt{buffer.e} (\tilde{k})}^\frac{1}{4}}.
\end{align*}
The first inequality holds due to the triangle inequality. The second inequality holds due to \Cref{assumptionf}, \Cref{lem:boundF} and \Cref{lem:boundf}.

Then, we will show that $\phi'(\cdot)\ge 1$.

It holds that $\phi(x)=x-\frac{1-F(x)}{f(x)}=x+\frac{1}{\log'(1-F(x))}$. Under \Cref{logconcave}, it holds that $1-F(\cdot)$ is log-concave implying $\log'(1-F(\cdot))$ is decreasing. Therefore, $\phi'(x)\ge 1$.

Therefore, we have $|\phi(\hat \alpha)-\hat \phi(\hat \alpha)|\le \frac{C_{12} H \log^2 K}{c_1\sqrt{\mathtt{buffer.e} (\tilde{k})}}+\frac{D_1\sqrt{H}\log K}{c_1^2 {\mathtt{buffer.e} (\tilde{k})}^\frac{1}{4}}$ and $\phi(\tilde \alpha)=\hat \phi(\hat \alpha)$. Then, it holds that 
\[
|\hat \alpha-\tilde \alpha|\le  \frac{C_{12} H \log^2 K}{c_1\sqrt{\mathtt{buffer.e} (\tilde{k})}}+\frac{D_1\sqrt{H}\log K}{c_1^2 {\mathtt{buffer.e} (\tilde{k})}^\frac{1}{4}},
\] because $\phi'(\cdot)\ge 1$.

Then, it ends our proof. \qed 

Now, we are ready to prove \Cref{unknownboundR}.
Using notations in \Cref{lem:estimater}, we use another factor $F$ to show that we use $F(\cdot)$ and $f(\cdot)$ in the function, while factor $\hat F$ to denote the use of $\hat F(\cdot)$ and $\hat f(\cdot)$. 

With the same methodology in \Cref{lem:estimater}, it holds that 
\begin{align*}
|R_h^k(\cdot,\cdot,F)-\hat R_h^k(\cdot,\cdot,F)| \le& [(1+6C_1)C_{11}\sqrt{H}\log K] N \|\phi(\cdot,\cdot)\|_{({\Lambda_h^{\mathtt{buffer.e} (\tilde{k})}})^{-1}}\\
&+\frac{NB_4}{2} [2(|\langle \phi_h^k, \theta_{ih}-\hat \theta_{ih}\rangle|)^2+2(\frac{D_2H\log^2 K}{{\mathtt{buffer.e} (\tilde{k})}^\frac{1}{4}})^2]\\
\le &D_3H\log ^2 K \|\phi(\cdot,\cdot)\|_{({\Lambda_h^{\mathtt{buffer.e} (\tilde{k})}})^{-1}}+D_4 H^2 \log^4 K \frac{1}{\sqrt{\mathtt{buffer.e} (\tilde{k})}},
\end{align*} where $D_3$ and $D_4$ are two constants.
The first inequality holds since $(a+b)^2\le 2(a^2+b^2)$. The second inequality holds by rearrangement.

Then, we will bound $|\hat R_h^k(\cdot,\cdot,F)-\hat R_h^k(\cdot,\cdot,\hat F)|$.

Since $\hat R_h^k(\cdot,\cdot,F)=\sum_{i=1}^N \E_F [\max \{\hat r_{ih}^{k-} ,\alpha_{ih}^k\}\ind (\hat r_{ih}^k\ge \max \{\hat r_{ih}^{k-} ,\alpha_{ih}^k\} )]$ and $\hat R_h^k(\cdot,\cdot,\hat F)=\sum_{i=1}^N \E_{\hat F} [\max \{\hat r_{ih}^{k-} ,\alpha_{ih}^k\}\ind (\hat r_{ih}^k\ge \max \{\hat r_{ih}^{k-} ,\alpha_{ih}^k\} )]$, we have that the difference of expected revenue about each buyer is smaller than $3NC_{12}\frac{H\log^ 2K}{\sqrt{\mathtt{buffer.e} (\tilde{k})}}$. It comes from the fact that the expected revenue depends on an $N$-fold integral with respect to the random variable $\{z_{ih}^k\}_{i=1}^N$. Since $\int x (dF-dF')=-\int (F-F')dx \le 3 \|F-F'\|_\infty\le 3C_{12}\frac{H\log^ 2K}{\sqrt{\mathtt{buffer.e} (\tilde{k})}}$, each integral has error less than $3C_{12}\frac{H\log^ 2K}{\sqrt{\mathtt{buffer.e} (\tilde{k})}}$. With $N$ buyers in total, it holds that 
\[
|\hat R_h^k(\cdot,\cdot,F)-\hat R_h^k(\cdot,\cdot,\hat F)|\le 3N^2C_{12}\frac{H\log^ 2K}{\sqrt{\mathtt{buffer.e} (\tilde{k})}}.
\]

Combining the two parts, it holds 
\begin{align*}
|R_h^k(\cdot,\cdot)-\hat R_h^k(\cdot,\cdot)|&=|R_h^k(\cdot,\cdot,F)-\hat R_h^k(\cdot,\cdot,\hat F)|\\
&\le C_{13}H \log^2 K \|\phi(\cdot,\cdot) \|_{(\Lambda_h^{\mathtt{buffer.e} (\tilde{k})})^{-1}}+\frac{C_{14} H^2\log^4K}{\sqrt{\mathtt{buffer.e} (\tilde{k})}},    
\end{align*}
which ends the proof. \re{Similarly, we can use \Cref{assumptionf} to achieve parallel results without \Cref{logconcave} as \Cref{redundant} says.}\qed 

\subsection{Proof of \texorpdfstring{\Cref{lem:unknownlsvi}}{lem:unknownlsvi}}
Now, we introduce some lemmas in parallel in order to prove \Cref{lem:unknownlsvi}.

\begin{lemma}\label{unknownfiltration}
For any given $h$ omitted for convenience, suppose $\{x_\tau\}_{\tau=1}^\infty$ being a stochastic process on state space $\mathcal{S}$ with corresponding filtration $\{\mathcal{F}_\tau\}_{\tau=0}^\infty$. Let $\{\phi_\tau\}_{\tau=1}^\infty$ be an $\mathbb{R}^d$-valued stochastic process when $\phi_\tau\in \mathcal{F}_{\tau-1}$. Since $\|\phi_\tau\|\le 1$ and $\Lambda_{\mathtt{buffer.e} (\tilde{k})}=\lambda I +\sum_{\tau=1}^{\mathtt{buffer.e} (\tilde{k})}\phi_\tau\phi_\tau^T$, then for any $\delta$, with probability at least $1-\delta$, for any $k$ corresponding to $\mathtt{buffer.e} (\tilde{k})$ and any $V\in \mathcal{V}$ so that $\sup_x|V(x)|\le 3H$, we have 
\begin{align*}
\|\sum_{\tau=1}^k \phi_\tau \{V(x_\tau)-\E [V(x_\tau)\given \mathcal{F}_{\tau-1}]\}\|_{\Lambda_{\mathtt{buffer.e} (\tilde{k})}^{-1}}^2\le& \frac{54C_9 H^3\log ^2 K}{\lambda \log \frac{1}{\gamma}}+\frac{32k^2\epsilon^2}{\lambda}\\
&+144H^2[\frac{d}{2}\log\frac{k+\lambda}{\lambda}+\log \frac{\mathcal{N}_\epsilon}{\delta}],    
\end{align*}
where $\mathcal{N}_\epsilon$ is the $\epsilon$-covering number of $\mathcal{V}$ with respect to the distance ${\rm dist}(V,V')=\sup_x(V(x)-V'(x))$.
\end{lemma}

\begin{lemma}\label{lem:unknowncovering}
Let $\mathcal{V}$ denote a class of functions mapping from $\mathcal{S}$ to $\mathbb{R}$ with the following parametric form
\[
V(\cdot)=\min \{\max_a \omega^T \phi(\cdot,\upsilon)+\hat R(\cdot,\upsilon)+\beta\|\phi(\cdot,\upsilon)\|_{\Lambda^{-1}}+A,3H\},
\]
where $\|\omega\|\le L$, $\beta\in[0,B]$, $A=\frac{C_{14}H^2\log ^4 K}{\sqrt{\mathtt{buffer.e} (\tilde{k})}}$ in episode $k$ and the minimum eigenvalue satisfies $\lambda_{\min}(\Lambda)\ge \lambda$. Suppose $\|\phi(\cdot,\cdot)\|\le 1$ and let $\mathcal{N}_\epsilon$ be the $\epsilon$-covering number of $\mathcal{V}$ with respect to the distance ${\rm dist}(V,V')=\sup_x|V(x)-V'(x)|$. Then, it holds 
\[
\log \mathcal{N}_\epsilon\le d\log(1+\frac{8L}{\epsilon})+d^2 \log (1+\frac{32\sqrt{d}B^2}{\lambda\epsilon^2})+dN\log(1+\frac{16NB_5\sqrt{d}}{\epsilon})+\log \mathcal{N}_{\frac{\epsilon}{12N^2}}(\mathcal{F}),
\] where $B_5$ is a constant.
\end{lemma}

\subsubsection{Proof of \texorpdfstring{\Cref{lem:unknowncovering}}{lem:unknowncovering}}
When $F(\cdot)$ is unknown, it holds that
\begin{align*}
    \sup_{x,\upsilon}|\hat R_1(x,\upsilon)-\hat R_2(x,\upsilon)|=&\sup_{x,\upsilon}|\hat R_1(x,\upsilon,\hat F_1)-\hat R_2(x,\upsilon,\hat F_2)|\\
    \le& \sup_{x,\upsilon}|\hat R_1(x,\upsilon,\hat F_1)-\hat R_2(x,\upsilon,\hat F_1)|\\
    &+\sup_{x,\upsilon}|\hat R_2(x,\upsilon,\hat F_1)-\hat R_2(x,\upsilon, \hat F_2)|.
\end{align*}

Then, we use $C_{\hat \theta}$ to denote the cardinality of the balls corresponding to $\hat \theta$ and $C_\mathcal{F}$ to denote the cardinality of the balls corresponding to $\mathcal{F}$.

Like the proof of \Cref{lem:covering}, we simply use $\frac{\epsilon}{4NB_5}$-ball to cover $\hat \theta_i$, and it holds that 
\[
|C_{\hat \theta}|\le (1+\frac{16NB_5\sqrt{d}}{\epsilon})^{dN}.
\]

Conditional on $\omega$, $A$ and $\{\hat \theta_i\}_{i=1}^N$, with \Cref{unknownboundR}, we know that in order to satisfy $\sup_{x,\upsilon}|\hat R(x,\upsilon,\hat F)-\hat R(x,\upsilon,  F)|\le\frac{\epsilon}{4}$, what we need is $\|\hat F-F\|_\infty\le \frac{\epsilon}{12N^2}$. Then, it ends the proof.\qed

Then, it holds the following lemma.
\begin{lemma}\label{lem:unknownVandPV}
For any $\delta$, with probability at least $1-\delta$, there exists constants $B_6$ and $B_7$ independent of $K$ and $H$ so that 
\begin{align*}
\forall{(k,h)}\in [K]\times[H]:\  \|\sum_{\tau=1}^k \phi_h^\tau [\hat V_{h+1}^k (x_{h+1}^\tau)-\mathbb{P}\hat V_{h+1}^k(x^\tau_h,\upsilon_h^\tau)]\|^2_{{(\Lambda_h^{\mathtt{buffer.e} (\tilde{k})})}^{-1}}\le& D_5 H^3\\
+&D_6 H^2 \log C_{15},
\end{align*}where $D_5\sim \tilde \cO (1)$ omitting $\log K$ and $D_6$ is a constant. 
\end{lemma}

\begin{proof}
Similar to the proof of \Cref{lem:VandPV}, we just replace $\mathcal{N}_\epsilon$ by $d\log(1+\frac{8L}{\epsilon})+d^2 \log (1+\frac{32\sqrt{d}B^2}{\lambda\epsilon^2})+dN\log(1+\frac{16NB_5\sqrt{d}}{\epsilon})+\log \mathcal{N}_{\frac{\epsilon}{12N^2}}(\mathcal{F})$. Then, we set $\lambda=1$, $B=C_{15}+C_{13}H \log^2 K $ and $\epsilon=\frac{dH}{k}$. With \Cref{assumption:coveringnumber}, we finish our proof.
\end{proof}
Now, let's show the determination of $C_{15}$.
\begin{lemma}\label{lem:unknownomegaandQ}
There exist $D_7\sim \tilde \cO (1)$ so that $C_{15}=D_7 H^\frac{3}{2}$, and for any fixed policy $\pi$, on \textrm{Good Event} $\mathscr{E}$, i.e., all inequalities hold, we have for all $(x,\upsilon,h,k)\in \mathcal{S}\times\Upsilon\times[H]\times[K]$ that:
\[
\langle\phi(\cdot,\cdot),\omega_h^k\rangle+\hat R_h^k(\cdot,\cdot) -  Q^\pi_h(\cdot,\cdot)=\mathbb{P}_h(\hat V_{h+1}^k-V^\pi_{h+1})(\cdot,\cdot)+\Delta_h^k(\cdot,\cdot),
\] where $\Delta_h^k(\cdot,\cdot)\le (C_{15}+C_{13}H\log ^2K)\|\phi(\cdot,\cdot)\|_{(\Lambda_h^{\mathtt{buffer.e} (\tilde{k})})^{-1}}+C_{14}\frac{H^2\log ^4K}{\sqrt{\mathtt{buffer.e} (\tilde{k})}}$.
\end{lemma}

\begin{proof}
The proof of \Cref{lem:unknownomegaandQ} is the same as proof of \Cref{lem:omegaandQ}. Let's show the determination of $D_7$ in parallel. With \Cref{lem:unknownVandPV} in hand, it holds that
\[
D_7 H^\frac{3}{2}\ge 3H\sqrt{d\lambda}+ \sqrt{D_5 H^3+D_6 H^2 \log C_{15}}+3H\sqrt{d\lambda} .
\]
Then, it is easy to see the existence of $D_7$ where $D_7\sim \tilde \cO (1)$.\qed
\end{proof}

Also, we have the following lemma about the recursive formula from $\delta_h^k=V_h^k(x_h^k)-V_h^{\pi_{\tilde k}}(x_h^k)$. It holds due to \Cref{lem:unknownomegaandQ} and \Cref{lem:ucb}. 
\begin{lemma}\label{lem:unknownrecursive}
Let $\delta_h^k=V_h^k(x_h^k)-V_h^{\pi_{\tilde k}}(x_h^k)$ and $\xi_{h+1}^k=\E [\delta_{h+1}^k\given x_h^k,\upsilon_h^k]-\delta_{h+1}^k$. Then conditional on \textrm{Good Event} $\mathscr{E}$, it holds that for any $(k,h)\in[K]\times[H]$,
\[
\delta_h^k\le \delta_{h+1}^k+\xi_{h+1}^k+2(C_{15}+C_{13}H \log^2 K )\|\phi(\cdot,\cdot) \|_{(\Lambda_h^{\mathtt{buffer.e} (\tilde{k})})^{-1} }+2C_{14}\frac{H^2\log^4 K}{\sqrt{\mathtt{buffer.e} (\tilde{k})}}.
\]
\end{lemma}

Now, we are ready to prove \Cref{lem:unknownlsvi}.

Similar to the proof of \Cref{lem:lsvi}, it holds that 
\[
\Delta_1\lesssim \tilde \cO (\sqrt{H^5{K}})+\sum_{k=1}^K \sum_{h=1}^H  2C_{14}\frac{H^2\log^4 K}{\sqrt{\mathtt{buffer.e} (\tilde{k})}}.
\]

Due to \Cref{algo:unKnownF}, we have $k\le 2\mathtt{buffer.e} (\tilde{k})$. Therefore, it holds that 
\[
\sum_{k=1}^K\frac{1}{\sqrt{\mathtt{buffer.e} (\tilde{k})}}\le \sum_{k=1}^K \frac{\sqrt{2}}{\sqrt{k}}\le 2\sqrt{2K}.
\]

Therefore, it holds that 
\[
\Delta_1\lesssim \tilde \cO (\sqrt{H^5K})+\tilde \cO (H^3 \sqrt{K}),
\]
which ends the proof. \qed

\re{
\section{Detailed Results of Numerical Experiments}\label{app:exp}
In this section, we give some details about our numerical experiments. 
}
\re{
\begin{table}[h]
    \centering
    \begin{tabular}{|c|ccc|}
        \hline
        \textbf{Trail\textbackslash Regret} & \textbf{CLUB} & \textbf{SCORP} & \textbf{NPAC-S} \\
        \hline
        1 & \textbf{57.20} & 170.77 & 131.41 \\
        \hline
        2 & 139.75 & 230.29 & \textbf{113.23} \\
                \hline
        3 & 58.01 & 189.06 & \textbf{41.46} \\
                \hline
        4 & 238.57 & 168.39 & \textbf{54.59} \\
                \hline
        5 & 79.43 & 161.72 & \textbf{59.99} \\
                \hline
        6 & 171.67 & 211.33 & \textbf{53.72} \\
                \hline
        7 & \textbf{52.24} & 204.67 & 185.61 \\
                \hline
        8 & \textbf{59.40} & 185.07 & 135.82 \\
                \hline
        9 & 228.57 & 176.15 & \textbf{37.69} \\
                \hline
        10 & 150.11 & 181.72 & \textbf{91.58} \\
                \hline
        11 & \textbf{80.74} & 197.85 & 123.08 \\
                \hline
        12 & 179.27 & \textbf{167.39} & 239.79 \\
                \hline
        13 & \textbf{37.25} & 186.11 & 56.14 \\
                \hline
        14 & \textbf{83.27} & 168.86 & 240.07 \\
                \hline
        15 & \textbf{54.92} & 163.89 & 219.48 \\
                \hline
        16 & \textbf{72.72} & 175.39 & 86.02 \\
                \hline
        17 & 56.35 & 174.99 & \textbf{35.80} \\
                \hline
        18 & 55.40 & 178.67 & \textbf{52.55} \\
                \hline
        19 & \textbf{34.40} & 170.65 & 70.55 \\
                \hline
        20 & \textbf{15.57} & 160.40 & 169.44 \\
                \hline
        21 & \textbf{95.18} & 164.27 & 171.89 \\
                \hline
        22 & 324.05 & 176.15 & \textbf{24.25} \\
                \hline
        23 & 184.31 & 174.79 & \textbf{30.46} \\
                \hline
        24 & \textbf{41.43} & 174.32 & 64.36 \\
                \hline
        25 & \textbf{51.32} & 171.11 & 89.65 \\
                \hline
        26 & \textbf{30.47} & 177.63 & 191.52 \\
                \hline
        27 & \textbf{30.46} & 178.80 & 58.29 \\
                \hline
        28 & 367.42 & 182.17 & \textbf{84.62} \\
                \hline
        29 & 54.69 & 171.78 & \textbf{44.27} \\
                \hline
        30 & 114.49 & 174.32 & \textbf{33.29} \\
        \hline
    \end{tabular}
    \caption{Regrets of three different algorithms in each trail.}
    \label{tab:expbandits}
\end{table}}

\re{In the contextual bandits setting, we show the total regrets of three different algorithms (i.e., CLUB, SCORP and NPAC-S) in all 30 trails in the following table. Among all 30 trials, CLUB has the lowest regret in 15 trials, while NPAC-S does in 14 trials. SCORP only wins in the twelfth trial. For their average regrets, it's 106.62 for CLUB, 178.96 for SCORP, and 99.69 for NPAC-S. Therefore, we conclude that for contextual bandit settings, the performances of CLUB and NPAC-S are comparable, overwhelming the performance of SCORP sufficiently. 
}

\re{
For the implementation details, we assume $N=1$, and there are two different contexts, both appearing with probability 0.5. Besides, we assume $\theta=[0,4,0.6]^T$ and underlying noise distribution is {\rm Unif}([-1,1]). In order to distinguish these strategies, we constrain that bids must be a multiple of 0.01. To simulate strategic bidders, we use \Cref{overbid}. Once it's in the buffer period, we assume bidders randomly bid. However, if not, we assume bidders bid their value plus a random noise with scale $\frac{C_3}{K}$. For NPAC-S, we use similar ways to simulate strategic behaviors. However, for SCORP, we stated before that it uses too many episodes to explore; we loosen its constraints and assume truthful bidding. Although we only consider an upper bound for its performance, SCORP still performs worse than CLUB and NPAC-S. So, we only compare CLUB and NPAC-S in MDP settings. To solve \Cref{algo:Fandthetahat}, we seek help from \texttt{scipy.optimize} package. Actually, most of the running time is spent on solving \Cref{algo:Fandthetahat}. We believe we can reduce our running time by using other commercial optimization solvers. 
}

\re{In the MDP setting, we show the total regrets of CLUB and NPAC-S in the 30 trials in \Cref{tab:MDP}. Among all 30 trials, CLUB wins NPAC-S every time. The average of CLUB is 203.07, overwhelming the corresponding 756.31 for NPAC-S. As a result, it shows that CLUB has better performance against NPAC-S in the MDP setting.}
\begin{table}[h]
    \centering
    \begin{tabular}{|c|cc|c|cc|}
        \hline
        \textbf{Trail\textbackslash Regret} & \textbf{CLUB} & \textbf{NPAC-S} & \textbf{Trail\textbackslash Regret} & \textbf{CLUB} & \textbf{NPAC-S} \\
        \hline
    1&\textbf{111.12} & 719.32 &16&   \textbf{ 202.51}  & 843.94 \\ \hline
    2&\textbf{86.96}  & 744.47 &17&    \textbf{24.77}  & 699.18 \\ \hline
    3&\textbf{369.94}  & 694.44 &18&    \textbf{262.83 } & 709.15 \\ \hline
    4&\textbf{78.32}  & 1204.41 &19&   \textbf{ 505.96 } & 802.21 \\ \hline
    5&\textbf{586.62}  & 660.06 &20&   \textbf{ 163.90}  & 696.09 \\ \hline
    6&\textbf{46.89}  & 647.03 &21&   \textbf{ 33.60}  & 653.59 \\ \hline
    7&\textbf{303.41}  & 695.98 &22&   \textbf{ 156.05}  & 872.66 \\ \hline
    8&\textbf{61.22}  & 698.99 &23&   \textbf{ 46.15 } & 746.18 \\ \hline
    9&\textbf{281.11}  & 686.92 &24&   \textbf{ 388.10}  & 781.76 \\ \hline
    10&\textbf{40.48}  & 742.37 &25&    \textbf{160.19}  & 699.93 \\ \hline
11&    \textbf{125.29}  & 790.36&26&   \textbf{ 552.07}  & 732.08 \\ \hline
12&    \textbf{140.18}  & 744.64 &27&  \textbf{  89.34}  & 734.74 \\ \hline
13&    \textbf{516.48}  & 855.74&28&   \textbf{ 112.73}  & 702.72 \\ \hline
14&    \textbf{55.23}  & 660.48&29&  \textbf{  191.32}  & 663.03 \\ \hline
15&    \textbf{87.22}  & 1002.02 &30&   \textbf{ 311.99}  & 804.94 \\ \hline
    \end{tabular}
    \caption{Regrets of two different algorithms in each trail.}
    \label{tab:MDP}
\end{table}

\re{For the detailed setting of MDP and the implementation, we consider the situation that $H=2$. We state different settings than the ones in contextual bandits as follows. The action space contains two actions. The first action will lead to the first context with probability 1, and the second action will lead to the second context in the next phase. In our MDP setting, we only discount once every episode, which means two phases. Therefore, we set the discount rate to be $\sqrt \gamma$ for NPAC-S. It is a more conservative situation and will decrease the extent of untruthful bidding for NPAC-S. At the same time, we assume NPAC-S will choose actions randomly. For our CLUB algorithm, we construct a 4-dimensional feature space to capture the structure of the underlying MDP. Additionally, instead of selecting $\delta$, we set ${\rm poly_1}(\cdot)=H\log^2(K)$ and ${\rm poly_2}(\cdot)=H^2\log^4(K)$, which decides a unique probability to break our PAC-learning bounds.}

\begin{figure}[!tbp]
    \centering
    \begin{subfigure}[b]{0.3\textwidth}
        \includegraphics[width=\textwidth]{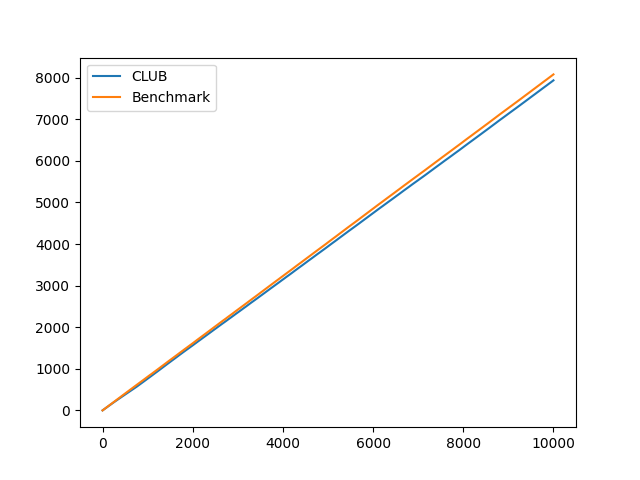}
        \caption{The performance of CLUB against the benchmark.}
        \label{fig:Gauss/CLUB}
    \end{subfigure}
    \hfill
    \begin{subfigure}[b]{0.3\textwidth}
        \includegraphics[width=\textwidth]{ 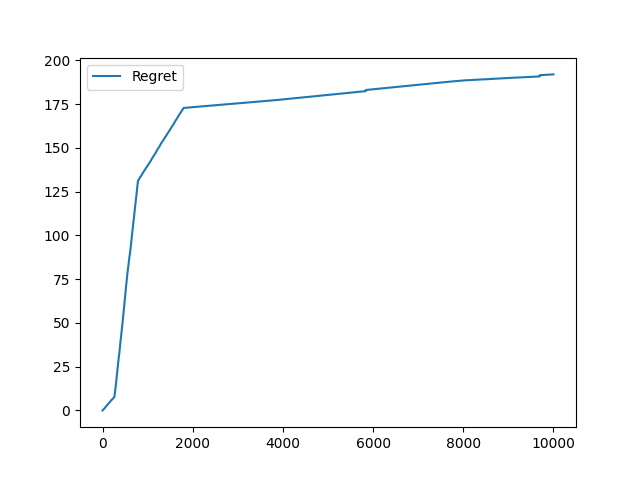}
        \caption{The regret accumulation of CLUB.}
        \label{fig:CLUB_regret}
    \end{subfigure}
    \hfill
    \begin{subfigure}[b]{0.29\textwidth}
        \includegraphics[width=\textwidth]{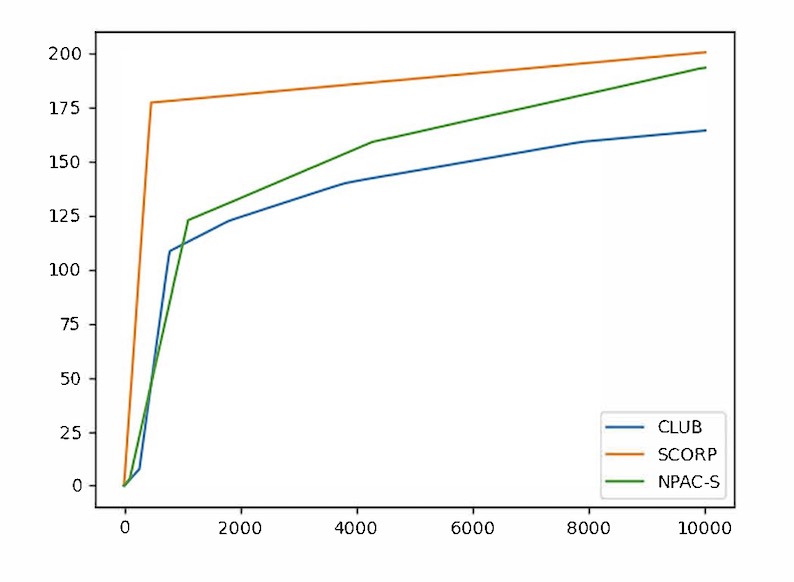}
        \caption{The average performances of three algorithms.}
        \label{fig:Gauss/average}
    \end{subfigure}
    \caption{Experiment results for the contextual bandit setting under truncated Gaussian noise distribution.
    }
    \label{fig:gauss/contextual}
    \centering
    \begin{subfigure}[b]{0.3\textwidth}
        \includegraphics[width=\textwidth]{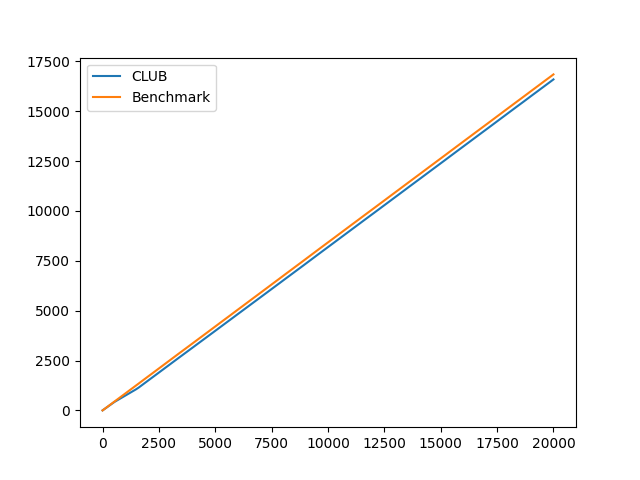}
        \caption{The performance of CLUB against the benchmark.}
        \label{fig:Gauss/MDP+sub1}
    \end{subfigure}
    \hfill
    \begin{subfigure}[b]{0.3\textwidth}
        \includegraphics[width=\textwidth]{ 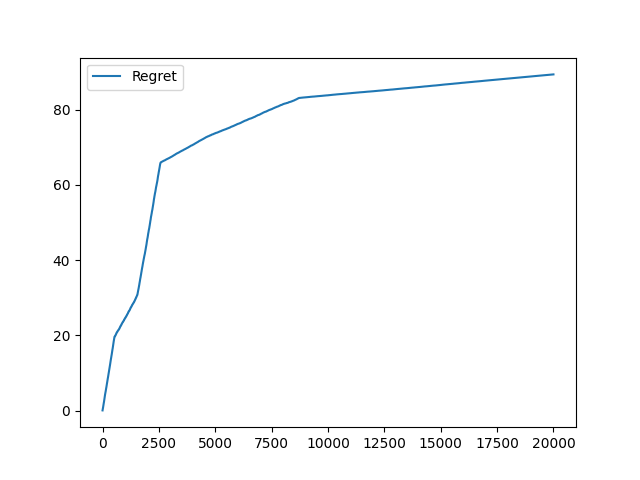}
        \caption{The regret accumulation of CLUB.}
        \label{fig:Gauss/MDP+sub2}
    \end{subfigure}
    \hfill
    \begin{subfigure}[b]{0.3\textwidth}
        \includegraphics[width=\textwidth]{ 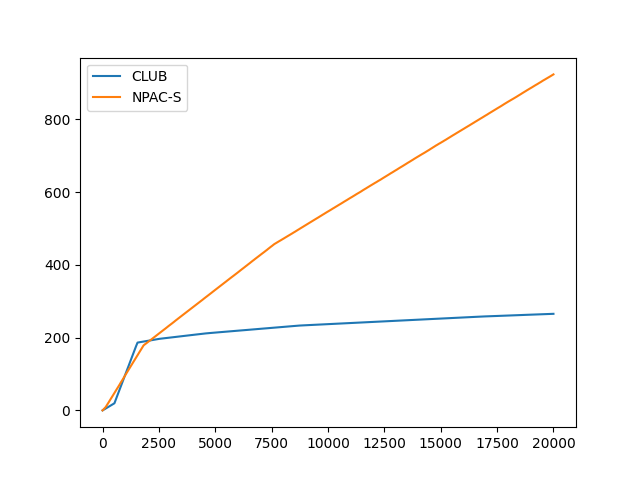}
        \caption{The average performances of two algorithms.}
        \label{fig:Gauss/MDP+sub3}
    \end{subfigure}
    \caption{Experiment results for the MDP setting under truncated Gaussian noise distribution.}
    \label{fig:gauss/MDP}
\end{figure}

To examine the robustness of our algorithm and its potential for real-world applications, we consider different market noise distributions. Specifically, we replace $F(\cdot)$ from a uniform distribution with a normal distribution $\cN(0,1)$ truncated to the interval $[-1,1]$, while keeping all other settings unchanged. For the contextual bandit setup (c.f. \Cref{fig:gauss/contextual}), we also report one representative instance in \Cref{fig:CLUB_regret}. We find that under truncated Gaussian noise, our algorithm CLUB outperforms both SCORP and NPAC-S, with average 164.09, 200.29 and 193.15, respectively, demonstrating the strong performance of CLUB under different noise distributions. For the MDP setting (c.f. \Cref{fig:gauss/MDP}), our CLUB algorithm is still significantly plausible compared to NPAC-S. The average regret of CLUB is only 265.59, while NPAC-S incurs 923.95. It shows that our CLUB algorithm can effectively handle different types of bidders in real-world environments.

\re{To sum up, the performance of CLUB and NPAC-S is comparable in contextual bandit settings, overwhelming the performance of SCORP sufficiently. As for MDP setting, CLUB is the only one to achieve sublinear regret bounds in both theory and practice. Therefore, CLUB captures the underlying information structures precisely and depicts a practical way in dynamic mechanism design.}


\bibliography{sample}

\end{document}